%% file: main.tex
\documentclass[twoside]{article}

\usepackage[accepted]{aistats-arxiv}
\usepackage{hyperref}
\usepackage{graphicx} 
\usepackage[round]{natbib}

\usepackage{float}
\usepackage{placeins}
\usepackage{afterpage}
\usepackage{dblfloatfix}

\usepackage[font={small},labelfont={bf}]{caption}

\usepackage{algorithm}
\usepackage{algorithmicx}
\usepackage[noend]{algpseudocode}
\algrenewcommand\alglinenumber[1]{\tiny #1:}
\algrenewcommand\algorithmicrequire{\textbf{Input:}}
\algrenewcommand\algorithmicensure{\textbf{Output:}}

\newfloat{algorithm}{t}{lop}
\floatname{algorithm}{Algorithm}
\usepackage{xcolor}
\usepackage{todonotes}

\usepackage{amsfonts}       
\usepackage{amsmath}
\usepackage{amsthm}
\usepackage{amssymb}
\usepackage{mathtools}
\usepackage{dsfont}

\newcommand{\NA}{\texttt{NA}}

\usepackage{booktabs}
\usepackage{multirow}
\usepackage{nicematrix}

\usepackage{enumitem}
\setlist{nosep}

\usepackage{thmtools}
\usepackage{thm-restate}
\makeatletter 
\newcommand{\IfRestatedTF}[2]{\ifthmt@thisistheone #2\else #1\fi}
\makeatother

\theoremstyle{plain}
\newtheorem{theorem}{Theorem}
\newtheorem{proposition}[theorem]{Proposition}
\newtheorem{lemma}[theorem]{Lemma}
\newtheorem{corollary}[theorem]{Corollary}
\newtheorem{remark}[theorem]{Remark}
\theoremstyle{definition}

\newcommand{\mdoubleplus}{\ensuremath\mathbin{+\mkern-10mu+}}

\usepackage{physics}

\usepackage{tikz}

\begin{document}

\twocolumn[

\aistatstitle{Model Evaluation in the Dark: Robust Classifier Metrics with Missing Labels}

\aistatsauthor{ Danial Dervovic \And Michael Cashmore }

\aistatsaddress{ J.P. Morgan AI Research \\
315 Argyle Street, Glasgow, UK\\
\texttt{danial.dervovic@jpmorgan.com} 
\And J.P. Morgan AI Research \\
315 Argyle Street, Glasgow, UK\\
\texttt{michael.cashmore@jpmorgan.com}
}
]
\begin{abstract}
\input{abstract}
\end{abstract}

\section{INTRODUCTION}
Missing data, the absence of values in a dataset, remains a persistent challenge in Machine Learning (ML). 
It introduces bias, hampers generalisation, and ultimately diminishes a model's predictive performance.

Multiple mechanisms can lead to data being missing~\citep{rubin1976inference}, most commonly Missing Completely at Random (MCAR), Missing at Random (MAR) and Missing Not at Random (MNAR).
MCAR\vfill\eject corresponds to the causes of the missing data, or ``missingness'', being completely unrelated to the data itself. 
MAR means missingness is related to the observed data but not the missing data.
If neither MCAR nor MAR hold, we have MNAR where  missingness is related to the unobserved data.
These mechanisms can be formalised using graphical models, as in Figure~\ref{fig:missing_data}.
Deducing whether data from a particular source are missing due to MCAR, MAR or MNAR is difficult without access to the data-generating mechanism~\citep{DHAULTFOEUILLE20101}.

\input{plots/mnar_diagram}

There are numerous approaches to remove or impute missing values, such as ignoring samples with missing values, known as listwise deletion~\citep{allison2001missing}; imputing a random, static, or statistically derived value~\citep{donders2006gentle}; or training another model to impute a predicted value~\citep{little2019statistical,song2007missing}. Without careful handling of the missing data, the predictive power of ML models can be reduced~\citep{ayilara2019impact}.

This work considers another issue for ML in the presence of missing data, which is \textit{evaluating} models, specifically when the missing values are ground-truth labels at test-time.
Unless labels are MCAR then any technique to ignore or impute missing values can lead to a biased estimate~\citep{williams2015missing}.
This is an issue that can appear in model monitoring contexts~\citep{Krishnaram2022}, when there is a delay in receiving labels from a production ML system.
For example, in evaluating a fraud detection model we may receive some ground-truth labels before others; or some users may delay providing feedback to a recommender system.

In particular, for the setting where some fraction of labels are missing at test-time and the task is evaluating the performance of a binary classifier, we provide algorithms -- Performance Estimation by Multiple Imputation (PEMI), in Algorithm~\ref{alg:PEMI}, and PEMI-Gauss in Algorithm~\ref{alg:PEMI-Gauss} -- for correctly describing a decision-maker's ignorance about the value of various performance measures.
Notably, these include precision, recall and ROC-AUC amongst others.
To the authors' knowledge this work is the first to tackle missing labels and ML model evaluation simultaneously.

The description of the ignorance of a performance measure's value is naturally given as a probability distribution.
We rigorously prove finite-sample convergence bounds to a normal distribution in Theorem~\ref{thm:ratio_sum_bernoulli}.
The algorithms are further proven in Theorem~\ref{thm:robustness} to be immune to a realistic amount of noise in the imputation process.
As an intermediate step, we prove finite-sample convergence bounds of correlated Gaussian random variables to a Gaussian distribution (Theorem~\ref{thm:ratio_gauss}), which may be of independent interest.
Experiments on real datasets in Section~\ref{sec:appendix_expt} show that PEMI and PEMI-Gauss are effective over a range of missingness mechanisms when even a large amount of labels (30\%) are missing.


\paragraph{Related Work.}

Missing data has been considered in breadth and depth within the statistical literature, primarily in the context of inference~\citep{JosseReiter2018, little2019statistical}.
Approaches fall broadly into two camps: likelihood-based and imputation-based.
Likelihood-based methods, such as those for binomial regression~\citep{Ibrahim1996} or Generalised Linear Models~\citep{Ibrahim2001} use Expectation Maximisation (EM)~\citep{dempster1977maximum}.

Multiple Imputation is the most common imputation-based approach within the ML context~\citep{BertsimasMultipleImp2018}, proposing to replace the missing covariates with a function of the non-missing data. This is done several times non-deterministically to account for variance.
Within supervised learning,~\citet{josse2024consistencysupervisedlearningmissing, BertsimasMultipleImp2018} examine so-called Impute-then-Regress methods.
The work by~\citet{MorvanNeurIPS21} generalises these results, showing that Impute-then-Regress
procedures are asymptotically Bayes consistent for all missing data mechanisms and almost all imputation functions, regardless of the distribution of $\langle X, Y \rangle$ and the number of missing covariates.
We note that these works assume only the covariates $X$ have missing data, whereas we are focused on missing labels $Y$.
These asymptotic, general results are supplemented by works~\citep{ayme2024randomfeaturesmodelsway,lobo2024harnessingpatternbypatternlinearclassifiers} that specialise to specific models and missingness mechanisms to obtain sharpened bounds.

Focusing on missing labels, Semi-supervised Learning (SSL)~\citep{Scudder1965} leverages unlabelled alongside labelled data~\citep{chapelle2006} to improve performance on supervised learning tasks. 
Work on SSL with awareness of the causal diagram of missingness mechanisms is in its infancy. 
Notable examples include the doubly robust methods of~\citet{DBLP:conf/iclr/HuNM0Z22} and~\citet{Sportisse2023}, where they attempt direct modelling of the missingness mechanism.
The SSL literature generally assumes access to a fully labelled test dataset, whereas we are specifically interested in the distribution over evaluation metrics induced by the uncertainty over missing test labels. 

Active Learning~\citep{settles.tr09, hino2020active} considers the scenario where there are few labels and the decision-maker simultaneously learns a model and a sample-efficient policy for querying an oracle to gain more labels for learning.
\cite{mishler2023active} consider active learning in the MNAR setting for training rather than evaluation.

Within the Computer Vision community, there is work on accounting for missing test labels on tasks that are not classification.
For example,  ranking models on unknown domains~\citep{Sun2021} and for domain adaptive instance segmentation~\citep{GUAN2024107204}.

\citet{Amoukou2024} provide estimated quantiles of an ML model's loss function at test-time when labels are absent, towards solving the Sequential Harmful Shift Detection problem.
These estimates are used in a sequential testing framework to alert when the test distribution differs from the training distribution.

To the authors' knowledge this work is the first focusing specifically on the predictive distribution for classification metrics in the presence of missing labels.

\section{PROBLEM SETUP}\label{sec:problem}
 


\begin{figure*}[!ht]
    \centering
    \includegraphics[width=0.3\textwidth]{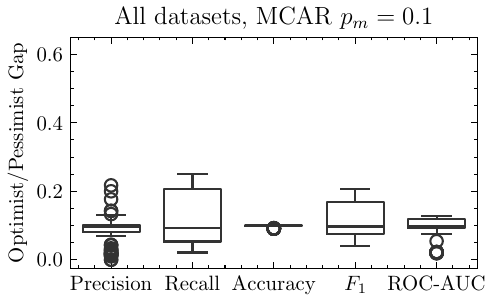}
    \includegraphics[width=0.3\textwidth]{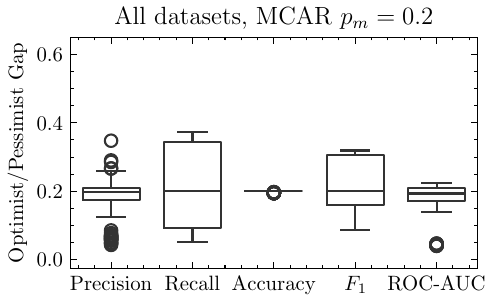}
    \includegraphics[width=0.3\textwidth]{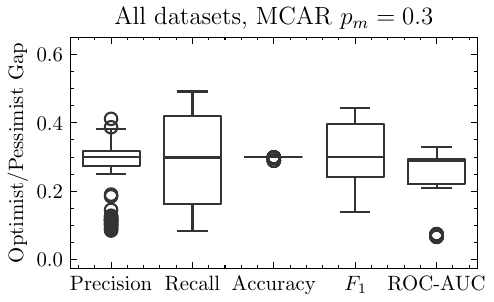}
    \caption{Gap between optimistic and pessimistic bounds on classifier performance measures $\widehat{Q}_n^{(S)}$. Each plot (left-to-right) corresponds to a fraction of missing labels $p_m \in \langle 0.1, 0.2, 0.3 \rangle$.}
    \label{fig:opt_pess_bounds}
\end{figure*}

In this paper we consider binary classification, in which labels $Y$ might be missing and input features $X$ are complete. 
Data $\langle X, Y, M \rangle \in \mathcal{X} \times \mathcal{Y} \times \{0 , 1\}$ are drawn iid (independently and identically distributed) from distribution $\mathcal{D}$, and $\mathcal{Y} = \{0, 1\}$. 
The observed label is denoted by $Y^* \in \mathcal{Y} \cup \{\NA\}$ and $M = 1$ indicates a missing label, that is
$$Y^* = \begin{cases}Y, & M = 0;\\ \NA, & M = 1.\end{cases}$$
We assume that there is a trained model $\hat{y}: \mathcal{X} \to [0, 1]$, where the output $\hat{y}(x)$ is interpreted as the model's subjective probability that the label associated with $x \in \mathcal{X}$ is $1$.
Equipped with a threshold $\tau \in (0, 1)$, $\hat{y}$ induces a classifier $\psi : \mathcal{X} \to \{0, 1\}$ by returning $\psi(x) = 1$ if $\hat{y}(x) \geq \tau$ and $\psi(x) = 0$ otherwise. 


Suppose we have a validation set, $D_n = \{\langle x_i, y^*_i, m_i \rangle \}_{i = 1}^n$, drawn iid from $\mathcal{D}$.
We wish to use this validation set to produce unbiased estimates of some performance metric:
$$
\mathcal{Q}(\hat{y}) = \mathbb{E}_{\mathcal{D}}[ Q(X, Y, \hat{y}(X)) ]
$$
where $Q$ is a probabilistic query\footnote{See Appendix~\ref{sec:prob_queries} for exposition on probabilistic queries.} corresponding to the metric. For example, accuracy has the probabilistic query $Q_{\text{acc}} = P(Y = \psi(X))$ .
We denote an estimator admitted by $\mathcal{Q}$ using the validation set $D_n$ as $\widehat{Q}_n(\hat{y})$, or simply $\widehat{Q}_n$, when $\hat{y}$ is clear from context.
For the accuracy example, when there are no missing labels, that is, $m_i = 0 \Rightarrow y^*_i = y_i$ for all $i \in \{1, \ldots, n\}$, we have:
\begin{align}\label{eq:q_acc_n}
&\widehat{Q}_{\text{acc}, n} = \textstyle\frac{1}{n}\sum_{i=1}^{n} \mathds{1}\{y_i = \psi(x_i)\} \\
&= \textstyle\frac{1}{n} \qty( \sum_{i : y_i = 0} \mathds{1}\{\psi(x_i) = 0\} + \sum_{i : y_i = 1} \mathds{1}\{\psi(x_i) = 1\} ) \notag
\end{align}
where $\mathds{1}\{E\}$ is the indicator function for an event $E$.
Eq.~\eqref{eq:q_acc_n} equals the traditional formula for accuracy:
\begin{align*}
\textstyle\frac{1}{n}(\text{True Positives} + \text{True Negatives}).
\end{align*}
If there is no missing data, i.e. $M = 0$ always, then $\mathcal{Q}$ usually admits an \textit{unbiased}, \textit{consistent} estimator using the validation set $D_n$. That is, 
$$\mathbb{E}_{\mathcal{D}}\qty[\widehat{Q}_n(\hat{y})] = \mathcal{Q}(\hat{y}) \qq{and}
$$
$$
\lim_{n \to \infty} \mathbb{P}_{\mathcal{D}}\qty[\abs{\widehat{Q}_n(\hat{y}) - \mathcal{Q}(\hat{y})} > \epsilon] = 0 \qq{for all} \epsilon > 0.$$
It is straightforward to show that the vanilla estimator for accuracy, $\widehat{Q}_{\text{acc}, n}$, is both unbiased and consistent in the nonmissing labels setting, which for completeness we demonstrate in Proposition~\ref{prop:Q_acc} (Appendix).

\paragraph{Missing Labels.}

Often in practise we have missing data labels in our validation set $D_n$, that is, $M \neq 0$ in general when data is drawn from $\mathcal{D}$.
Unless the labels are MCAR then computing $\widehat{Q}_n$ using the data $(D_n \,\vert \,  M = 0)$, that is, ignoring datapoints with missing labels, will in general give a biased estimate for ~$\mathcal{Q}$~\citep{williams2015missing}.
We wish to produce an unbiased estimate for $\mathcal{Q}$, or better yet, a distribution over plausible values centred on an unbiased estimate.

For simplicity, without loss of generality we assume that the first $k$ labels within $D_n$ are known, giving the set of known labels $\mathcal{K} = \{1, 2, \ldots, k\}$ and the remaining indices as masked labels $\overline{\mathcal{K}} = \{ k+1, k+2, \ldots, n\}$.

\section{NA\"IVE ESTIMATION}\label{sec:naive_est}
For a given $D_n$ there are $2^{n-k}$ possible outcomes describing the true pattern of labels $y_1, \ldots, y_n$.
The set of potential scenarios is defined as  $\mathcal{S} = \big\{\langle y_1, \ldots, y_k, s'_1, \ldots s'_{n-k} \rangle \mid s'_{j} \in \{0, 1\}  \ \text{for all}\ j \in \{1, \ldots, n - k \} \big\}$.
In words, $\mathcal{S} \subset \{0, 1\}^n$ is the set of label assignments to $D_n$ consistent with the observed labels.
Each  $s \in \mathcal{S}$ induces imputed evaluation data $D_n^{(s)}$ by substituting elements of the scenario bitstring $s$ as realised values of $Y$ in $D_n$.
Subsequently, an estimator assuming no missing data $\widehat{Q}_n$ is computed on the imputed dataset $D_n^{(s)}$, yielding the imputed estimator $\widehat{Q}_n^{(s)}$.
Denote by $S$ the random variable describing our ignorance of the true scenario.

We will ask two questions, namely:
\begin{enumerate}
    \item What are the upper and lower bounds on $\widehat{Q}_n^{(S)}$?
    \item What is the probability distribution over $\widehat{Q}_n^{(S)}$ describing our ignorance of $\widehat{Q}_n$?
\end{enumerate}

From the iid assumption on $\mathcal{D}$, we can surmise that the distribution over scenarios is a product distribution over independent label distributions.
We define the random variable $Y_i$ corresponding to the decision-maker's knowledge of the label of a point in $D_n$ like so:
\begin{equation}\label{eq:bern_labels}
Y_i \sim \begin{cases} \mathcal{B}(p_i), & i \in \overline{\mathcal{K}}; \\ \mathds{1}{\{Y_i = y_i\}}, & i \in \mathcal{K}, \end{cases}
\end{equation}
where $\mathcal{B}(p)$ is a Bernoulli distribution with parameter $p \in (0, 1)$ and the $y_i \in \{0, 1\}$ are constants for $i \in \mathcal{K}$.
The parameters $p_i$ for $i \in \overline{\mathcal{K}}$ represent the conditional probability $\mathbb{P}[Y_i = 1 \,\vert\, X = x_i, M = 1]$.
The scenario random variable $S = \langle Y_1, Y_2, \ldots, Y_{n} \rangle$ has the product distribution
$\mathbb{P}[S = s] = \mathbb{P}[Y_1 = s_1] \cdots \mathbb{P}[Y_{n} = s_{n }]$ for $s = s_1 s_2\cdots s_{n}$.

\paragraph{Computational Considerations.}

When there are few missing test labels, i.e. $n - k \lesssim 15$, one can enumerate all scenarios $s \in \mathcal{S}$. 
In this case it is straightforward to answer question 1; we obtain lower and upper bounds by sorting the values in $\{\widehat{Q}_n^{(s)}\}_{s \in \mathcal{S}}$. 
%
The number of scenarios scales exponentially in the number of missing labels $n - k$, precluding this brute force approach in practical problems.
%

\paragraph{Bounds on Imputed Estimator.}\label{sec:bounds}

For a classifier $\psi: \mathcal{X} \to \{0, 1\}$ it is straightforward to produce optimistic and pessimistic bounds for any imputed metric estimator $\widehat{Q}_n^{(S)}$.
Indeed for all missing datapoints $i \in \overline{\mathcal{K}}$, we imagine that an oracle can choose the scenario $s$ that takes place, substituting $s_i = \psi(x_i)$ for the optimistic bound and $s_i = 1 - \psi(x_i)$ for the pessimistic bound.
The oracle evaluates the metric $\widehat{Q}_n^{(s)}$ as one would in the nonmissing case.
The optimistic bound corresponds to the case when the classifier correctly classifies every missing datapoint, the pessimistic bound is the opposite. 

In Figure~\ref{fig:opt_pess_bounds} we show the width of these bounds, evaluated in the MCAR setting on the same experimental setup as in Section~\ref{sec:experiments}. 
We plot the distribution of the widths of the gap between the optimistic and pessimistic values for: precision, recall, accuracy, $F_1$-score and ROC-AUC.
Given that the stated metrics take values in the $[0, 1]$ interval, these bounds are clearly too loose to be of use to a practitioner.
Accordingly, we would prefer a distribution showing how \emph{plausible} any particular value for our chosen metric is.

\section{PERFORMANCE ESTIMATION BY MULTIPLE IMPUTATION}\label{sec:PEMI}

Recall from~\eqref{eq:bern_labels} that for test points with missing labels $i \in \overline{\mathcal{K}}$, Bernoulli distributions $\mathcal{B}(p_i)$ describe knowledge of the labels $Y_i$.
This induces a distribution over the performance metric $\widehat{Q}_n^{(S)}$, suggesting a straightforward algorithm, Performance Estimation by Multiple Imputation (PEMI), to approximate $\widehat{Q}_n^{(S)}$ that is formalised in Algorithm~\ref{alg:PEMI}.
Namely, sample the relevant Bernoulli random variables yielding a bitstring $s$, evaluate $\widehat{Q}_n^{(s)}$, then repeat, recording the values in an empirical cdf (cumulative distribution function). 

\input{pemi}

\paragraph{Assignment of Bernoulli Parameters for PEMI.}
In practise, it is unlikely that we would know the true values of the $\{ p_i \}$.
However, if we have a \emph{well-calibrated} model for these probabilities then these can be substituted in lieu of the true Bernoulli parameters into Algorithm~\ref{alg:PEMI}.
A model being well-calibrated means that the predicted probability of an event reflects the true frequency of the event's occurrence~\citep{DeGroot1983}.
For example, consider a model predicting if it will rain or not. Suppose 1000 inputs to this model receive a predicted 20\% chance of rain. Then out of these, roughly 200 will observe rain if the model is well-calibrated.

By default, many modern ML models are not well-calibrated~\citep{Zadrozny2001, Guo2017, kuleshov2018accurate}.
In this work we use the scaling-binning calibrator of~\citet{kumar2019calibration} to provide estimates of the $\{p_i\}$ for use within Algorithm~\ref{alg:PEMI} to give an empirical distribution of $\widehat{Q}_n^{(S)}$.
At the time of writing, there is early work~\citep{Kweon2024, Gong2025} on MNAR-aware model calibration in the context of recommender systems, but for our purposes the scaling-binning calibrator is sufficient.

If one wishes to be less committal in the assumptions placed upon the missing label distributions, one can invoke the \emph{Maximum Entropy Principle} (MaxEnt) of~\cite{Jaynes1957}, whereby the distribution describing our uncertainty about a random variable is the one maximising Shannon entropy, subject to testable information. 
Within our setting, where the $Y_i$ take values in $\{0, 1\}$ and are drawn iid, we have $p_i = \frac{1}{2}$ when no constraints are imposed.
We may also impose that the mean over the training set and test set are equal, giving $p_i = N_+ / N$ where $N$ is the total number of examples in the training set and $N_+$ the number of positives.
Choosing $p_i$ as the calibrated probability estimate is equivalent to imposing $X_i = x_i$ as testable constraints in the MaxEnt framework.

\paragraph{Performance Metrics.}

We shall see that for several common classification metrics, $\widehat{Q}_n^{(S)}$ is either a weighted sum of Bernoulli random variables or a ratio thereof.
In both cases, we are able to show that the distribution is approximately Gaussian (Lemma~\ref{lem:bernoulli_sum} and Theorem~\ref{thm:ratio_sum_bernoulli}).
The first case is a standard result; the ratio of sums of Bernoulli variables is novel to the authors' knowledge.
We also show in Theorem~\ref{thm:robustness} that these results are robust under a realistic noise model for the calibrator.

\paragraph{Confusion matrix-based metrics.}
In this work, we mainly focus on evaluation metrics based on the confusion matrix (CM), as defined below.
\begin{align*}\label{eq:confusion_matrix}
\text{CM} &= 
\left[
\begin{smallmatrix}
\text{True Positives} & \text{False Negatives} \\
\text{False Positives} & \text{True Negatives} 
\end{smallmatrix}
\right] = 
n\cdot
\left[
\begin{smallmatrix}
\text{TP} & \text{FN} \\
\text{FP} & \text{TN} 
\end{smallmatrix}
\right] \\
 &= n\cdot
\left[
\begin{smallmatrix}
P(Y = 1, \psi(X) = 1) & P(Y = 1, \psi(X) = 0)  \\
P(Y = 0, \psi(X) = 1)  & P(Y = 0, \psi(X) = 0) 
\end{smallmatrix}
\right]
\end{align*}
The elements of the CM may be thought of as the product of the number of evaluation points $n$ with an estimator of certain probabilistic queries using $D_n$.
Common classification metrics comprise simple arithmetic combinations of elements of the CM, and as such correspond to probabilistic queries too.
These probabilistic queries are shown in Table~\ref{tab:prob_queries}.
\input{cm_prob_queries}
\input{cm_prob_queries_estimators}
In Table~\ref{tab:prob_queries_estimators} we show the estimators of the CM elements.
When $\overline{\mathcal{K}} = \emptyset$ we recover the nonmissing estimators for the CM elements.
When $\overline{\mathcal{K}} \neq \emptyset$ the estimators of the CM elements are now random variables. 
We list their means and covariance matrix in Lemma~\ref{lem:cm_mean_cov} in the appendix.

Thus, under the iid assumption, substitution of the CM estimators from Table~\ref{tab:prob_queries_estimators} into Table~\ref{tab:prob_queries} leads us to the following: The classification metrics precision, recall, accuracy and $F_1$-score are distributed as sums of Bernoulli random variables or (correlated) ratios thereof.

\paragraph{Rank-based metrics.}

There are other performance metrics for binary classifiers that do not depend on the confusion matrix, but instead of the \emph{ranking} between data points, that is, how $\hat{y}(x)$ and $\hat{y}(x')$ compare with one another for $x, x' \in \mathcal{X}$.
One such metric is the area under the curve of the receiver operating characteristic~\citep{Hanley1983, BRADLEY19971145}, denoted as ROC-AUC, which we shall also focus on in this work.
Indeed, ROC-AUC corresponds to an estimation of the following probabilistic query
\begin{equation}\label{eq:roc_auc_query}
P(\hat{y}(X) \geq \hat{y}(X') \mid Y = 1,\ Y' = 0),
\end{equation}
where $\langle X, Y \rangle$ and $\langle X', Y' \rangle$ are understood to come from independent draws from $\mathcal{D}$~\citep{BRADLEY19971145}.
We describe the estimator $\widehat{Q}_{\text{roc-auc}, n}^{(S)}$, in Eq.~\eqref{eq:roc_auc_estimator_missing} and its mean and variance in Remark~\ref{rem:roc_auc_mean_cov} (relegated to Appendix for brevity).

\section{APPROXIMATE DISTRIBUTIONS}\label{sec:approx_dist}

Theoretical results in the sequel have proof in Appendix Section~\ref{sec:proofs}.

\paragraph{Sums of Bernoulli Random Variables.}

It is well-known that a sum of Bernoulli random variables under certain conditions is well-approximated by a Gaussian distribution.
To be more precise, ``well-approximated'' means close in Kolmogorov-Smirnov (KS) distance.
For random variables $U, V$ with cdfs $F_U$, $F_V$, the KS distance between their distributions is given by
\begin{align*}
    \mathrm{d}_{\mathsf{KS}}(F_U, F_V) &= \sup_{t \in \mathbb{R}} \abs{F_U(t) - F_V(t)}. 
\end{align*}
We can bound the KS distance of a weighted sum of Bernoulli variables to an appropriately-scaled Gaussian, denoted by $\mathcal{N}(\mu, \sigma^2)$ using Lemma~\ref{lem:bernoulli_sum}.
\begin{restatable}{lemma}{bernoullisum}
\label{lem:bernoulli_sum}
    Let $Z = \sum_{i = 1}^n a_i Y_i$, where $Y_i \sim \mathcal{B}(p_i)$ and mutually independent, $p_i \in (0, 1)$ and $a_i \in \mathbb{R}\setminus \{0\}$.
    Then, we have that
    \vspace*{-3mm}
    \begin{equation*}\mathrm{d}_{\mathsf{KS}}(F_Z , \mathcal{N}(\mu_z, \sigma_z^2)) \leq \frac{C_0}{\sqrt{n v_{*}}} \frac{1 + a^{*}}{2 a_{*}}, \qq{where}
    \end{equation*}
    \vspace*{-6.25mm}
    \begin{align*}
    v_{*} &= \min_i \{ p_i (1 - p_i)\},\ \ a_{*} = \min_i \{\abs{a_i}\},\ a^{*} = \max_i \{\abs{a_i}\},\\
    \textstyle \mu_z &= \textstyle \sum_{i = 1}^n a_i p_i,\ \textstyle \sigma_z^2 = \sum_{i = 1}^n a_i p_i (1 - p_i)
    \end{align*}
    and $C_0 = 0.5600$ is a universal constant.
\end{restatable}

Consultation of Tables~\ref{tab:prob_queries}~and~\ref{tab:prob_queries_estimators} tells us that the estimators for Accuracy and Precision take the format $Z + \mathrm{const.}$ (with $a_* = a^* = 1$), leading to $O(n^{-1/2})$ scaling in the KS distance to a Gaussian (for fixed $v_*$) from Lemma~\ref{lem:bernoulli_sum}.

\paragraph{Ratios of Normal Variables.}

For the remaining CM-based metrics -- Recall and $F_1$ score -- we see from Tables~\ref{tab:prob_queries}~and~\ref{tab:prob_queries_estimators} that each will be distributed according to a ratio of (correlated) sums of Bernoullis.  
From Lemma~\ref{lem:bernoulli_sum} we can treat the sums in the numerator and denominator as Gaussian random variables and consider the distribution of the ratio.
It is known~\citep{Marsaglia1965} that the ratio of correlated normal random variables is poorly behaved in general, in that no moments of the resulting distribution exist.
Nonetheless, it has been shown empirically that in many cases this distribution is approximately normal itself~\citep{Marsaglia2006}. We present the first rigorous bounds on the distance of this ratio distribution from normality.

\begin{restatable}[Ratio of Correlated Gaussians]{theorem}{ratiogauss}
\label{thm:ratio_gauss}
Let $Z$ and $W$ be jointly Gaussian random variables, i.e. $(Z, W) \sim \mathcal{N}(m, \Sigma)$, where $m = [ \mu_z, \mu_w ]^\mathsf{T}$ and
$ \Sigma =
\qty[\begin{smallmatrix}
\sigma_z^2 & \rho \sigma_z \sigma_w\\
\rho \sigma_w \sigma_z & \sigma_w^2\\
\end{smallmatrix}].
$
Moreover, let the the distribution of $T := Z/W$ be described by $G(t) = \mathbb{P}[T \leq t]$.
Then, under conditions 
\begin{enumerate}
    \item $\mathbb{P}(W > 0) \to 1$,. i.e. $\mu_w / \sigma_w \gg 1$;
    \item $2 \abs{\mu_z} \abs{\frac{\sigma_w}{\mu_w} - \frac{\sigma_z}{\mu_z}} \gg 1$;
\end{enumerate}
the distribution of $Z / W$ satisfies
\begin{equation*}
\mathrm{d}_{\mathsf{KS}}(G, \mathcal{N}(\mu, \sigma^2)) \leq  \sqrt{\frac{2}{\pi}} \cdot \frac{\sigma_w^2 (\abs{\mu_z} + \sigma_z^2) + \mu_w^2 }{\sigma \mu_w^3},
\end{equation*}
$$\text{where} \quad \mu = \textstyle \frac{\mu_z}{\mu_w}, \quad \sigma^2 = \textstyle \frac{{\mu_z}^{2} {\sigma_w}^{2} + {\mu_w}^{2} {\sigma_z}^{2} - 2 \, \rho \sigma_z \sigma_w {\mu_z} {\mu_w}}{{\mu_w}^{4}}.
$$
\end{restatable}
In Remark~\ref{rem:bounds_interp} (Appendix) an interpretation of each term in the bound of Theorem~\ref{thm:ratio_gauss} is given.

\paragraph{Ratio Distribution of Sum of Weighted Bernoulli Variables.}

To show a rigorous bound on the distance of the predictive distributions of Recall and $F_1$-score, we cannot assume the numerator and denominator are Gaussian, but explicitly consider them as sums of Bernoulli random variables. 
We combine the result of Theorem~\ref{thm:ratio_gauss} with Lemma~\ref{lem:bernoulli_sum} to prove Theorem~\ref{thm:ratio_sum_bernoulli}, an explicit bound on KS distance from a normal distribution applicable to these performance metrics.
\begin{restatable}[Gaussian Approximation]{theorem}{ratiosumbernoulli}
\label{thm:ratio_sum_bernoulli}
    Consider the random variables 
    \begin{equation*}
    \widetilde{Z} = \alpha + \textstyle\sum_{i = 1}^n a_i Y_i, \quad \widetilde{W} = \beta + \textstyle\sum_{i = 1}^n b_i Y_i
    \end{equation*}
    for $0 < \alpha \leq \beta$, $a_i \geq 0$, $b_i > 0$, $b_i \geq a_i$, $Y_i \sim \mathcal{B}(p_i)$ and mutually independent for $p_i \in (0, 1)$. Moreover, $v_{*} = \min_i \{p_i (1 - p_i)\}$, $a_{*} = \min_i \{\abs{a_i} \mid a_i > 0\}$, $b^{*} = \max_i \{\abs{b_i} \}$ and $n_a = \abs{\{ a_i \mid a_i > 0\}} \leq  n$. 
    The ratio distribution $\widetilde{G}(t) = \mathbb{P}[\widetilde{Z} / \widetilde{W} \leq t]$ is approximately Gaussian, that is
    \vspace*{-0.5em}
    \begin{multline*}
        \mathrm{d}_{\textsf{KS}}(\widetilde{G}, \mathcal{N}(\mu , \sigma ))  \leq \frac{C_0 }{\sqrt{n_a v_{*}}} \frac{1 + b^{*}}{a_{*}} \IfRestatedTF{}{\\} + \sqrt{\frac{2}{\pi}} \cdot \frac{\sigma_w^2 (\abs{\mu_z} + \sigma_z^2) + \mu_w^2 }{\sigma \mu_w^3},
    \end{multline*}
    where $C_0 = 0.5600$ is a universal constant and
    \begin{align*}
    \mu_z &= \alpha + \textstyle \sum_{i = 1}^n a_i p_i, \quad  \sigma_z^2 = \textstyle \sum_{i = 1}^n a^2_i p_i(1 - p_i), \\
    \mu_w &= \beta + \textstyle \sum_{i = 1}^n b_i p_i, \quad  \sigma_w^2 = \textstyle \sum_{i = 1}^n b^2_i p_i(1 - p_i), \\
    \rho &= \textstyle \frac{1}{\sigma_z \sigma_w} \textstyle \sum_{i = 1}^n a_i b_i p_i(1 - p_i) , \\
    \mu &= \textstyle \frac{\mu_z}{\mu_w}, \quad \sigma^2 = \textstyle \frac{{\mu_z}^{2} {\sigma_w}^{2} + {\mu_w}^{2} {\sigma_z}^{2} - 2 \, \rho \sigma_z \sigma_w {\mu_z} {\mu_w}}{{\mu_w}^{4}}.
    \end{align*}
\end{restatable}

\input{pemi_gauss}

The bound in Theorem~\ref{thm:ratio_sum_bernoulli} is a finite-sample bound.
It is not immediately obvious from the bound what the asymptotic behaviour is from the functional form.
The assumptions are too general to provide the asymptotic behaviour on all cases -- consider, e.g. $p_i \to 0$ for some $i \in [n]$, which can make the bound arbitrarily large.
Nonetheless, as a coherence check, we can make several small assumptions and recover $O(n^{-1/2})$ aymptotic behaviour, as shown in Corollary~\ref{cor:o_n_half_bound} (Appendix).  

\paragraph{Gaussian Approximation for ROC-AUC.}

For ROC-AUC, the probabilistic query is also a ratio of sums of Bernoullis (see Eq.~\ref{eq:roc_auc_estimator_missing} in Appendix for details), but now there is dependence between the elements of the sums, such that we cannot directly invoke Berry-Esseen type bounds as in Lemma~\ref{lem:bernoulli_sum}.
As a corollary the bound of Theorem~\ref{thm:ratio_sum_bernoulli} doesn't necessarily hold.
Under certain mixing properties~\citep{bradley2007introduction,dedecker2007weak, doukhan2012mixing} likely satisfied here, one can derive similar convergence bounds -- see the paper by~\citet{Tikhomirov1981} for a classical example.
For the ROC-AUC example this is very involved and delegated to future work.
Nonetheless, we use the mean and variance computed in Remark~\ref{rem:roc_auc_mean_cov} (Appendix).

\paragraph{PEMI-Gauss Algorithm.}

We formalise approximating the limiting behaviour of the PEMI algorithm in Algorithm~\ref{alg:PEMI-Gauss}, Performance Estimation by Multiple Imputation - Gaussian Approximation (PEMI-Gauss).
PEMI-Gauss works with guarantees for any CM-based performance metric, but may also be applied whenever a mean and variance can be computed. 
PEMI-Gauss offers advantages over the empirical cdf provided by PEMI in terms of simplicity, interpretability, analytical convenience, and computational efficiency, in that we no longer generate $B \cdot (n - k)$ random bits.

\begin{figure}
    \centering
    \includegraphics[width=0.435\textwidth]{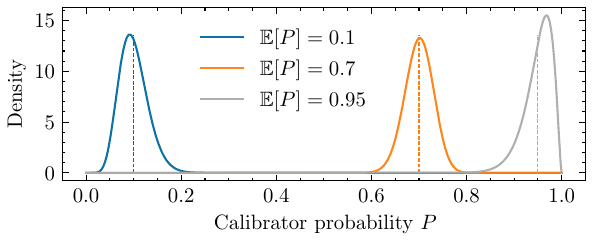}
    \vspace*{-3mm}
    \caption{Illustration of flexibility afforded by Theorem~\ref{thm:robustness}. For $\mathbb{V}[P] = 0.0009$, we show the allowed variation in a calibrator's output $P$ when the expectation $\mathbb{E}[P]$ is fixed to the given true Bernoulli probabilities $p \in \{0.1, 0.7, 0.95\}$.}
    \label{fig:robustness}
\end{figure}

\paragraph{Robustness.}

The performance guarantees offered by PEMI-Gauss to CM-based classification metrics implicitly assume that we have \emph{perfectly calibrated} Bernoulli parameters, which is unrealistic for real-world use cases.
We are able to prove that the same guarantees hold when the calibrated Bernoulli parameters are subject to a realistic noise model.
\begin{restatable}[Robustness]{theorem}{robustness}
\label{thm:robustness}
    The conclusion of Theorem~\ref{thm:ratio_sum_bernoulli} remains valid if we relax the assumption that $Y_i \sim \mathcal{B}(p_i)$ for $p_i \in (0, 1)$ to the following: the random variables $Y_i \sim \mathcal{B}(P_i)$, where $P_i \sim \operatorname{Beta}(\alpha_i, \beta_i)$ such that $\mathbb{E}[P_i] = p_i$ and $\mathbb{V}[P_i] < v_*$ for all $i \in [n]$.
\end{restatable}

Theorem~\ref{thm:robustness} is a formal assertion that using PEMI-Gauss with a calibrator that is correct \emph{on average} results in correct predictive distributions. 
As a realistic illustration: suppose that in the evaluation set the smallest (resp. largest) true Bernoulli parameter is $p_* = 0.0001$ (resp. $p^* = 0.9999$). 
The maximum variance allowed by Theorem~\ref{thm:robustness} of the underlying Beta distributions is therefore $v_* = 0.000999$. 
For a Bernoulli parameter $p$, we can plot the pdf (probability density function) of the distribution of $P \sim \operatorname{Beta}(\alpha_p, \beta_p)$ corresponding to imposing $\mathbb{E}[P] = p$ and $\mathbb{V}[P] = 0.0009 < v_*$ (which uniquely defines $\alpha_p$, $\beta_p$).
Consider the input data $X_p = \{ X \mid \mathbb{P}[Y = 1 \mid X,\ M = 1] \approx p\}$, that has missing labels and true Bernoulli parameter $p$. 
If the calibrated predictions follow the distribution of the random variable $P$ over $X_p$ then the guarantee of Theorem~\ref{thm:robustness} holds.
We plot in Figure~\ref{fig:robustness} the density of $P$ for several values of $p$ and note that the spread of the distributions is quite generous, that is, the calibrator can err significantly and Theorem~\ref{thm:robustness} ensures the distribution output by PEMI-Gauss has high fidelity.

\section{EXPERIMENTS}\label{sec:experiments}
Recall that PEMI and PEMI-Gauss return a predictive distribution for estimators $\widehat{Q}_n^{(S)}$ when labels are missing in the evaluation set $D_n$.
We study the quality of these algorithms along the following dimensions:
\begin{enumerate}
    \item Centre and shape of predictive distribution.
    \item Effectiveness across missingness mechanisms.
\end{enumerate}
\vspace*{-2mm}
\paragraph{Setup.}
For a given dataset $D = \{ \langle x_1, y_1 \rangle, \ldots , \langle x_N, y_N \rangle \}$ we partition randomly into 10 stratified folds $D^{1}, \ldots , D^{10}$, that is, the fraction of positive labels $N_+ / N$ is the same in each fold $D^u$ and the same as in $D$.
Within a training set $D \setminus D^{u}$, 90\% of the data is randomly chosen to train an XGBoost~\citep{xgboost} model\footnote{Default hyperparameters are used with Categorical mode enabled.} $\hat{y}$ and the remaining 10\% is chosen to fit a scaling-Binning calibrator~\citep{kumar2019calibration} $\hat{c}: [0, 1] \to [0, 1]$ with 10 bins (default).
Each test fold $D^u$ is subsequently split into two further random (non-stratified) sub-folds, $D^{u, 1}$ and $D^{u, 2}$.
The same model trained and calibrated on $D \setminus D^u$ is evaluated on two test sets: \emph{i}. $D^u$ with labels missing from $D^{u, 1}$; and \emph{ii.} $D^u$ with labels missing from $D^{u, 2}$.
For a given dataset this gives 20 replications in total and ensures that the missing labels are independent of one another.

We consider missingness proportions $p_m \in \{0.1, 0.2, 0.3\}$, masking labels in two settings:
\emph{i.} \textbf{MCAR.} Randomly sample a fraction $p_m$ from the relevant sub-fold  $D^{u, \,\cdot\,}$;
\emph{ii.} \textbf{MNAR.} Randomly sample a fraction $p_m$ from the relevant sub-fold $D^{u, \,\cdot\,}$, ensuring a fraction $\eta \in (0, 1)$ has a positive label.
For masked datapoints $i \in \overline{\mathcal{K}}$, we reserve the true labels $y_i$ for evaluation and set $y^*_i = \NA$ in the test sets $D^{u, \,\cdot\,}$. 
For the MNAR experiments, we independently consider $\eta \in \{ 0.1, 0.2, 0.4, 0.6, 0.8, 0.9 \}$.

For each sub-fold $D^{u, \,\cdot\,}$, there are three ways we choose the Bernoulli parameters input to PEMI and PEMI-Gauss for all $i \in \overline{\mathcal{K}}$. 
\begin{enumerate}
    \item Set $p_i = \frac{1}{2}$ (MaxEnt -- no testable constraints);
    \item Set $p_i = \frac{N_+}{N}$  (MaxEnt -- fixed mean);
    \item Set $p_i = \hat{c}(\hat{y}(x_i))$ (calibrated);
\end{enumerate}

The six datasets under consideration are \textsf{Dota2 Games Results}~\citep{dota2_games_results_367}, \textsf{IMDB.drama}~\citep{IMDB.drama, IMDB.drama.permission}, \textsf{Bank Marketing}~\citep{bank_marketing_222}, \textsf{Diabetes}~\citep{diabetes_34}, \textsf{German Credit}~\citep{statlog_(german_credit_data)_144} and \textsf{Adult}~\citep{adult_2}.


\paragraph{Evaluation.}
We turn to the forecasting literature for assessing the predictive distribution of $\widehat{Q}_n^{(S)}$.
Probability Integral Transform (PIT) is the idea that given a random variable $U$ with cdf $F_U$, the random variable $V = F_U(U)$ has a standard uniform distribution, which we denote by $\mathcal{U}[0, 1]$.
\cite{Diebold1998} show that for any sequence of random variables $T_1, T_2, ...$ with no forward dependence -- formally,
$T_r \perp \langle T_{r+1}, T_{r+2}, \ldots \rangle \mid \langle T_1, T_2, \ldots, T_r \rangle$ for all $r$ -- then $F_{T_1}(T_1), F_{T_2}(T_2), \ldots \sim \mathcal{U}[0, 1]$ independently.
This result is used to evaluate forecast distributions $\widehat{F}_{r}$, by evaluating on realisations of $T_r$, that is, computing $\widehat{F}_{r}(t_r)$ and comparing the empirical cdf (aggregated over $r$) visually against the cdf of a $\mathcal{U}[0, 1]$ random variable\footnote{A plot $x = y$ on the Cartesian plane with $x, y \in [0, 1]$.}.
Let $\widehat{Q}^u$ be the estimator $\widehat{Q}(\hat{y})$ evaluated on $D^u$, i.e. the ground truth for fold $u$.
From~\citet{Diebold1998} we expect the correct distribution of $\widehat{Q}_n^{(S)}$ evaluated at the ground truth values, $F_{\widehat{Q}_n^{(S)}}(\widehat{Q}^u)$, to be distributed according to $\mathcal{U}[0, 1]$, for any dataset, fold or metric.

For metrics \{Precision, Recall, Accuracy, $F_1$-score, ROC-AUC\} we compute 120 ground-truth quantities $\widehat{Q}^u$.
Then, for each predictive distribution $\widehat{F}_{\mathcal{A}}$ provided by algorithm $\mathcal{A}$ we evaluate $\widehat{F}_{\mathcal{A}}(\widehat{Q}^u)$ and record the values in an empirical cdf $\widehat{F}_{\mathcal{A}}^{\text{PIT}}$.
The distance of $\widehat{F}_{\mathcal{A}}^{\text{PIT}}$ is then measured against $\mathcal{U}[0, 1]$ in the following ways: Wasserstein-1 ($W_1$) distance; and KS distance,
where for real-valued r.v.s $U$, $V$ we have
\begin{equation*}
\textstyle\mathrm{d}_{W_1}(F_U, F_V) = \int_{-\infty}
^{+\infty} \abs{F_U(t) - F_V(t)} \dd{t}. 
\end{equation*}
The $W_1$ and KS distances capture the quality of the shape of the predictive distributions.
We also measure the central tendency of the predicted distribution by root mean-squared error (RMSE) and mean absolute error (MAE).
For all distances lower values are better.

\paragraph{Baselines.}

We are not aware of any prior works studying this problem, so as a baseline we use the bootstrap distribution of the estimator $\widehat{Q}^u$ for a given metric on the nonmissing data only, with $B = 10000$ replicates as recommended by~\citet{Hesterberg2015s}.

\paragraph{Results.}

\begin{table}[t]
\centering
        {
    \caption{(Distribution shape) $\mathrm{d}_{W_1}(\widehat{F}^{\text{PIT}}_{\mathcal{A}}, \mathcal{U}[0, 1])$ for different algorithms in MCAR setting, $p_m = 0.3$.
        Best algorithm in \textbf{bold}.}
        \label{tab:W_1_tab_p_0.3_basic}
        \resizebox{\columnwidth}{!}{%
        \input{tables/W_1_df_p_0.3_alpha_0.9.tex}
        }
        }
\end{table}

\begin{table}[t]
\centering
        {
        \caption{(Distribution central tendency) MAE of centre of predictive distributions in MCAR setting, $p_m = 0.3$. Best algorithm in \textbf{bold}.}
        \label{tab:mae_tab_p_0.3_basic}
        \resizebox{\columnwidth}{!}{%
        \input{tables/df_mae_p_0.3_alpha_0.9.tex}
        }
        }
\end{table}

In Tables~\ref{tab:W_1_tab_p_0.3_basic}~and~\ref{tab:mae_tab_p_0.3_basic} we show the $W_1$ and MAE fidelity measures for the MCAR setting when the fraction of missing labels $p_m$ is 30\% of the test set.
In all cases we see that either PEMI or PEMI-Gauss with calibrated $\{p_i\}$ is the best choice.
Tables~\ref{tab:W_1_tab_p_0.3_conf}~and~\ref{tab:mae_tab_p_0.3_conf} (Appendix) show the calibrated multiple-imputation based methods to be mostly overlapping at the $\alpha=0.9$ confidence level.
Note that the Gaussian approximation works well, in that the numbers for PEMI-Gauss are similar or better than the PEMI.
The effect of poor calibration on the multiple imputation methods is stark, showing significantly diminished performance as compared with their calibrated counterparts.
In the MCAR setting, we would expect that performance measures focusing on the location of the predictive distribution, such as MAE and MSE, will show little difference between the PEMI methods and bootstrap. This is because there is no bias in the missing data.
On the other hand, we would expect that for measures considering shape also -- namely KS and $W_1$ distance -- for there to be some impact on performance due to ignoring missing data. Indeed this is largely borne out in the tables in Appendix Section~\ref{sec:mcar_tables}.

\begin{figure*}
    \centering
    \includegraphics[width=0.99\textwidth]{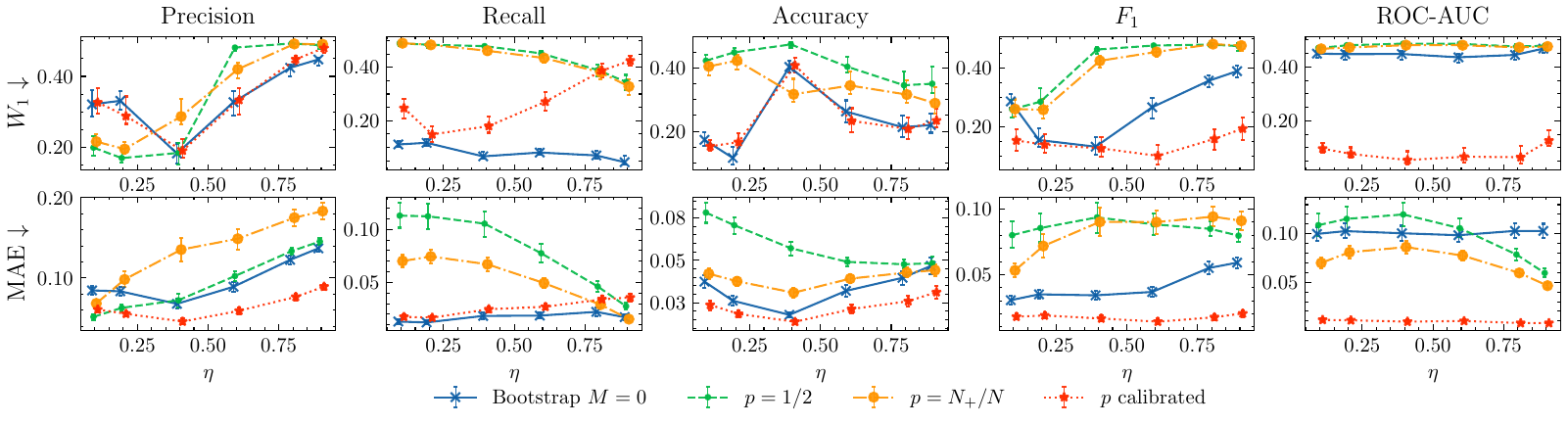}
    \caption{Effect of varying MNAR class imbalance $\eta \in (0, 1)$ on fidelity using PEMI-Gauss and baseline for $p_{m} = 0.3$. Each column represents a metric estimator, with the top row showing $W_1$-distance and bottom row showing MAE. Error bars are bootstrapped confidence intervals at the $\alpha=0.9$ level. Fidelities are on $y$-axis (lower is better) with $\eta$ varying on the $x$-axis; $x$-axis values are slightly jittered to help separate each series visually.}
    \label{fig:mnar_seq_plots}
\end{figure*}

Figure~\ref{fig:mnar_seq_plots} shows the effect of varying the MNAR class imbalance $\eta$ on the predictive quality of PEMI-Gauss and the bootstrap distribution the for $W_1$ and MAE distances.
There is an exception for the Precision metric and $W_1$-distance with $\eta \in \{0.1, 0.2\}$ -- here the uncalibrated models outperform. 
The shape of the predictive distribution for the Precision metric using well calibrated Bernoulli probabilities is actually worse for small $\eta$ than more crude estimates for the $p_i$.
In the parlance of Section~\ref{sec:PEMI}, Precision measures TP/P. 
P is fixed by the classifier outputs $\psi_i$ so it is the shape of the distribution of $\mathrm{TP} = \sum_{i \in \mathcal{K}} y_i \psi_i + \sum_{i \in \overline{\mathcal{K}}} Y_i\psi_i$ that controls the Precision distribution. 
Here, $Y_i$ are more likely to be zero due to small $\eta$. 
The result is that the calibrated probabilities $p_i$ are overestimated and this drives the difference.
The cruder estimates of $p_i$ are less opinionated so give a smaller error.

For the Recall metric, we have the bootstrap distribution of nonmissing data outperforming PEMI-Gauss.
There is a similar reason for this as in the previous paragraph.
Due to small $\eta$ the PEMI predictive distribution is underestimating the true value. Recall is given by $\mathrm{TP}/ (\mathrm{TP} + \mathrm{FN})$.
Overestimation in the denominator dominates overestimation in the numerator so we have that the overall predictive distribution underestimates.

The gap between the Bootstrap and PEMI-Gauss is less pronounced in the MNAR setting as compared with MCAR, but still significant, for all but the Recall metric where the bootstrap outperforms.
For all estimation methods, performance in the MNAR setting is worse than MCAR as we may expect.

We show the results of these experiments in full for $p_m \in \{0.1, 0.2, 0.3\}$, alongside results for KS distance and RMSE in Appendix~\ref{sec:appendix_expt}.
The data are qualitatively similar and the same conclusions can be drawn.

\section{CONCLUSION}\label{sec:conclusion}

In this paper we demonstrate -- via the PEMI and PEMI-Gauss algorithms -- that multiple imputation can be successfully applied to model evaluation when test-time labels are missing.
Crucially, the theoretical and empirical results in Sections~\ref{sec:approx_dist}~and~\ref{sec:experiments} show that a well-calibrated imputation model is key to the success of this approach.
We observe that state-of-the-art calibrators are sufficient for this purpose.
However, PEMI and PEMI-Gauss are more effective with certain classification metrics, so we caution users of this method to evaluate on their own use cases prior to deployment.

For future work, there are interesting directions including, but not limited to:
\begin{itemize}
    \item Greater breadth of binary classification metrics (especially rank-based) and proving guarantees for these.
    \item Extension to multi-class classification. 
    We might consider extending from $Y_i$ being Bernoulli distributed to multinoulli (categorical) distributed.
    For the robustness results, we used the conjugate prior to the Bernoulli distribution, the Beta distribution. In a similar way we might consider the conjugate to the multinoulli, that is, the Dirichlet distribution.
    On the evaluation metrics themselves, their multi-class variants are typically treated the same as their binary counterparts, using a ``one vs rest'' logic.
    \item Expanding PEMI to regression models and relevant metrics.
    Here, the continuous, potentially unbounded nature of the label space $\mathcal{Y}$ may present some difficulties.
\end{itemize}

{\small
\input{disclaimer}
}

\bibliographystyle{ACM-Reference-Format}
\bibliography{references}

\appendix
\onecolumn
\input{appendix_misc}

\input{appendix_proofs}

\clearpage
\input{appendix_expt}

\end{document}

%% file: abstract.tex
Missing data in supervised learning is well-studied, but the specific issue of \emph{missing labels} during model evaluation has been overlooked. Ignoring samples with missing values, a common solution, can introduce bias, especially when data is Missing Not At Random (MNAR). We propose a multiple imputation technique for evaluating classifiers using metrics such as precision, recall, and ROC-AUC. This method not only offers point estimates but also a predictive distribution for these quantities when labels are missing. We empirically show that the predictive distribution's location and shape are generally correct, even in the MNAR regime. Moreover, we establish that this distribution is approximately Gaussian and provide finite-sample convergence bounds. Additionally, a robustness proof is presented, confirming the validity of the approximation under a realistic error model.

%% file: plots/mnar_diagram.tex
\begin{figure}[h]
    \centering
    \begin{tikzpicture}[node distance=1cm, auto, scale=0.58]
    \node (X1) at (0,0) {$X$};
    \node (Y1) [right of=X1] {$Y$};
    \node (Ystar1) [right of=Y1] {$Y^*$};
    \node (M1) [below of=Y1, node distance=0.8cm] {$M$};

    \draw[->] (X1) -- (Y1);
    \draw[->] (Y1) -- (Ystar1);
    \draw[->] (M1) -- (Ystar1);

    \node at (1.75,1) {\textsf{MCAR}};

    \node (X2) at (5,0) {$X$};
    \node (Y2) [right of=X2] {$Y$};
    \node (Ystar2) [right of=Y2] {$Y^*$};
    \node (M2) [below of=Y2, node distance=0.8cm] {$M$};

    \draw[->] (X2) -- (Y2);
    \draw[->] (Y2) -- (Ystar2);
    \draw[->] (X2) -- (M2);
    \draw[->] (M2) -- (Ystar2);

    \node at (6.75,1) {\textsf{MAR}};

    \node (X3) at (10,0) {$X$};
    \node (Y3) [right of=X3] {$Y$};
    \node (Ystar3) [right of=Y3] {$Y^*$};
    \node (M3) [below of=Y3, node distance=0.8cm] {$M$};

    \draw[->] (X3) -- (Y3);
    \draw[->] (Y3) -- (Ystar3);
    \draw[->] (M3) -- (Ystar3);
    \draw[->] (Y3) -- (M3);

    \node at (11.75,1) {\textsf{MNAR}};
    \end{tikzpicture}
    \caption{Illustration of missingness mechanisms in relation to this work, with arrows indicating causation. $X \in \mathcal{X}$ are input features, $Y \in \mathcal{Y}$ is the label, $M \in \{0, 1\}$ is the missingness flag and $Y^* \in \mathcal{Y} \cup \{\NA\}$ is the masked label.}
    \label{fig:missing_data}
\end{figure}

%% file: pemi.tex
\begin{algorithm}
\caption{Performance Estimation by Multiple Imputation (PEMI)}
\label{alg:PEMI}
\footnotesize
\begin{algorithmic}[1]
\Require Metric estimator $\widehat{Q}_n$, nonmissing labels $\{y_i \mid i \in \mathcal{K}\}$, missing Bernoulli parameters $\{p_i \mid i \in \overline{\mathcal{K}}\}$, number of samples $B \in \mathbb{N}$.
\Ensure Empirical cdf $\widehat{F}$
\State $\widehat{F} \gets \qty(t \mapsto 0)$ \Comment{Zero function}
\For{$b \in \{1, \ldots, B\}$}
    \State $s \gets \langle \, \rangle$ \Comment{$\langle \, \rangle$ is empty string} 
    \For{$i \in \langle 1, \ldots, n\rangle$}
        \If{$i \in \overline{\mathcal{K}}$}
            \State Sample $s_i \sim \mathcal{B}(p_i)$
            \State $s \gets s \mdoubleplus \langle s_i \rangle$ \Comment{$\mdoubleplus$ is string concatenation} 
        \Else
            \State $s \gets s \mdoubleplus \langle y_i \rangle$ 
        \EndIf
    \EndFor
    \State $\widehat{F} \gets \widehat{F} + \qty(t \mapsto \frac{1}{B}\mathds{1}\{ \widehat{Q}_n^{(s)} \leq t\})$
\EndFor
\State \Return{$\widehat{F}$}
\end{algorithmic}
\end{algorithm}

%% file: cm_prob_queries.tex
\begin{table}
\centering
\caption{CM-based classification metrics with associated probabilistic queries.}
\label{tab:prob_queries}
\resizebox{\columnwidth}{!}{%
\begin{tabular}{lcc}
\toprule
Metric      & CM Formula                                                                                                                                                                       & Probabilistic Query                                                                                                                                        \\ \midrule
Precision   & $\frac{\text{TP}}{\text{TP} + \text{FP}}$ & $\frac{P(Y = 1, \psi(X) = 1)}{P(Y = 1, \psi(X) = 1) + P(Y = 0, \psi(X) = 1)} = \frac{P(Y = 1, \psi(X) = 1)}{P(\psi(X) = 1)}$ \\ 
Recall      & $\frac{\text{TP}}{\text{TP} + \text{FN}}$                                                                                                                                                                          &                                                                                                            $\frac{P(Y = 1, \psi(X) = 1)}{P(Y = 1, \psi(X) = 1) + P(Y = 1, \psi(X) = 0)}$                                                \\ 
Accuracy    & $\text{TP} + \text{TN}$                                                                                                                                                                           &                                                      $P(Y = 1, \psi(X) = 1) + P(Y = 0, \psi(X) = 0)$                                                                                                      \\ 
$F_1$ Score & $\frac{2 \text{TP}}{2\text{TP} + \text{FP} +  \text{FN}}$                                                                                                                                                                          &                                                                                  $\frac{2 P(Y = 1, \psi(X) = 1)}{2P(Y = 1, \psi(X) = 1) + P(Y = 0, \psi(X) = 1) +  P(Y = 1, \psi(X) = 0)}$                                                                          \\
\bottomrule
\end{tabular}
}
\end{table}

%% file: cm_prob_queries_estimators.tex
\begin{table}
\centering
\caption{CM estimators. For a variable $z \in \{0, 1\}$, let $\overline{z} = 1 - z$. Recall $Y_i \sim \mathcal{B}(p_i)$ for all $i \in \overline{\mathcal{K}}$.}
\label{tab:prob_queries_estimators}
    \resizebox{\columnwidth}{!}{%
\begin{tabular}{lc}
\toprule
CM & Estimator \\
\midrule
TP & $\sum_{i = 1}^n \mathds{1}\{y_i = 1 \wedge \psi_i = 1\}= \sum_{i \in \mathcal{K}} y_i \psi_i + \sum_{i \in \overline{\mathcal{K}}} Y_i\psi_i$ \\
FN & $\sum_{i = 1}^n \mathds{1}\{y_i = 1 \wedge \psi_i = 0\}= \sum_{i \in \mathcal{K}} y_i \overline{\psi_i} + \sum_{i \in \overline{\mathcal{K}}} Y_i \overline{\psi_i}$ \\
FP & $\sum_{i = 1}^n \mathds{1}\{y_i = 0 \wedge \psi_i = 1\} = \sum_{i \in \mathcal{K}} \overline{y_i} \psi_i + \sum_{i \in \overline{\mathcal{K}}} \overline{Y_i} \psi_i$ \\
TN &  $\sum_{i = 1}^n \mathds{1}\{y_i = 0 \wedge \psi_i = 0\} = \sum_{i \in \mathcal{K}} \overline{y_i} \overline{\psi_i} + \sum_{i \in \overline{\mathcal{K}}} \overline{Y_i} \overline{\psi_i}$  \\
\bottomrule
\end{tabular}
}
\end{table}

%% file: pemi_gauss.tex
\begin{algorithm}
\caption{Performance Estimation by Multiple Imputation - Gaussian Approximation (PEMI-Gauss)}
\label{alg:PEMI-Gauss}
\footnotesize
\begin{algorithmic}[1]
\Require Metric estimator $\widehat{Q}_n$, nonmissing labels $\{y_i \mid i \in \mathcal{K}\}$, missing Bernoulli parameters $\{p_i \mid i \in \overline{\mathcal{K}}\}$.
\Ensure Estimated CDF $\widehat{F}$
    \State Compute Gaussian params $\mu$, $\sigma^2$ using Table~\ref{tab:prob_queries}, Table~\ref{tab:prob_queries_estimators}, Lemma~\ref{lem:cm_mean_cov}, Theorem~\ref{thm:ratio_sum_bernoulli} and Remark~\ref{rem:roc_auc_mean_cov}.
    \State $\widehat{F} \gets \qty(t \mapsto \Phi(\frac{t - \mu}{\sigma}))$  \Comment{$\Phi$ is the standard normal cdf}
\State \Return{$\widehat{F}$}
\end{algorithmic}
\end{algorithm}

%% file: tables/W_1_df_p_0.3_alpha_0.9.tex
\begin{tabular}{llrrrrr}
\toprule
 &  & Precision & Recall & Accuracy & $F_1$ & ROC-AUC \\
\midrule
Bootstrap $M = 0$ &  & 0.108729 & 0.089308 & 0.068415 & 0.089793 & 0.443732 \\
\addlinespace
\multirow[c]{3}{*}{PEMI $B = 5$} & $p = \frac{1}{2}$ & 0.205000 & 0.361667 & 0.386667 & 0.380833 & 0.375000 \\
 & $p$ calibrated & 0.077500 & 0.072500 & 0.055833 & 0.074167 & \bfseries 0.085000 \\
 & $p = N_+ / N$ & 0.357500 & 0.364167 & 0.340833 & 0.355833 & 0.377500 \\
\addlinespace
\multirow[c]{3}{*}{PEMI $B = 10$} & $p = \frac{1}{2}$ & 0.258750 & 0.417083 & 0.441667 & 0.411667 & 0.431667 \\
 & $p$ calibrated & 0.063750 & 0.047083 & \bfseries 0.035000 & 0.032083 & 0.096250 \\
 & $p = N_+ / N$ & 0.408333 & 0.405833 & 0.378333 & 0.412500 & 0.428333 \\
\addlinespace
\multirow[c]{3}{*}{PEMI $B = 100$} & $p = \frac{1}{2}$ & 0.285333 & 0.450958 & 0.484583 & 0.469167 & 0.481917 \\
 & $p$ calibrated & \bfseries 0.030717 & \bfseries 0.041592 & 0.060375 & \bfseries 0.020367 & 0.118675 \\
 & $p = N_+ / N$ & 0.448625 & 0.453083 & 0.423167 & 0.458667 & 0.470542 \\
\addlinespace
\multirow[c]{3}{*}{PEMI-Gauss} & $p = \frac{1}{2}$ & 0.264135 & 0.456714 & 0.484778 & 0.472469 & 0.485517 \\
 & $p$ calibrated & 0.058313 & 0.043286 & 0.042600 & 0.022275 & 0.129226 \\
 & $p = N_+ / N$ & 0.437487 & 0.456702 & 0.418207 & 0.459145 & 0.477043 \\
\bottomrule
\end{tabular}

%% file: tables/df_mae_p_0.3_alpha_0.9.tex
\begin{tabular}{llrrrrr}
\toprule
 &  & Precision & Recall & Accuracy & $F_1$ & ROC-AUC \\
\midrule
Bootstrap $M = 0$ &  & 0.018494 & 0.015045 & 0.011883 & 0.013720 & 0.098082 \\
\addlinespace
\multirow[c]{3}{*}{PEMI $B = 5$} & $p = \frac{1}{2}$ & 0.044845 & 0.103826 & 0.074209 & 0.092816 & 0.117412 \\
 & $p$ calibrated & 0.015533 & \bfseries 0.010099 & 0.008119 & 0.010729 & 0.010893 \\
 & $p = N_+ / N$ & 0.063529 & 0.062823 & 0.035999 & 0.061992 & 0.076707 \\
\addlinespace
\multirow[c]{3}{*}{PEMI $B = 10$} & $p = \frac{1}{2}$ & 0.044080 & 0.102523 & 0.074850 & 0.091197 & 0.117487 \\
 & $p$ calibrated & 0.015096 & 0.011005 & 0.008198 & 0.009548 & 0.011014 \\
 & $p = N_+ / N$ & 0.063623 & 0.061134 & 0.035974 & 0.061332 & 0.077073 \\
\addlinespace
\multirow[c]{3}{*}{PEMI $B = 100$} & $p = \frac{1}{2}$ & 0.045993 & 0.102531 & 0.074734 & 0.092006 & 0.117845 \\
 & $p$ calibrated & 0.014831 & 0.010373 & \bfseries 0.007727 & 0.009215 & \bfseries 0.010091 \\
 & $p = N_+ / N$ & 0.062079 & 0.060404 & 0.035468 & 0.062175 & 0.077259 \\
\addlinespace
\multirow[c]{3}{*}{PEMI-Gauss} & $p = \frac{1}{2}$ & 0.045475 & 0.102887 & 0.074606 & 0.091768 & 0.117791 \\
 & $p$ calibrated & \bfseries 0.014629 & 0.010520 & 0.007762 & \bfseries 0.009094 & 0.010327 \\
 & $p = N_+ / N$ & 0.062586 & 0.060855 & 0.035542 & 0.061928 & 0.077555 \\
\bottomrule
\end{tabular}

%% file: disclaimer.tex
\paragraph{Disclaimer.}
This paper was prepared for informational purposes by the Artificial Intelligence Research group of JPMorgan Chase \& Co. and its affiliates (``JP Morgan'') and is not a product of the Research Department of JP Morgan. JP Morgan makes no representation and warranty whatsoever and disclaims all liability, for the completeness, accuracy or reliability of the information contained herein. This document is not intended as investment research or investment advice, or a recommendation, offer or solicitation for the purchase or sale of any security, financial instrument, financial product or service, or to be used in any way for evaluating the merits of participating in any transaction, and shall not constitute a solicitation under any jurisdiction or to any person, if such solicitation under such jurisdiction or to such person would be unlawful.

%% file: appendix_misc.tex
\section{Probabilistic Queries}
\label{sec:prob_queries}

Within this work we invoke the concept of a probabilistic query, which we use in the same manner as~\citet{Mohan2014}. 
Since this is not well-known, we describe it briefly here.

Suppose we have random variables $X_1, X_2, \ldots$.
A \emph{primitive probabilistic query} $P : X_1, X_2, \ldots \to [0, 1]$ describes a new random variable corresponding to a question we wish to ask about the random variables, that has a probability as an answer.
For example, $P(X_1 = X_2)$ is the query corresponding to the question \emph{``what is the probability that r.v. $X_1$ equals r.v. $X_2$?''}.
We can combine primitive probabilistic queries using elementary algebra to obtain a general probabilistic query, e.g., 
$$\frac{P(X_1 = X_2)}{P(X_2 < X_3)}.$$

Primitive probabilistic queries relate to \emph{probabilities} in that under a distribution $\mathcal{D}$, the expectation of the primitive query $P$ equals the probability $\mathbb{P}$, that is,
$$\mathbb{E}_{\mathcal{D}}[P(E)] = \mathbb{P}_{\mathcal{D}}[E] \text{ for all events } E.$$

This is a notational convenience to distinguish the mathematical form of a probabilistic question $P(E)$ from the actual realised numerical probability $\mathbb{P}[E]$, allowing the same query to be evaluated under different distributions.
For example, consider distributions $\mathcal{D}_1$ and $\mathcal{D}_2$ such that $\mathcal{D}_1 \neq \mathcal{D}_2$. The probabilistic query $P(X_1 = X_2)$ can be discussed as an object separate to the probabilities $\mathbb{P}_{\mathcal{D}_1}[ X_1 = X_2] \neq \mathbb{P}_{\mathcal{D}_2}[X_1 = X_2]$.

%% file: appendix_proofs.tex
\section{Proofs and Ancillary Results}
\label{sec:proofs}

For a positive integer $n \in \mathbb{N}$, let $[n] = \{1, 2, \ldots, n\}$.
A function $f(x) = O(g(x))$ if there exist $B > 0$, $x_0 \in \mathbb{R}$ such that $\abs{f(x)} \leq B g(x)$ for all $x \geq x_0$.

\subsection{Vanilla Accuracy Estimator}

\begin{proposition}\label{prop:Q_acc}
The vanilla accuracy estimator, 
$$\widehat{Q}_{\text{acc},n} = \frac{1}{n}\sum_{i \in [n]} \mathds{1}\{y_i = \psi(x_i)\},$$
is both an unbiased and consistent estimator of the performance metric $\mathcal{Q}_{\text{acc}} = \mathbb{E}_{\mathcal{D}}[P(\psi(X) = Y)]$.
\end{proposition}
\begin{proof}
\emph{Unbiasedness.}
\begin{align*}
\mathbb{E}\qty[\widehat{Q}_{\text{acc},n}] &= \mathbb{E}\qty[\frac{1}{n}\sum_{i \in [n]} \mathds{1}{\{y_i = \psi(x_i)\}}] \\
&= \frac{1}{n}\sum_{i \in [n]} \mathbb{E}\qty[\mathds{1}{\{y_i = \psi(x_i)\}}] \tag{linearity of expectation}\\
&= \frac{1}{n}\sum_{i \in [n]} \mathbb{P}\qty[Y = \psi(X)] \\
&= \frac{1}{n}\sum_{i \in [n]} \mathbb{E}\qty[P(Y = \psi(X))] \tag{Definition of probabilistic query}\\
&= \mathbb{E}[P(Y = \psi(X))] = \mathcal{Q}_{\text{acc}}.
\end{align*}
\emph{Consistency.}
\begin{align*}
\mathbb{V}\qty[\widehat{Q}_{\text{acc},n}] &= \mathbb{V}\qty[\frac{1}{n}\sum_{i \in [n]} \mathds{1}{\{y_i = \psi(x_i)\}}] \\
&= \frac{1}{n^2}\sum_{i \in [n]} \mathbb{V}\qty[\mathds{1}{\{y_i = \psi(x_i)\}}] \tag{$Y_i$ are mutually independent}\\
&= \frac{1}{n^2}\sum_{i \in [n]} \mathbb{P}[Y = \psi(X)](1 - \mathbb{P}[Y = \psi(X)]) \\
&=\frac{1}{n}\mathbb{P}[Y = \psi(X)](1 - \mathbb{P}[Y = \psi(X)]).
\end{align*}
As $\mathbb{V}\qty[\widehat{Q}_{\text{acc},n}] = O(1/n)$, we have that $\widehat{Q}_{\text{acc},n}$ is a consistent estimator of $\mathcal{Q}_{\text{acc}}$ from Chebyshev's inequality.
\end{proof}

\subsection{Confusion Matrix Estimators}

\begin{lemma}
\label{lem:cm_mean_cov}
Suppose we have missing evaluation data ($\overline{\mathcal{K}} \neq \emptyset$).
The mean and covariance of the confusion matrix element estimators $(\widehat{\text{TP}}, \widehat{\text{FN}}, \widehat{\text{FP}}, \widehat{\text{TN}})$ are 
\NiceMatrixOptions
  {
    code-for-first-col = {\color{blue}\scriptstyle} ,
    code-for-first-row = {\color{blue}\scriptstyle} 
  }
\begin{align*}
    \mu_{\text{CM}} &=
    \begin{pNiceMatrix}[first-col,first-row]
      &  &  \\
    \text{TP} & \sum_{i \in \mathcal{K}} y_i \psi_i + \sum_{i \in \overline{\mathcal{K}}} p_i \psi_i\\
    \text{FN} & \sum_{i \in \mathcal{K}} y_i \overline{\psi_i} + \sum_{i \in \overline{\mathcal{K}}} p_i \overline{\psi_i} \\
    \text{FP} & \sum_{i \in \mathcal{K}} \overline{y_i} \psi_i + \sum_{i \in \overline{\mathcal{K}}} (1 - p_i) \psi_i \\
    \text{TN} & \sum_{i \in \mathcal{K}} \overline{y_i} \overline{\psi_i} + \sum_{i \in \overline{\mathcal{K}}} (1 - p_i) \overline{\psi_i} \\
    \end{pNiceMatrix}
    ,\\
    \Sigma_{\text{CM}} &= 
    \begin{pNiceMatrix}[first-col,first-row]
      & \text{TP}  & \text{FN}  & \text{FP} & \text{TN} \\
    \text{TP} &  \sum_{i \in \overline{\mathcal{K}}} p_i(1 - p_i) \psi_i & 0 & -\sum_{i \in \overline{\mathcal{K}}} p_i(1 - p_i) \psi_i & 0 \\
    \text{FN}  & 0 & \sum_{i \in \overline{\mathcal{K}}} p_i(1 - p_i) \overline{\psi_i} & 0 & - \sum_{i \in \overline{\mathcal{K}}} p_i(1 - p_i) \overline{\psi_i} \\
    \text{FP} & - \sum_{i \in \overline{\mathcal{K}}} p_i(1 - p_i) \psi_i & 0 & \sum_{i \in \overline{\mathcal{K}}} p_i(1 - p_i) \psi_i & 0 \\
    \text{TN} & 0 & - \sum_{i \in \overline{\mathcal{K}}} p_i(1 - p_i) \overline{\psi_i} & 0 & \sum_{i \in \overline{\mathcal{K}}} p_i(1 - p_i) \overline{\psi_i} \\
    \end{pNiceMatrix},
\end{align*}
where for a variable $z \in \{0, 1\}$, $\overline{z} = 1 - z$.
\end{lemma}
\begin{proof}
Each element is a straightforward consequence of the following: linearity of expectation; covariance of linear combinations; the fact that $\mathbb{E}[Y] = p$, $\mathbb{V}[Y] = p(1 - p)$ for $Y \sim \mathcal{B}(p)$; $\operatorname{Cov}[U, V] = 0$ for independent r.v.s $U, V$.   
\end{proof}

\subsection{Sums of Bernoulli Variables}

We shall require the following classic result.

\begin{theorem}[Berry-Esseen, \citep{Esseen1956, shevtsova2010improvement}]
    Let $T_1, T_2, \ldots$ be independent random variables with $\mathbb{E}[T_i] = 0$, $\mathbb{E}[T_i^2] = \sigma_i^2 > 0$ and $\mathbb{E}[\abs{T_i}^3] = \rho_i < \infty$. 
    Also, let
    $$S_n = \frac{T_1 + T_2 + \cdots + T_n}{\sqrt{\sigma_1^2 + \sigma_1^2 + \cdots + \sigma_n^2}}$$
    and $F_{S_n}$ be the cdf of $S_n$.
    Then, the Berry-Esseen theorem states there exists a universal constant $C_0 \in [0.4097,  0.5600]$ such that $\mathrm{d}_{\mathsf{KS}}(F_{S_n}, \mathcal{N}(0, 1)) \leq C_0 \psi_0,$
    where
    $$\psi_0 = \qty(\sum^n_{i = 1} \sigma_i^2)^{-3/2} \cdot \sum^n_{i = 1} \rho_i .$$
\end{theorem}

\subsubsection{Proof of Lemma~\ref{lem:bernoulli_sum}}

\bernoullisum*
\begin{proof}
    We proceed with a direct application of the Berry-Esseen theorem.
    First define the random variables $T_i = \frac{a_i}{\sigma_z}(Y_i - p_i)$ and $S_n = \sum_{i = 1}^n T_i$.
    We then have 
    \begin{equation}
        \mathbb{E}[T_i] = 0,\ \mathbb{E}[T_i^2] = \frac{a_i^2}{\sigma_z^2} p_i (1 - p_i),\ \mathbb{E}[\abs{T_i}^3] = \frac{a_i^3}{\sigma_z^3} p_i (1 - p_i) (1 - 2p_i + 2p_i^2).
    \end{equation}
    Towards applying Berry-Esseen, we then evaluate
    \begin{align}
        \psi_0 &= \qty(\sum_{i = 1}^n \frac{a_i^2}{\sigma_z^2} p_i (1 - p_i))^{-3/2} \cdot \sum_{i = 1}^n \frac{a_i^3}{\sigma_z^3} p_i (1 - p_i) (1 - 2p_i + 2p_i^2) \\
        &= \qty(\frac{1}{\sigma_z^2})^{-3/2}\qty(\sum_{i = 1}^n a_i^2 p_i (1 - p_i))^{-3/2} \cdot \qty(\frac{1}{\sigma_z^3}) \sum_{i = 1}^n a_i^3 p_i (1 - p_i) \qty(\frac{1}{2} + 2(p_i - \frac{1}{2})^2) \tag{Factorising} \\ 
        &= \frac{1}{2}\qty(\sum_{i = 1}^n a_i^2 p_i (1 - p_i))^{-1/2} + 2\cdot \frac{\sum_{i = 1}^n a_i^3 p_i (1 - p_i) (p_i - \frac{1}{2})^2}{\qty(\sum_{i = 1}^n a_i^2 p_i (1 - p_i))^{3/2}} \tag{Splitting sum and $\sigma_z$ cancellation} \\
        &\leq \frac{1}{2}\qty(\sum_{i = 1}^n a_i^2 p_i (1 - p_i))^{-1/2} + \frac{1}{2}\cdot \frac{\sum_{i = 1}^n a_i^3 p_i (1 - p_i) }{\qty(\sum_{i = 1}^n a_i^2 p_i (1 - p_i))^{3/2}} \tag{Since $(p_i - \frac{1}{2})^2 \leq \frac{1}{4}$} \\ 
        &\leq \frac{1}{2}\qty(\sum_{i = 1}^n a_i^2 p_i (1 - p_i))^{-1/2} + \frac{a^{*}}{2}\cdot \frac{\sum_{i = 1}^n a_i^2 p_i (1 - p_i) }{\qty(\sum_{i = 1}^n a_i^2 p_i (1 - p_i))^{3/2}}  \tag{Since $a_i \leq a^{*}$ for all $i \in [n]$}\\
        &= \frac{1 + a^{*}}{2}\qty(\sum_{i = 1}^n a_i^2 p_i (1 - p_i))^{-1/2} \\
        &\leq \frac{1 + a^{*}}{2}\qty(\sum_{i = 1}^n a_{*}^2 p_i (1 - p_i))^{-1/2} \tag{Since $a_i^2 \geq a_{*}^2$ for all $i \in [n]$} \\
        &= \frac{1 + a^{*}}{2 a_{*}}\qty(\sum_{i = 1}^n  p_i (1 - p_i))^{-1/2}.\label{eq:bernoulli_1} 
    \end{align}
    Consider that
    \begin{align}
    \sum_{i = 1}^n p_i (1 - p_i)\geq \sum_{i = 1}^n v_{*} =n v_{*} \\
    \Longleftrightarrow \qty(\sum_{i = 1}^n p_i (1 - p_i))^{- 1/2} \leq (n v_{*})^{- 1/2} \label{eq:bernoulli_2}
    \end{align}
    Substituting~\eqref{eq:bernoulli_1}~and~\eqref{eq:bernoulli_2} into the Berry-Esseen theorem yields $d_{\mathsf{KS}}(F_{S_n}, \mathcal{N}(0, 1)) \leq \frac{C_0}{\sqrt{n v_{*} }} \frac{1 + a^{*}}{2 a_{*}}$.
    This implies $\abs{\mathbb{P}(S_n \leq t) - \Phi(t)} \leq \frac{C_0}{\sqrt{n v_{*}}} \frac{1 + a^{*}}{2 a_{*}}$ for all $t \in \mathbb{R}$. 
    Choosing $t = \frac{z - \mu_z}{\sigma_z}$ for any $z \in \mathbb{R}$ we have that
    $$ \abs{\mathbb{P}(S_n \leq \frac{z - \mu_z}{\sigma_z} ) - \Phi(\frac{z - \mu_z}{\sigma_z})} = \abs{\mathbb{P}(Z \leq z ) - \Phi(\frac{z - \mu_z}{\sigma_z})}$$
    since
    \begin{align}
        S_n \leq  \frac{z - \mu_z}{\sigma_z} &\Longleftrightarrow \sum_{i = 1}^n \frac{a_i}{\sigma_z}(X_i - p_i) \leq  \frac{z - \mu_z}{\sigma_z} \\
        &\Longleftrightarrow \sum_{i = 1}^n \frac{a_i}{\sigma_z} X_i - \sum_{i = 1}^n \frac{a_i}{\sigma_z} p_i \leq  \frac{z - \mu_z}{\sigma_z} \\
        &\Longleftrightarrow \frac{Z - \mu_z}{\sigma_z} \leq  \frac{z - \mu_z}{\sigma_z} \Longleftrightarrow Z \leq z
    \end{align}
    and the result follows.
\end{proof}

\subsection{Ratio of Correlated Gaussians}

\subsubsection{Proof of Theorem~\ref{thm:ratio_gauss}}

We require the following ancillary result for proving the convergence bound of a ratio of correlated normal random variables.

\input{taylor_frac}

\ratiogauss*

\begin{proof}
We proceed via the standard Geary-Hinkley transformation~\citep{Hayya1975ANO}, then bound the Kolmogorov-Smirnov distance from a Normal distribution using an upper bound on the standard Normal CDF and Local-Lipschitz type argument.

Using the axioms of probability we have
$$G(t) := \mathbb{P}(Z/W \leq t) = \mathbb{P}(Z - tW < 0, W > 0) + \mathbb{P}(Z - tW \geq 0, W \leq 0)$$
$$ = \mathbb{P}(Z - tW < 0 \vert W > 0) \mathbb{P}(W > 0) + \mathbb{P}(Z - tW \geq 0 \vert W \leq 0) \mathbb{P}(W \leq 0) $$
$$ \to \mathbb{P}(Z - tW < 0 \vert W > 0) \mathbb{P}(W > 0) \to \mathbb{P}(Z - tW < 0)$$
where we invoke the condition $\mathbb{P}(W > 0) \to 1$. 

The random variable $D := Z - tW$ is itself a normal distribution with parameters $\mu_d = \mu_z - t \mu_w$, $\sigma_d^2 = \sigma_z^2 - 2 t c + t^2 \sigma_w^2$, where $c := \rho \sigma_z \sigma_w $.
Thus we have 
$$ \mathbb{P}(Z - tW < 0) = \Phi(\frac{-\mu_d}{\sigma_d}) = \Phi\left( \frac{t - \frac{\mu_z}{\mu_w}}{ \mu_w^{-1} \sqrt{\sigma_z^2 - 2 t c + t^2 \sigma_w^2}} \right) = \Phi\left( \frac{t - \mu}{ s(t) } \right),$$
where we have defined $s(t) := \mu_w^{-1} \sqrt{\sigma_z^2 - 2 t c + t^2 \sigma_w^2}$.


We evaluate 
 
$$\abs{ \mathbb{P}(Z/W \leq t) - \Phi\qty(\frac{t - \mu}{\sigma}) } = \left\vert \Phi\left( \frac{t - \mu}{ s(t) } \right) - \Phi\qty(\frac{t - \mu}{\sigma}) \right\vert$$

$$= \left\vert \frac{1}{\sqrt{2 \pi}} \int_{-\infty}^{(t - \mu) / \sigma}\exp(-\frac{z^2}{2})\,\mathrm{d}z - \frac{1}{\sqrt{2 \pi}} \int_{-\infty}^{(t - \mu) / s(t)}\exp(-\frac{z^2}{2})\,\mathrm{d}z \right\vert$$

$$ = \left\vert \frac{1}{\sqrt{2 \pi}} \int_{(t - \mu) / s(t)}^{(t - \mu) / \sigma}\exp(-\frac{z^2}{2})\,\mathrm{d}z \right\vert = \left\vert \frac{1}{\sqrt{2 \pi}} \int_{(t - \mu) / s(t)}^{(t - \mu) / \sigma} \frac{1}{z} \cdot z\exp(-\frac{z^2}{2})\,\mathrm{d}z \right\vert $$

$$=\left\vert \frac{1}{\sqrt{2 \pi}} \left(\frac{-1}{z} \cdot \exp(-\frac{z^2}{2}) \right\rvert_{(t - \mu) / s(t)}^{(t - \mu) / \sigma} + \frac{1}{\sqrt{2 \pi}} \int_{(t - \mu) / s(t)}^{(t - \mu) / \sigma} \frac{1}{z} \exp(-\frac{z^2}{2})\,\mathrm{d}z \right\vert \leq \left\vert \frac{1}{\sqrt{2 \pi}} \left(\frac{-1}{z} \cdot \exp(-\frac{z^2}{2}) \right\rvert_{(t - \mu) / s(t)}^{(t - \mu) / \sigma}  \right\vert$$
where we use integration-by-parts.
Evaluating at the limits gives
\begin{align}\label{eq:int_ub}
 \abs{ \mathbb{P}(Z/W \leq t) - \Phi\qty(\frac{t - \mu}{\sigma}) } &= \frac{1}{\sqrt{2 \pi}} \left\vert  \frac{s(t)}{t- \mu}\exp(-\frac{(t - \mu)^2}{2 s(t)^2}) - \frac{\sigma}{t- \mu}\exp(-\frac{(t - \mu)^2}{2 \sigma^2}) \right\vert \notag \\
 &\leq \frac{1}{\sqrt{2 \pi}} \frac{1}{\vert t- \mu \vert} \left\vert  s(t)\exp(-\frac{(t - \mu)^2}{2 s(t)^2}) - \sigma\exp(-\frac{(t - \mu)^2}{2 \sigma^2}) \right\vert \notag \\
 &= \frac{1}{\sqrt{2 \pi}} \frac{1}{\vert t- \mu \vert} \left\vert f(t) \right\vert
\end{align}
where we have defined $f(t) :=  s(t)\exp(-(t - \mu)^2/2 s(t)^2) - \sigma\exp(-(t - \mu)^2/2 \sigma^2) $.
We will consider the Taylor expansion of $s(t)$ and substitute into the definition of $f(t)$ to establish our bounds. 
Consider now the Taylor expansion of $s(t)$ around $t = \mu$, whence
$$s(t) = \sigma + s_1(t - \mu) + s_2(t - \mu)^2 + \cdots $$
Using Lemma~\ref{lem:taylor_frac}, we expand $(t - \mu) / s(t)$ as
\begin{align}\label{eq:t_minus_mu_over_s}
\frac{t - \mu}{s(t)} &= \frac{t - \mu}{\sigma} - \frac{(t - \mu) s_1}{\sigma^2}(t - \mu) + (t - \mu) \qty( \frac{s_1^2}{\sigma^3} - \frac{s_2}{\sigma^2} )(t - \mu)^2 + O((t - \mu)^3) \notag\\
&= \frac{t - \mu}{\sigma} - \frac{s_1}{\sigma^2}(t - \mu)^2 + O((t - \mu)^3).
\end{align}
Recall the power series 
\begin{equation}\label{eq:gauss_power_series}
    \exp(- \frac{z^2}{2}) = \sum_{k=0}^{\infty} \frac{(-1)^{k} z^{2k}}{2^k k!} = 1 - \frac{z^2}{2} + \frac{z^4}{8} - \cdots
\end{equation}
which converges for all $z \in (-\infty, \infty)$.
We substitute~\eqref{eq:t_minus_mu_over_s} into~\eqref{eq:gauss_power_series} yielding
\begin{align}
    \exp( - \frac{(t - \mu)^2}{2 s(t)^2} ) = 1 - \frac{1}{2} \qty[ \frac{t - \mu}{\sigma} - \frac{s_1}{\sigma^2}(t - \mu)^2 + O((t - \mu)^3) ] + O((t - \mu)^4). 
    \end{align}
Multiplying by the Taylor expansion for $s(t)$ yields
\begin{align}\label{eq:f_exp_s_term}
        s(t) \exp( - \frac{(t - \mu)^2}{2 s(t)^2} )
    &=
    \qty(\sigma + s_1(t - \mu) + s_2(t - \mu)^2 + O((t - \mu)^3) )  \\ 
    &\cdot \qty(1 - \frac{1}{2} \qty[ \frac{t - \mu}{\sigma} - \frac{s_1}{\sigma^2}(t - \mu)^2 + O((t - \mu)^3) ] + O((t - \mu)^4) ) \\
    &=
    \sigma \qty[ 1 - \frac{1}{2}\qty(\frac{t - \mu}{\sigma})^2 ] + s_1 (t - \mu) + s_2(t - \mu)^2 + O((t - \mu)^3)
\end{align}
Similarly, we have using~\eqref{eq:gauss_power_series} that
\begin{equation}\label{eq:f_exp_sigma_term}
    \sigma \exp( \frac{-(t - \mu)^2}{2 \sigma^2} ) = \sigma \qty[ 1 - \frac{1}{2}\qty(\frac{t - \mu}{\sigma})^2 ] + O((t - \mu)^4).
\end{equation}
Subtracting~\eqref{eq:f_exp_sigma_term}~from~\eqref{eq:f_exp_s_term} yields
\begin{equation}\label{eq:f_exp}
    f(t) = s_1(t - \mu) + s_2(t - \mu)^2 + O((t - \mu)^3).
\end{equation}
The function $s(t)$ is the square root of a quadratic function, and therefore is very close to a linear function.
As such we expect the quadratic terms of $s(t)$ to dominate, that is, we expect in the neighbourhood of $t = \mu$
$$\abs{s_0 + s_1(t - t_0) + s_2(t - t_0)^2} \geq \abs{\sum_{k=3}^\infty s_k (t - t_0)^k }.$$
Moreover, as $f(t)$ is a smooth function of $s(t)$ and $t$ we expect this property to hold for the expansion of $f(t)$ in~\eqref{eq:f_exp}, that is, in the neighbourhood of $t = \mu$
\begin{equation}\label{eq:f_exp_dominate}
    \abs{f_0 + f_1(t - \mu) + f_2(t - \mu)^2} \geq \abs{\sum_{k=3}^\infty f_k (t - \mu)^k },
\end{equation}
where $f(t) = \sum_{k=0}^\infty f_k (t - \mu)^k$ is the Taylor expansion of $f(t)$ around $t = \mu$.
Indeed, one can verify that $f_0 = f(\mu) = 0$ and $f_1 = s_1$, $f_2 = s_2$.

Thus we have
\begin{align}
    \abs{f(t)} &= \abs{s_1(t - \mu) + s_2(t - \mu)^2 + \sum_{k=3}^\infty f_k (t - \mu)^k} \\
    &\leq \abs{s_1(t - \mu) + s_2(t - \mu)^2} + \abs{\sum_{k=3}^\infty f_k (t - \mu)^k} \tag{triangle inequality} \\
    &\leq 2\abs{s_1(t - \mu) + s_2(t - \mu)^2} \tag{Using~\eqref{eq:f_exp_dominate}}
\end{align}




It can be shown by direct differentiation that 
$$s_1 = \frac{\sqrt{{\mu_z}^{2} {\sigma_w}^{2} + {\mu_w}^{2} {\sigma_z}^{2} - 2 \, c {\mu_w} {\mu_z}} {\left({\mu_z} {\sigma_w}^{2} - c {\mu_w}\right)}}{{\mu_w} {\mu_z}^{2} {\sigma_w}^{2} + {\mu_w}^{3} {\sigma_z}^{2} - 2 \, c {\mu_w}^{2} {\mu_z}} = \frac{{\mu_z} {\sigma_w}^{2} - c {\mu_w}}{\sqrt{{\mu_z}^{2} {\sigma_w}^{2} + {\mu_w}^{2} {\sigma_z}^{2} - 2 \, c {\mu_w} {\mu_z}} {\mu_w}}
 =  \frac{\mu_z \sigma_w^2 - c \mu_w}{ \mu_w^3 \sigma}$$
\begin{multline*}
  s_2 = \frac{{\left({\mu_w}^{2} {\sigma_w}^{2} {\sigma_z}^{2} - c^{2} {\mu_w}^{2}\right)} \sqrt{{\mu_z}^{2} {\sigma_w}^{2} + {\mu_w}^{2} {\sigma_z}^{2} - 2 \, c {\mu_w} {\mu_z}}}{2 \, {\left({\mu_z}^{4} {\sigma_w}^{4} + {\mu_w}^{4} {\sigma_z}^{4} - 4 \, c {\mu_w} {\mu_z}^{3} {\sigma_w}^{2} + 4 \, c^{2} {\mu_w}^{2} {\mu_z}^{2} + 2 \, {\left({\mu_w}^{2} {\mu_z}^{2} {\sigma_w}^{2} - 2 \, c {\mu_w}^{3} {\mu_z}\right)} {\sigma_z}^{2}\right)}} \\ = \frac{{\left({\sigma_w} {\sigma_z} + c\right)} {\left({\sigma_w} {\sigma_z} - c\right)} {\mu_w}^{2}}{2 \, {\left({\mu_z}^{2} {\sigma_w}^{2} + {\mu_w}^{2} {\sigma_z}^{2} - 2 \, c {\mu_w} {\mu_z}\right)}^{\frac{3}{2}}} 
= \frac{\sigma_w^2  \sigma_z^2 (1 - \rho^2)}{2 \mu_w^4 \sigma^3}  
\end{multline*}

Thus we have
\begin{align}\label{eq:f_ub_1}
    \abs{f} &\leq 2 \abs{ \frac{\mu_z \sigma_w^2 - c \mu_w}{ \mu_w^3 \sigma} (t - \mu) + \frac{\sigma_w^2  \sigma_z^2 (1 - \rho^2)}{2 \mu_w^4 \sigma^3} (t - \mu)^2 } \\
    &\leq 2 \abs{ \frac{\mu_z \sigma_w^2 - c \mu_w}{ \mu_w^3 \sigma} (t - \mu)} + 2 \abs{\frac{\sigma_w^2  \sigma_z^2 (1 - \rho^2)}{2 \mu_w^4 \sigma^3} (t - \mu)^2 } \tag{triangle inequality} \\
    &=
    2 \abs{ \frac{\mu_z \sigma_w^2 - c \mu_w}{ \mu_w^3 \sigma} (t - \mu)} + 2 \abs{\frac{\sigma_w^2  \sigma_z^2 (1 - \rho^2)}{\mu_w^3 \sigma} (t - \mu) } \cdot \abs{\frac{t - \mu}{2 \mu_w \sigma^2}}
\end{align}

Consider the interval $I = [ \mu - 2\abs{\mu_w} \sigma^2 , \mu + 2\abs{\mu_w} \sigma^2 ]$. 
Observe that for $t \in I$,  $\abs{\frac{t - \mu}{2 \mu_w \sigma^2}} \leq 1$.
For $t \notin I$, we have from Assumption 2 in the theorem statement that $f(t)$ is negligible, because
\begin{align*}
    f(t) \to 0 & \Longleftarrow (2 \abs{\mu_w} \sigma^2 \gg \sigma) \wedge (t \notin I) \tag{$\exp(\,\cdot\,)$ terms in $f(t)$ vanish} \\
    &  \Longleftrightarrow ( 2 \abs{\mu_w} \sigma \gg 1 ) \wedge (t \notin I)\\
    &  \Longleftrightarrow ( 2 \abs{\mu_w}^{-1} \sqrt{\mu_z^2 \sigma_w^2 + \mu_w^2 \sigma_z^2 - 2 \rho \sigma_z \sigma_w \mu_z \mu_w} \gg 1 ) \wedge (t \notin I)\tag{definition of $\sigma$}\\
    &  \Longleftarrow ( 2 \abs{\mu_w}^{-1} \sqrt{\mu_z^2 \sigma_w^2 + \mu_w^2 \sigma_z^2 - 2 \sigma_z \sigma_w \mu_z \mu_w} \gg 1 ) \wedge (t \notin I)\tag{$\rho \in [-1, 1]$} \\
    &  \Longleftrightarrow ( 2 \abs{\mu_w}^{-1} \sqrt{(\mu_z \sigma_w - \mu_w \sigma_z)^2} \gg 1 ) \wedge (t \notin I)\tag{factoring inside $\sqrt{\, \cdot\, }$} \\
    &  \Longleftrightarrow ( 2 \abs{\mu_w}^{-1} \abs{\mu_z \sigma_w - \mu_w \sigma_z} \gg 1 ) \wedge (t \notin I) \\
    &  \Longleftrightarrow ( 2 \abs{\mu_z} \abs{\frac{\sigma_w}{\mu_z} - \frac{\sigma_z}{\mu_z} } \gg 1 ) \wedge (t \notin I), \tag{Rearranging terms} \\
\end{align*}
which is precisely Assumption 2.
We substitute the statement $\abs{\frac{t - \mu}{2 \mu_w \sigma^2}} \leq 1$ into~\eqref{eq:f_ub_1} yielding
\begin{align}\label{eq:f_ub_2}
    \abs{f} &\leq \abs{ \frac{\mu_z \sigma_w^2 - c \mu_w}{ \mu_w^3 \sigma} (t - \mu)} + \abs{\frac{\sigma_w^2  \sigma_z^2 (1 - \rho^2)}{\mu_w^3 \sigma}} \notag \\
    \abs{f} &\leq \abs{ \frac{\mu_z \sigma_w^2 - c \mu_w}{ \mu_w^3 \sigma} } \abs{t - \mu} + \abs{\frac{\sigma_w^2  \sigma_z^2 (1 - \rho^2)}{\mu_w^3 \sigma}} \abs{t - \mu}
\end{align}
Substituting~\eqref{eq:f_ub_2}~into~\eqref{eq:int_ub} yields
\begin{align}
    \abs{ \mathbb{P}(Z/W \leq t) - \Phi\qty(\frac{t - \mu}{\sigma}) } & \leq \frac{2}{\sqrt{2 \pi}} \qty( \abs{ \frac{\mu_z \sigma_w^2 - c \mu_w}{ \mu_w^3 \sigma} } + \abs{\frac{\sigma_w^2  \sigma_z^2 (1 - \rho^2)}{\mu_w^3 \sigma}}) \\
    &\leq \sqrt{\frac{2}{\pi}} \frac{1}{\sigma \abs{\mu_w}^3} \qty( \abs{ \mu_z \sigma_w^2 - c \mu_w} + \sigma_w^2  \sigma_z^2 (1 - \rho^2) ) \tag{$\rho \in (-1, 1)$} \\
    &\leq \sqrt{\frac{2}{\pi}} \frac{1}{\sigma \abs{\mu_w}^3} \qty( \abs{ \mu_z \sigma_w^2 - \rho \mu_w \sigma_w \sigma_z} + \sigma_w^2  \sigma_z^2 (1 - \rho^2) ) \tag{$c = \rho \sigma_w \sigma_z$} \\
    &\leq \sqrt{\frac{2}{\pi}} \frac{1}{\sigma \abs{\mu_w}^3} \qty( \abs{ \mu_z \sigma_w^2} + \abs{\rho \mu_w \sigma_w \sigma_z} + \sigma_w^2  \sigma_z^2 (1 - \rho^2) ) \tag{triangle inequality} \\
    &= \sqrt{\frac{2}{\pi}} \frac{1}{\sigma \abs{\mu_w}^3} \qty( \abs{ \mu_z \sigma_w^2} -\sigma_w^2 \sigma_z^2 (\abs{\rho} - \frac{\mu_w}{2 \sigma_w \sigma_z})^2 + \frac{\mu_w^2}{4} + \sigma_w^2 \sigma_z^2 ) \tag{completing the square} \\
    &\leq \sqrt{\frac{2}{\pi}} \frac{1}{\sigma \abs{\mu_w}^3} \qty( \abs{ \mu_z \sigma_w^2} + \frac{\mu_w^2}{4} + \sigma_w^2 \sigma_z^2 ) \\
    &\leq \sqrt{\frac{2}{\pi}} \frac{1}{\sigma \abs{\mu_w}^3} \qty( \abs{ \mu_z \sigma_w^2} + \mu_w^2 + \sigma_w^2 \sigma_z^2 )
\end{align}
and the result follows since $t$ is arbitrary and $\mu_w > 0$ by assumption.
\end{proof}

\begin{remark}\label{rem:bounds_interp}
There is a nice interpretation of each of the terms in the upper bound of Theorem~\ref{thm:ratio_gauss}.
\begin{itemize}
    \item $\frac{1}{\mu_w^3}$ : the further the mean of the denominator from zero the better the approximation. As $\mu_w$ grows there is less and less support to the distribution of the ratio below zero.
    \item $\frac{1}{\sigma}$ : the narrower the ratio distribution, the worse the approximation. We can think of broader distributions being more ``forgiving''.
    \item  $\abs{\mu_z \sigma_w^2}$: the closer the mean of the numerator to zero, the better the approximation.
    \item $\mu_w^2$ : weakens the $\frac{1}{\mu_w^3}$ term to $\frac{1}{\mu_w}$.
    \item $\sigma_w^2 \sigma_z^2$ : penalty term for larger variances.
\end{itemize}

\end{remark}

\subsection{Ratios of Correlated Sums of Bernoullis}

\subsubsection{Proof of Theorem~\ref{thm:ratio_sum_bernoulli}}

\ratiosumbernoulli*

\begin{proof}
We wish to determine the ratio distribution $\widetilde{Z} / \widetilde{W}$.
From Lemma~\ref{lem:bernoulli_sum} we have that $\widetilde{Z} \approx Z$ and $\widetilde{W} \approx W$ are approximately Gaussian.
Moreover, for fixed $t\in \mathbb{R}$ we have that $\widetilde{Z} - t \widetilde{W}$ is close to the Gaussian distributed $Y = Z - t W$.
Indeed we have 
\begin{align}
    \abs{\mathbb{P}(\widetilde{Z} / \widetilde{W} \leq t) - \Phi(\frac{t - \mu}{\sigma})}
    &=
    \abs{\mathbb{P}(\widetilde{Z} - t \widetilde{W} \leq 0) - \Phi(\frac{t - \mu}{\sigma})} \tag{re-arranging terms in first $\mathbb{P}(\,\cdot\,)$}\\
    &= \abs{\mathbb{P}(\widetilde{Z} - t \widetilde{W} \leq 0) - \mathbb{P}(Z - t W \leq 0) +  \mathbb{P}(Z - t W \leq 0) - \Phi(\frac{t - \mu}{\sigma})} \\
    &\leq \abs{\mathbb{P}(\widetilde{Z} - t \widetilde{W} \leq 0) - \mathbb{P}(Z - t W \leq 0)} + \abs{\mathbb{P}(Z - t W \leq 0) - \Phi(\frac{t - \mu}{\sigma})} \\
    &\leq \abs{\mathbb{P}(\widetilde{Z} - t \widetilde{W} \leq 0) - \mathbb{P}(Z - t W \leq 0)} + \abs{\mathbb{P}(Z / W \leq t) - \Phi(\frac{t - \mu}{\sigma})}\label{eq:gauss_overall_1}
\end{align}

For the first term of~\eqref{eq:gauss_overall_1}, we have

\begin{align}
    &\abs{\mathbb{P}(\widetilde{Z} - t \widetilde{W} \leq 0) - \mathbb{P}(Z - t W \leq 0)} \\
    &= \abs{\mathbb{P}(\widetilde{Z} - t \widetilde{W} \leq 0) - \mathbb{P}(Z - t \widetilde{W} \leq 0) + \mathbb{P}(Z - t \widetilde{W} \leq 0) - \mathbb{P}(Z - t W \leq 0)} \\
    &\leq \abs{\mathbb{P}(\widetilde{Z} - t \widetilde{W} \leq 0) - \mathbb{P}(Z - t \widetilde{W} \leq 0)} + \abs{\mathbb{P}(Z - t \widetilde{W} \leq 0) - \mathbb{P}(Z - t W \leq 0)} \tag{triangle inequality}\\
    &= \abs{\mathbb{P}(\widetilde{Z} \leq t \widetilde{W}) - \mathbb{P}(Z \leq t \widetilde{W})} + \abs{\mathbb{P}(W \geq Z/t ) - \mathbb{P}(\widetilde{W} \geq Z/t )} \tag{rearranging inequalities in $\mathbb{P}(\,\cdot\,)$}\\
    &=\abs{\mathbb{P}(\widetilde{Z} \leq t \widetilde{W}) - \mathbb{P}(Z \leq t \widetilde{W})} + \abs{1 - \mathbb{P}(W \leq Z/t ) - (1 - \mathbb{P}(\widetilde{W} \leq Z/t ))} \\
    &=\abs{\mathbb{P}(Z \leq t \widetilde{W}) - \mathbb{P}(\widetilde{Z} \leq t \widetilde{W})} + \abs{\mathbb{P}(W < Z/t )  - \mathbb{P}(\widetilde{W} < Z/t )}\label{eq:gauss_overall_2_1}\\
    &=\abs{\mathbb{P}(Z \leq t \widetilde{W}) - \mathbb{P}(\widetilde{Z} \leq t \widetilde{W})} + \abs{\mathbb{P}(W \leq Z/t )  - \mathbb{P}(\widetilde{W} \leq Z/t )},\label{eq:gauss_overall_2_2}
\end{align}
where~\eqref{eq:gauss_overall_2_2} follows~\eqref{eq:gauss_overall_2_1} from continuity of $W$ and $\widetilde{W}$.
Lemma~\ref{lem:bernoulli_sum} provides bounds on $\abs{\mathbb{P}(Z \leq z) - \mathbb{P}(\widetilde{Z} \leq z)}$ and $\abs{\mathbb{P}(W \leq w) - \mathbb{P}(\widetilde{W} \leq w)}$ for arbitrary $z, w \in \mathbb{R}$, giving
\begin{align}\label{eq:gaussian_overall_3}
    \abs{\mathbb{P}(\widetilde{Z} - t \widetilde{W} \leq 0) - \mathbb{P}(Z - t W \leq 0)} &\leq  \frac{C_0}{\sqrt{n_a v_{*}}} \frac{1 + a^{*}}{2 a_{*}} + \frac{C_0}{\sqrt{n v_{*}}} \frac{1 + b^{*}}{2 b_{*}} \\
    &\leq \frac{C_0 }{\sqrt{n_a v_{*}}} \qty( \frac{1 + a^{*}}{2 a_{*}} + \frac{1 + b^{*}}{2 b_{*}}) 
    \leq \frac{C_0 }{\sqrt{n_a v_{*}}} \frac{1 + b^{*}}{a_{*}}
\end{align}

For the second term of~\eqref{eq:gauss_overall_1} we invoke Theorem~\ref{thm:ratio_gauss}, which yields
\begin{align}\label{eq:gaussian_overall_4}
\abs{\mathbb{P}(Z / W \leq t) - \Phi(\frac{t - \mu}{\sigma})} \leq \sqrt{\frac{2}{\pi}} \cdot \frac{\sigma_w^2 (\abs{\mu_z} + \sigma_z^2) + \mu_w^2 }{\sigma \mu_w^3},
\end{align}
and from Lemma~\ref{lem:bernoulli_sum} we have that
\begin{align}
\mu_z = \alpha + \sum_{i = 1}^n a_i p_i, \quad  \sigma_z^2 = \sum_{i = 1}^n a^2_i p_i(1 - p_i), \quad
\mu_w = \beta + \sum_{i = 1}^n b_i p_i, \quad  \sigma_w^2 = \sum_{i = 1}^n b^2_i p_i(1 - p_i). \end{align}
The quantity $\operatorname{Cov}[\widetilde{Z}, \widetilde{W}]$ is readily computed as $\sum_{i = 1}^n a_i b_i p_i(1 - p_i)$, from which we compute $\rho$ using the standard formula $\rho_{X, Y} = \operatorname{Cov}[X, Y] / (\sigma_X \sigma_Y)$.
The mean of the approximate distribution $\mu = \mu_z / \mu_w$ as given in Theorem~\ref{thm:ratio_gauss}.

Substituting~\eqref{eq:gaussian_overall_3}~and~\eqref{eq:gaussian_overall_4} into~\eqref{eq:gauss_overall_1} gives
\begin{equation}
\abs{\mathbb{P}(\widetilde{Z} / \widetilde{W} \leq t) - \Phi(\frac{t - \mu}{\sigma})} \leq \frac{C_0 }{\sqrt{n_a v_{*}}} \frac{1 + b^{*}}{a_{*}}  + \sqrt{\frac{2}{\pi}} \cdot \frac{\sigma_w^2 (\abs{\mu_z} + \sigma_z^2) + \mu_w^2 }{\sigma \mu_w^3}
\end{equation}
and the result follows since $t$ is arbitrary.
\end{proof}

\subsubsection{Theorem~\ref{thm:ratio_sum_bernoulli} Asymptotics}

We require the following lemma.

\begin{lemma}\label{lem:s_t_min}
Let $s(t) := \mu_w^{-1} \sqrt{\sigma_z^2 - 2 t c + t^2 \sigma_w^2}$, where $\mu_w, \sigma_w, \sigma_z > 0$ and $c \in [-\sigma_w \sigma_z, +\sigma_w, \sigma_z]$.
Then, $s(t) \geq s(c / \sigma_w^2) = \mu_w^{-1} \sqrt{\sigma_z^2 - c^2 / \sigma_w^2}$ for all $t \in \mathbb{R}$, that is, $c / \sigma_w^2$ minimises $s(t)$.
\end{lemma}
\begin{proof}
The first derivative of $s(t)$ is
$$\pdv{s}{t} = \frac{\sigma_w^2 t - c}{\mu_w (\sigma_z^2 - 2 t c + t^2 \sigma_w^2)^{1/2}},$$
\noindent with $\pdv{s}{t} = 0$ having unique solution $t = c / \sigma_w^2$. 
The second derivative
$$ \pdv[2]{s}{t}\Bigg\rvert_{t = c / \sigma_w^2} = \frac{\sigma_w^2 \sigma_z^2 - c^2}{\mu_w (\sigma_z^2 - 2 t c + t^2 \sigma_w^2)^{3/2}}\Bigg\rvert_{t = c / \sigma_w^2} = \frac{\sigma_w^2 \sigma_z^2 - c^2}{\mu_w (\sigma_z^2(1 - \rho^2))^{3/2}} = \frac{\sigma_w^2}{\mu_w \sigma_z \sqrt{1 - \rho^2}} > 0$$
from which we deduce $t = c / \sigma_w^2$ minimises $s(t)$.
\end{proof}

\begin{corollary}[$O(n^{-1/2})$ scaling]
\label{cor:o_n_half_bound}
If in the assumptions of Theorem~\ref{thm:ratio_sum_bernoulli} we further impose the following:
\begin{itemize}
    \item $b_i = 1$ for all $i \in [n]$
    \item $a_i = 0$ for all $i \in [n]$, apart from a constant fraction $r \in (0, 1)$ of $a_i$ that are equal to one, that is, $n_a = \lfloor r n \rfloor$.
    \item $p_i = p$ for all $i \in [n]$
\end{itemize}
then we have that the ratio distribution $\widetilde{G}(t)
 = \mathbb{P} (\widetilde{Z} / \widetilde{W} \leq t)$ satisfies $\mathrm{d}_{\textsf{KS}}(\widetilde{G}, \mathcal{N}(\mu, \sigma^2) ) \sim O(n^{-1/2})$ as $n$ grows large.
\end{corollary}
\begin{proof}
We shall directly use the bounds of Theorem~\ref{thm:ratio_sum_bernoulli}.
We have that
\begin{equation}\label{eq:cor_vars}
\begin{aligned}
\mu_z &= \alpha + n_a p, \quad  \sigma_z^2 = n_a p (1 - p), \\
\mu_w &= \beta + n p, \quad  \sigma_w^2 = n p (1 - p), \\
\rho &= \textstyle \frac{1}{\sqrt{n n_a} p (1 - p)} n_a p (1 - p), \\
a_* &= b^* = 1  \quad v_* = p (1 - p)
\end{aligned}
\end{equation}
Moreover,
\begin{align}
\sigma &\geq \mu_w^{-1} \sqrt{\sigma_z^2 - \frac{c^2}{\sigma_w^2}} \tag{Lemma~\ref{lem:s_t_min}} \\
&=\frac{n_a}{n p} \sqrt{p (1 - p) \qty(1 - \frac{1}{n^2})} \geq \frac{r (n -1)}{n p} \sqrt{p (1 - p) \qty(1 - \frac{1}{n^2})} \tag{Using Eq.~\eqref{eq:cor_vars}}
\end{align}
Therefore, we have
\begin{equation}\label{eq:cor_bound_1}
\frac{1}{\sigma \mu_w^3} \leq \frac{1}{(\beta + n p)^3} \frac{n p}{r (n - 1)} \qty(p (1 - p) \qty(1 - \frac{1}{n^2}))^{-\frac{1}{2}} \underset{n \to \infty}{\longrightarrow} \frac{1}{n^3 p^2 r} \qty(p (1 - p) )^{-\frac{1}{2}}
\end{equation}
Moreover, we have
\begin{align}\label{eq:cor_bound_2}
    \sigma_w^2 (\abs{\mu_z} + \sigma_z^2) + \mu_w^2 &= n p (1 - p) \qty[ \lfloor n r\rfloor p + \lfloor n r \rfloor p (1 - p) ] + n^2 p^2 (1 - p)^2 \notag \\
    &\leq n^2 p^2 r + n^2 p^2(1 - p)^2 (1 + r) \\
    &\leq 3 n^2 p^2 \tag{$0 < r, p < 1$}
\end{align}
Eqs.~\eqref{eq:cor_bound_1}~and~\eqref{eq:cor_bound_2} combined imply
\begin{equation}\label{eq:cor_bound_3}
    \frac{\sigma_w^2 (\abs{\mu_z} + \sigma_z^2) + \mu_w^2}{\sigma \mu_w^3} \underset{n \to \infty}{\longrightarrow} O(n^{-1}).
\end{equation}

Direct substitution of Eqs.~\eqref{eq:cor_bound_3}~and~\eqref{eq:cor_vars} into Theorem~\ref{thm:ratio_sum_bernoulli} gives
\begin{align}
\mathrm{d}_{\textsf{KS}}(F_{\widetilde{Z} / \widetilde{W}}, \mathcal{N}(\mu, \sigma^2) ) \leq 
\frac{2C_0}{\sqrt{r n p (1 - p)}} + O(n^{-1}) \underset{n \to \infty}{\longrightarrow} O(n^{-1/2}),
\end{align}
which follows since $C_0$, $r$ and $p$ are constants.
\end{proof}

\subsection{Robustness}

\subsubsection{Proof of Theorem~\ref{thm:robustness}}
\robustness*
\begin{proof}
We shall demonstrate that if $\mathbb{V}[P_i] < v_*$, then $X_i \sim \mathcal{B}(p_i)$ for all $i \in [n]$ which is the exact setting for Theorem~\ref{thm:ratio_sum_bernoulli}.

A $\operatorname{Beta}(\alpha, \beta)$-distributed random variable $P$ can be alternatively specified by its expectation and variance, whence
\begin{equation}\label{eq:beta_dist_moments}
\alpha = \qty( \frac{\mathbb{E}[P](1 - \mathbb{E}[P])}{\mathbb{V}[P]} - 1 ) \mathbb{E}[P], \quad \beta = \qty( \frac{\mathbb{E}[P](1 - \mathbb{E}[P])}{\mathbb{V}[P]} - 1 ) (1 - \mathbb{E}[P]).
\end{equation}
Suppose we impose that $\mathbb{E}[P] = p$ for some $p \in (0, 1)$ and let $\mathbb{V}[P] = v$.
We thus have from~\eqref{eq:beta_dist_moments} that
\begin{equation}\label{eq:beta_dist_moments_2}
    \alpha = \qty( \frac{p(1 - p)}{v} - 1) p, \quad \beta = \qty( \frac{p(1 - p)}{v} - 1) (1 - p).  
\end{equation}
A $\operatorname{Beta}(\alpha, \beta)$ distribution is well-formed when both $\alpha, \beta > 0$.
Substitution into~\eqref{eq:beta_dist_moments_2} then specifies that $P$ follows a valid $\operatorname{Beta}(\alpha, \beta)$ distribution when 
$$\qty( \frac{p(1 - p)}{v} - 1) > 0 \Longleftrightarrow p(1 - p) > v.$$
By definition, $v_{*} = \min_i \{p_i (1 - p_i)\}$ and so we have that the distributions $P_i \sim \operatorname{Beta}(\alpha_i, \beta_i)$ are all well formed when $\mathbb{V}[P_i] < v_*$.

It is well-known that the distribution $Y \sim \mathcal{B}(P)$, where $P \sim \operatorname{Beta}(\alpha, \beta)$ is identical to $Y \sim \mathcal{B}(\frac{\alpha}{\alpha + \beta})$.
We also have that $\mathbb{E}[P] = \frac{\alpha}{\alpha + \beta}$. 
By assumption, we impose that $\mathbb{E}[P_i] = p_i$ for all $i \in [n]$, giving that $\frac{\alpha_i}{\alpha_i + \beta_i} = p_i$.
Thus, $Y_i \sim \mathcal{B}(p_i)$ which is precisely the setup of Theorem~\ref{thm:ratio_sum_bernoulli}.
\end{proof}

\begin{remark}
Theorem~\ref{thm:robustness} is a relaxation of Theorem~\ref{thm:ratio_sum_bernoulli} because for $P_i \sim \operatorname{Beta}(\alpha_i, \beta_i)$ such that $\mathbb{E}[P_i] = p_i$, in the limit $\mathbb{V}[P_i] \to 0$, we have that $P_i \sim \mathds{1}\{P_i = p_i\}$.
\end{remark}

\subsection{Estimators, Mean and Covariance for ROC-AUC}

The probability query corresponding to ROC-AUC is given in~\eqref{eq:roc_auc_query} as
$$P(\hat{y}(X) \geq \hat{y}(X') \mid Y = 1,\ Y' = 0).$$
We can use the definition of conditional expectation to rewrite~\eqref{eq:roc_auc_query} as
\begin{equation}\label{eq:roc_auc_query_frac}
    \frac{P(\hat{y}(X) \geq \hat{y}(X'),\ Y = 1,\ Y' = 0)}{P(Y = 1,\ Y' = 0)}.
\end{equation}

Suppose $n \in \mathbb{N}$ and we have an index set $[n] := \{1, \ldots, n\}$ for evaluation set $D_n = \{\langle x_i, y_i, m_i \rangle \}_{i \in [n]}$.
Assuming briefly the vanilla case where $m_i = 0$ for all $i \in [n]$, we can substitute in the usual estimators into~\eqref{eq:roc_auc_query_frac} for the ratio like so
\begin{equation}\label{eq:roc_auc_estimator}
\widehat{Q}_{\text{roc-auc}, n} = \frac{\widehat{P}(\hat{y}(X) \geq \hat{y}(X'),\ Y = 1,\ Y' = 0)}{\widehat{P}(Y = 1,\ Y' = 0)} = \frac{\sum_{i \in [n]} \sum_{j \in [n]} \mathds{1}\{\hat{y}(x_i) \geq \hat{y}(x_j) \} \mathds{1}\{y_i = 1\}  \mathds{1}\{y_j = 0\} }{\sum_{i \in [n]} \sum_{j \in [n]}  \mathds{1}\{y_i = 1\}  \mathds{1}\{y_j = 0\}}
\end{equation}

Now we allow some missing labels $m_i = 1 \Leftrightarrow y_i = \NA$.
We have a partition into known $\mathcal{K} \subset [n]$ and unknown $\overline{\mathcal{K}} \subset [n]$ indices, that is, $\mathcal{K} \cap \overline{\mathcal{K}} = \emptyset$, $\mathcal{K} \cup \overline{\mathcal{K}} = [n]$.
Define the random variables $Y_i$ for $i \in [n]$ such that
\begin{equation}\label{eq:roc_auc_bern}
    Y_i \sim \begin{cases} \mathcal{B}(p_i), & i \in \overline{\mathcal{K}}; \\ \mathds{1}\{Y_i = y_i\}, & i \in \mathcal{K}, \end{cases}
\end{equation}
and $\mathcal{B}(p)$ is a Bernoulli distribution with parameter $p \in (0, 1)$ and the $y_i \in \{0, 1\}$ are constants for $i \in \mathcal{K}$. 

The estimator for ROC-AUC when labels are missing now becomes
\begin{equation}\label{eq:roc_auc_estimator_missing}
    \widehat{Q}_{\text{roc-auc}, n}^{(S)} =   \frac{\sum_{i \in [n]} \sum_{j \in [n]} \mathds{1}\{\hat{y}(x_i) \geq \hat{y}(x_j) \} Y_i (1 - Y_j)}{\sum_{i \in [n]} \sum_{j \in [n]} Y_i (1 - Y_j)}
\end{equation}

We shall require Lemma~\ref{lem:roc_auc_mean_cov} for the mean and covariance of $\widehat{Q}_{\text{roc-auc}, n}^{(S)}$. 
\begin{lemma}\label{lem:roc_auc_mean_cov}
Let $\xi_{i, j}, \zeta_{i,j} \in \mathbb{R}$ for all $i, j \in [n]$. Then

$$\mathbb{E}\left[ \sum_{i,j \in [n]} \xi_{i, j} Y_i (1 - Y_j) \right] = \sum_{i,j \in [n]} \xi_{i, j} \omega_{i, j},$$

where 

$$\omega_{i, j} =: \mathbb{E}[Y_i (1 - Y_j)] = \begin{cases} 
    y_i(1 - y_j), & (i \neq j) \wedge (i, j \in \mathcal{K}); \\
    p_i(1 - y_j), & (i \neq j) \wedge (i \in \overline{\mathcal{K}},\ j \in \mathcal{K}); \\
    y_i(1 - p_j), & (i \neq j) \wedge (i \in \mathcal{K},\ j \in \overline{\mathcal{K}}); \\
    p_i(1 - p_j), & (i \neq j) \wedge (i, j \in \overline{\mathcal{K}}); \\
    0, & \text{otherwise}.
\end{cases}$$

Moreover,
$$\operatorname{Cov}\left[ \sum_{i,j \in [n]} \xi_{i, j} Y_i (1 - Y_j) , \sum_{i,j \in [n]} \zeta_{i, j} Y_i (1 - Y_j)  \right] 
    = \sum_{i, j, k ,l \in [n]} \xi_{i, j} \zeta_{k, l} \cdot \left( \nu_{i, k} - \Gamma_{i, j, k} - \Gamma_{k, l, i} + \eta_{i, j, k, l}  \right),$$

where

$$\nu_{i, j} =: \mathbb{E}[Y_i Y_j] - \mathbb{E}[Y_i]\mathbb{E}[Y_j] = \begin{cases}     
    p_i(1 - p_i), & (i = j) \wedge (i, j \in \overline{\mathcal{K}}); \\
    0, & \text{otherwise}
\end{cases}$$

and

$$\Gamma_{i,j,k} =: \mathbb{E}[Y_i Y_j Y_k] - \mathbb{E}[Y_i Y_j]\mathbb{E}[Y_k] = \begin{cases}
    p_i (1 - p_i), & (i = j = k) \wedge (i, j, k \in \overline{\mathcal{K}}); \\
    p_i p_j (1 - p_i), & (i = k \neq j) \wedge (i, j, k \in \overline{\mathcal{K}}); \\
    p_i p_j (1 - p_j), & (k = j \neq i) \wedge (i, j, k \in \overline{\mathcal{K}}); \\
    y_i p_k (1 - p_k), & ( j = k \neq i) \wedge (i \in \mathcal{K},\ j, k \in \overline{\mathcal{K}}); \\
    y_j p_k (1 - p_k), & ( i = k \neq j) \wedge (j \in \mathcal{K},\ i, k \in \overline{\mathcal{K}}); \\
    0, & \text{otherwise}
\end{cases}
$$

and 

$$\eta_{i, j, k, l} =: \mathbb{E}[Y_i Y_j Y_k Y_l] - \mathbb{E}[Y_i Y_j] \mathbb{E}[Y_k Y_l] = \begin{cases}
    p_i(1 - p_i), & (i = j = k = l) \wedge (i, j, k, l \in \overline{\mathcal{K}}); \\
    p_i p_l (1 - p_i), & (i=j=k\neq l) \wedge (i, j, k, l \in \overline{\mathcal{K}}); \\
    p_i p_k (1 - p_i), & (i=j=l\neq k) \wedge (i, j, k, l \in \overline{\mathcal{K}}); \\
    p_i p_j (1 - p_i), & (i=k=l\neq j) \wedge (i, j, k, l \in \overline{\mathcal{K}}); \\
    p_i p_j (1 - p_j), & (j=k=l\neq i) \wedge (i, j, k, l \in \overline{\mathcal{K}}); \\
    p_i p_j (1 - p_i p_j), & (i=k\neq j=l) \wedge (i, j, k, l \in \overline{\mathcal{K}}); \\
    p_i p_j (1 - p_i p_j), & (i=l\neq j=k) \wedge (i, j, k, l \in \overline{\mathcal{K}}); \\
    p_i p_j p_l (1 - p_i), & (i=k\neq j\neq l \neq i) \wedge (i, j, k, l \in \overline{\mathcal{K}}); \\
    p_i p_j p_k (1 - p_j), & (j=l\neq i \neq k \neq j) \wedge (i, j, k, l \in \overline{\mathcal{K}}); \\
    p_i p_j p_k (1 - p_i), & (i=l\neq j \neq k \neq i) \wedge (i, j, k, l \in \overline{\mathcal{K}}); \\
    p_i p_j p_l (1 - p_j), & (j=k \neq i\neq l \neq j)  \wedge (i, j, k, l \in \overline{\mathcal{K}}); \\
    y_i \Gamma_{k, l, j}, & i \in \mathcal{K}; \\
    y_j \Gamma_{k, l, i}, & j \in \mathcal{K}; \\
    y_k \Gamma_{i, j, l}, & k \in \mathcal{K}; \\
    y_l \Gamma_{i, j, k}, & l \in \mathcal{K}; \\
    0, & \text{otherwise}.
\end{cases}
$$
\end{lemma}
\begin{proof}
Follows readily from the following: linearity of expectation; covariance of linear combinations; the fact that $\mathbb{E}[Y] = p$, $\mathbb{V}[Y] = p(1 - p)$ for $Y \sim \mathcal{B}(p)$; $\operatorname{Cov}[U, V] = 0$ for independent r.v.s $U, V$.   

\end{proof}

\begin{remark}[ROC-AUC Estimator]\label{rem:roc_auc_mean_cov}
We can substitute $\xi_{i, j} = \mathds{1}\{\hat{y}(x_i) \geq \hat{y}(x_j) \}$ and $\zeta_{i, j} = 1$ into Lemma~\ref{lem:roc_auc_mean_cov} to recover the mean and covariance matrix of the numerator and denominator of the estimator $\widehat{Q}_{\text{roc-auc}, n}^{(S)}$.
The mean and variance of the ratio of 
Gaussians computed using Theorem~\ref{thm:ratio_gauss} can then be readily used with PEMI-Gauss algorithm.
\end{remark}

We note that computing the covariance matrix is intractable for large $n$, since there are $O(n^4)$ components in the sum.
In practise, for $n > 120$ we use PEMI with $B = 10000$ and empirically compute $\hat{\mu}$ and $\widehat{\sigma^2}$ from the returned distribution for use within PEMI-Gauss.

%% file: taylor_frac.tex
\begin{lemma}\label{lem:taylor_frac}
    Consider a real number $a \in \mathbb{R}$ and function $s: \mathbb{R} \to \mathbb{R}$ with Taylor series $s(t) = s_0 + s_1 (t - t_0) + s_2 (t - t_0)^2 + O((t - t_0)^3)$ such that $s_0 > 0$.
    We can expand
    \[
    \frac{a}{s(t)} = \frac{a}{s_0} - \frac{a s_1}{s_0^2} (t - t_0) + a \left( \frac{s_1^2}{s_0^3} - \frac{s_2}{s_0^2} \right) (t - t_0)^2 + O((t - t_0)^3).
    \]
\end{lemma}
\begin{proof}
Given the Taylor series expansion of \( s(t) \):
\[
s(t) = s_0 + s_1 (t - t_0) + s_2 (t - t_0)^2 + O((t - t_0)^3),
\]
we want to find the Taylor series expansion of \( \frac{a}{s(t)} \).

First, we rewrite \( s(t) \) in a more convenient form:
\[
s(t) = s_0 \left( 1 + \frac{s_1}{s_0} (t - t_0) + \frac{s_2}{s_0} (t - t_0)^2 + O((t - t_0)^3) \right).
\]

Let \( \epsilon = \frac{s_1}{s_0} (t - t_0) + \frac{s_2}{s_0} (t - t_0)^2 + O((t - t_0)^3)\). Then:
\[
s(t) = s_0 (1 + \epsilon).
\]

The reciprocal of \( s(t) \) is:
\[
\frac{1}{s(t)} = \frac{1}{s_0 (1 + \epsilon)} = \frac{1}{s_0} \cdot \frac{1}{1 + \epsilon}.
\]

Using the Taylor series expansion for \( \frac{1}{1 + \epsilon} \) around \( \epsilon = 0 \):
\[
\frac{1}{1 + \epsilon} = 1 - \epsilon + \epsilon^2 - \epsilon^3 + O(\epsilon^4).
\]

Substituting back \( \epsilon = \frac{s_1}{s_0} (t - t_0) + \frac{s_2}{s_0} (t - t_0)^2 + O((t - t_0)^3)\), we get:
\begin{multline*}
\frac{1}{s(t)} = \frac{1}{s_0} \bigg( 1 - \left( \frac{s_1}{s_0} (t - t_0) + \frac{s_2}{s_0} (t - t_0)^2 + O((t - t_0)^3)  \right) \\ + \left( \frac{s_1}{s_0} (t - t_0) + \frac{s_2}{s_0} (t - t_0)^2 + O((t - t_0)^3) \right)^2 + O((t - t_0)^3) \bigg).
\end{multline*}

Expanding the square term and collecting like powers of \( (t - t_0) \), we get:
\[
\frac{1}{s(t)} = \frac{1}{s_0} \left( 1 - \frac{s_1}{s_0} (t - t_0) - \frac{s_2}{s_0} (t - t_0)^2 + \left( \frac{s_1}{s_0} (t - t_0) \right)^2 + O((t - t_0)^3) \right).
\]

Simplifying further:
\[
\frac{1}{s(t)} = \frac{1}{s_0} - \frac{s_1}{s_0^2} (t - t_0) + \left( \frac{s_1^2}{s_0^3} - \frac{s_2}{s_0^2} \right) (t - t_0)^2 + O((t - t_0)^3).
\]

Finally, multiplying by \( a \), we get the Taylor series expansion for \( \frac{a}{s(t)} \):
\[
\frac{a}{s(t)} = \frac{a}{s_0} - \frac{a s_1}{s_0^2} (t - t_0) + a \left( \frac{s_1^2}{s_0^3} - \frac{s_2}{s_0^2} \right) (t - t_0)^2 + O((t - t_0)^3).
\]

So, the Taylor series expansion of \( \frac{a}{s(t)} \) around \( t = t_0 \) is:
\[
\frac{a}{s(t)} = \frac{a}{s_0} - \frac{a s_1}{s_0^2} (t - t_0) + a \left( \frac{s_1^2}{s_0^3} - \frac{s_2}{s_0^2} \right) (t - t_0)^2 + O((t - t_0)^3).
\]
\end{proof}

%% file: appendix_expt.tex
\section{Supplementary Empirical Results}

All experiments were conducted using an AWS EC2 \texttt{c5a.12xlarge} instance and datasets were downloaded using the OpenML API~\citep{openml}.

\label{sec:appendix_expt}
\subsection{MCAR Regime}
\label{sec:mcar_tables}

\subsubsection{Missingness Fraction $p_m = 0.1$}

\begin{table}[H]
\centering
        {\scriptsize
        \input{tables/ks_df_conf_p_0.1_alpha_0.9.tex}
        }
        \caption{KS-distance of PIT for different algorithms in MCAR setting (all datasets). 90\% confidence intervals. $p_m = 0.1$. Best results in \textbf{bold}.}
        \label{tab:ks_tab_p_0.1_conf}
\end{table}
\afterpage{\FloatBarrier}

\begin{table}[H]
\centering
        {\scriptsize
        \input{tables/W_1_df_conf_p_0.1_alpha_0.9.tex}
        }
        \caption{$W_1$ distance of PIT for different algorithms in MCAR setting (all datasets). 90\% confidence intervals. $p_m = 0.1$. Best results in \textbf{bold}.}
        \label{tab:W_1_tab_p_0.1_conf}
\end{table}
\afterpage{\FloatBarrier}

\begin{table}[H]
\centering
        {\scriptsize
        \input{tables/df_mae_conf_p_0.1_alpha_0.9.tex}
        }
        \caption{MAE of centre of evaluation metric predictive distributions (all datasets), MCAR setting. 90\% confidence intervals. $p_m = 0.1$. Best results in \textbf{bold}.}
        \label{tab:mae_tab_p_0.1_conf}
\end{table}
\afterpage{\FloatBarrier}

\begin{table}[H]
\centering
        {\scriptsize
        \input{tables/df_mse_conf_p_0.1_alpha_0.9.tex}
        }
        \caption{RMSE of centre of evaluation metric predictive distributions (all datasets), MCAR setting. 90\% confidence intervals. $p_m = 0.1$. Best results in \textbf{bold}.}
        \label{tab:mse_tab_p_0.1_conf}
\end{table}
\afterpage{\FloatBarrier}

\subsubsection{Missingness Fraction $p_m = 0.2$}

\begin{table}[H]
\centering
        {\scriptsize
        \input{tables/ks_df_conf_p_0.2_alpha_0.9.tex}
        }
        \caption{KS-distance of PIT for different algorithms in MCAR setting (all datasets). 90\% confidence intervals. $p_m = 0.2$. Best results in \textbf{bold}.}
        \label{tab:ks_tab_p_0.2_conf}
\end{table}
\afterpage{\FloatBarrier}

\begin{table}[H]
\centering
        {\scriptsize
        \input{tables/W_1_df_conf_p_0.2_alpha_0.9.tex}
        }
        \caption{$W_1$ distance of PIT for different algorithms in MCAR setting (all datasets). 90\% confidence intervals. $p_m = 0.2$. Best results in \textbf{bold}.}
        \label{tab:W_1_tab_p_0.2_conf}
\end{table}
\afterpage{\FloatBarrier}

\begin{table}[H]
\centering
        {\scriptsize
        \input{tables/df_mae_conf_p_0.2_alpha_0.9.tex}
        }
        \caption{MAE of centre of evaluation metric predictive distributions (all datasets), MCAR setting. 90\% confidence intervals. $p_m = 0.2$. Best results in \textbf{bold}.}
        \label{tab:mae_tab_p_0.2_conf}
\end{table}
\afterpage{\FloatBarrier}

\begin{table}[H]
\centering
        {\scriptsize
        \input{tables/df_mse_conf_p_0.2_alpha_0.9.tex}
        }
        \caption{RMSE of centre of evaluation metric predictive distributions (all datasets), MCAR setting. 90\% confidence intervals. $p_m = 0.2$. Best results in \textbf{bold}.}
        \label{tab:mse_tab_p_0.2_conf}
\end{table}
\afterpage{\FloatBarrier}

\subsubsection{Missingness Fraction $p_m = 0.3$}

\begin{table}[H]
\centering
        {\scriptsize
        \input{tables/ks_df_conf_p_0.3_alpha_0.9.tex}
        }
        \caption{KS-distance of PIT for different algorithms in MCAR setting (all datasets). 90\% confidence intervals. $p_m = 0.3$. Best results in \textbf{bold}.}
        \label{tab:ks_tab_p_0.3_conf}
\end{table}
\afterpage{\FloatBarrier}

\begin{table}[H]
\centering
        {\scriptsize
        \input{tables/W_1_df_conf_p_0.3_alpha_0.9.tex}
        }
        \caption{$W_1$ distance of PIT for different algorithms in MCAR setting (all datasets). 90\% confidence intervals. $p_m = 0.3$. Best results in \textbf{bold}.}
        \label{tab:W_1_tab_p_0.3_conf}
\end{table}
\afterpage{\FloatBarrier}

\begin{table}[H]
\centering
        {\scriptsize
        \input{tables/df_mae_conf_p_0.3_alpha_0.9.tex}
        }
        \caption{MAE of centre of evaluation metric predictive distributions (all datasets), MCAR setting. 90\% confidence intervals. $p_m = 0.3$. Best results in \textbf{bold}.}
        \label{tab:mae_tab_p_0.3_conf}
\end{table}
\afterpage{\FloatBarrier}

\begin{table}[H]
\centering
        {\scriptsize
        \input{tables/df_mse_conf_p_0.3_alpha_0.9.tex}
        }
        \caption{RMSE of centre of evaluation metric predictive distributions (all datasets), MCAR setting. 90\% confidence intervals. $p_m = 0.3$. Best results in \textbf{bold}.}
        \label{tab:mse_tab_p_0.3_conf}
\end{table}
\afterpage{\FloatBarrier}

\subsection{MNAR Regime}
\subsubsection{Missingness Fraction $p_m = 0.1$}

\begin{figure}[H]
    \centering
    \includegraphics[width=0.99\textwidth]{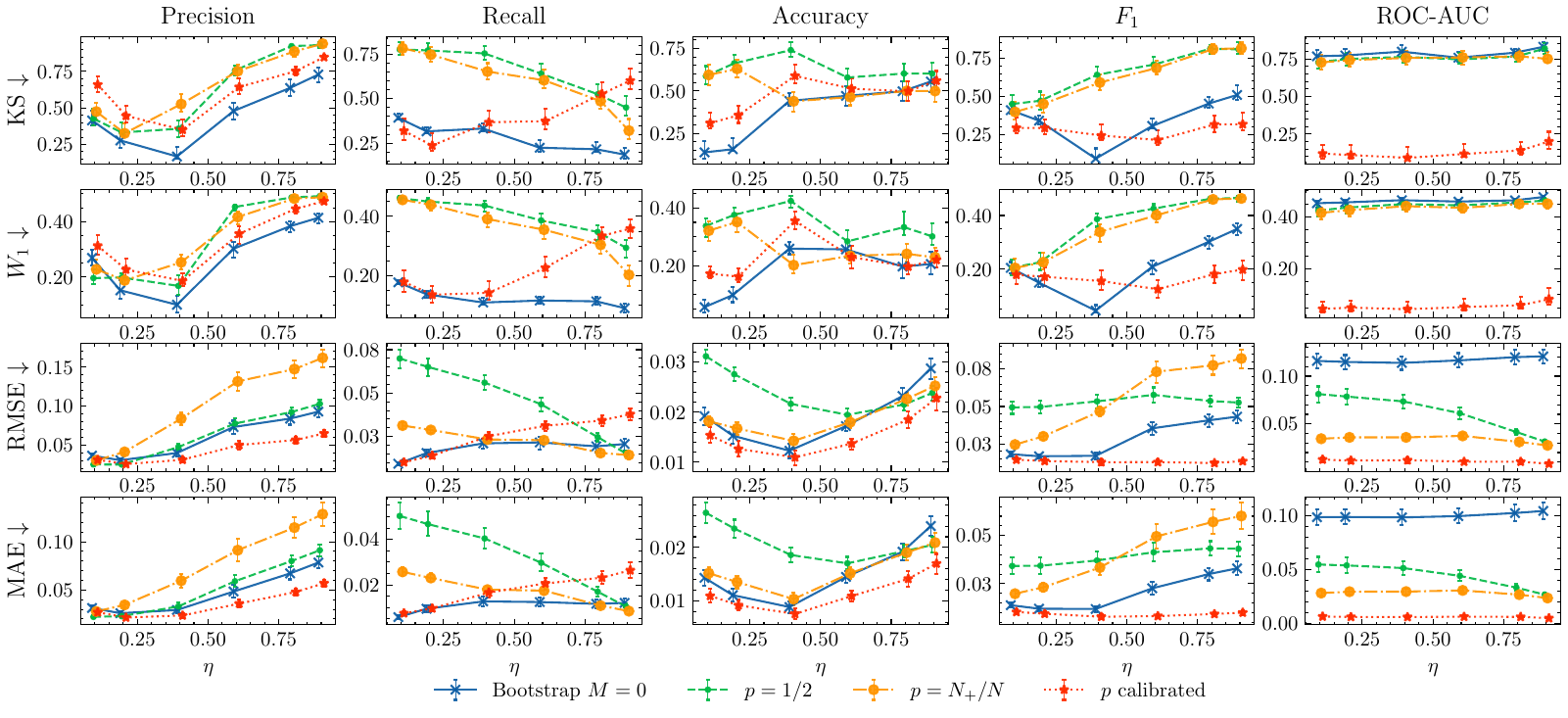}
    \caption{Effect of varying MNAR class imbalance $\eta \in (0, 1)$ on evaluation metric predictive distribution quality using PEMI-Gauss for $p_{m} = 0.1$. Each plot represents a metric $\widehat{Q}_n^{(S)}$. Error bars are bootstrapped confidence intervals at the $\alpha=0.9$ level. Quality is on $y$-axis with $\eta$ varying on the $x$-axis, with $x$-axis values slightly jittered to help separate each series visually. Each line corresponds to a predictive distribution, with $p$ algorithms using PEMI-Gauss.\vspace{12cm} }
    \label{fig:mnar_seq_plots_all_p_01}
\end{figure}
\afterpage{\FloatBarrier}

\subsubsection{Missingness Fraction $p_m = 0.2$}

\begin{figure}[H]
    \centering
    \includegraphics[width=0.99\textwidth]{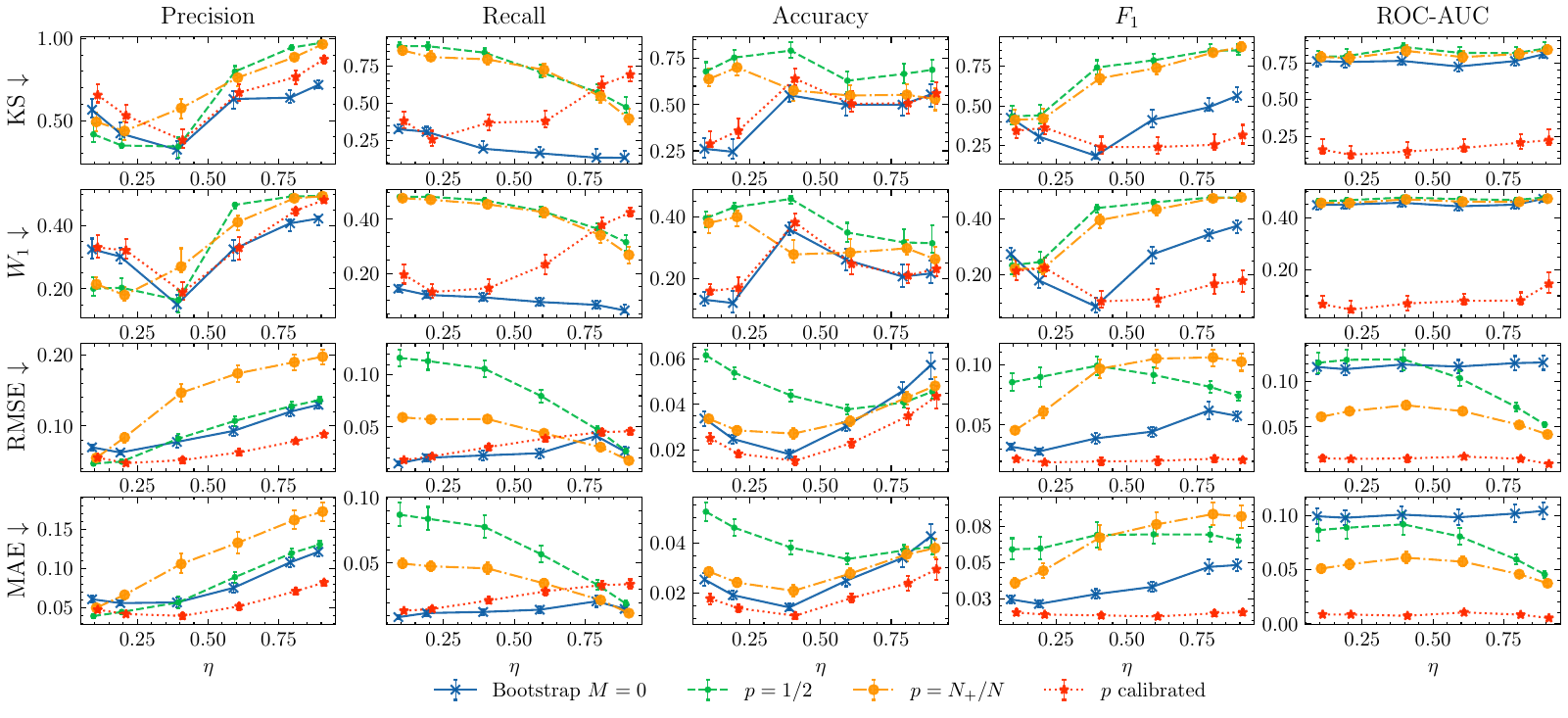}
    \caption{Effect of varying MNAR class imbalance $\eta \in (0, 1)$ on evaluation metric predictive distribution quality using PEMI-Gauss for $p_{m} = 0.2$. Each plot represents a metric $\widehat{Q}_n^{(S)}$. Error bars are bootstrapped confidence intervals at the $\alpha=0.9$ level. Quality is on $y$-axis with $\eta$ varying on the $x$-axis, with $x$-axis values slightly jittered to help separate each series visually. Each line corresponds to a predictive distribution, with $p$ algorithms using PEMI-Gauss. }
    \label{fig:mnar_seq_plots_all_p_02}
\end{figure}
\afterpage{\FloatBarrier}

\subsubsection{Missingness Fraction $p_m = 0.3$}

\begin{figure}[H]
    \centering
    \includegraphics[width=0.99\textwidth]{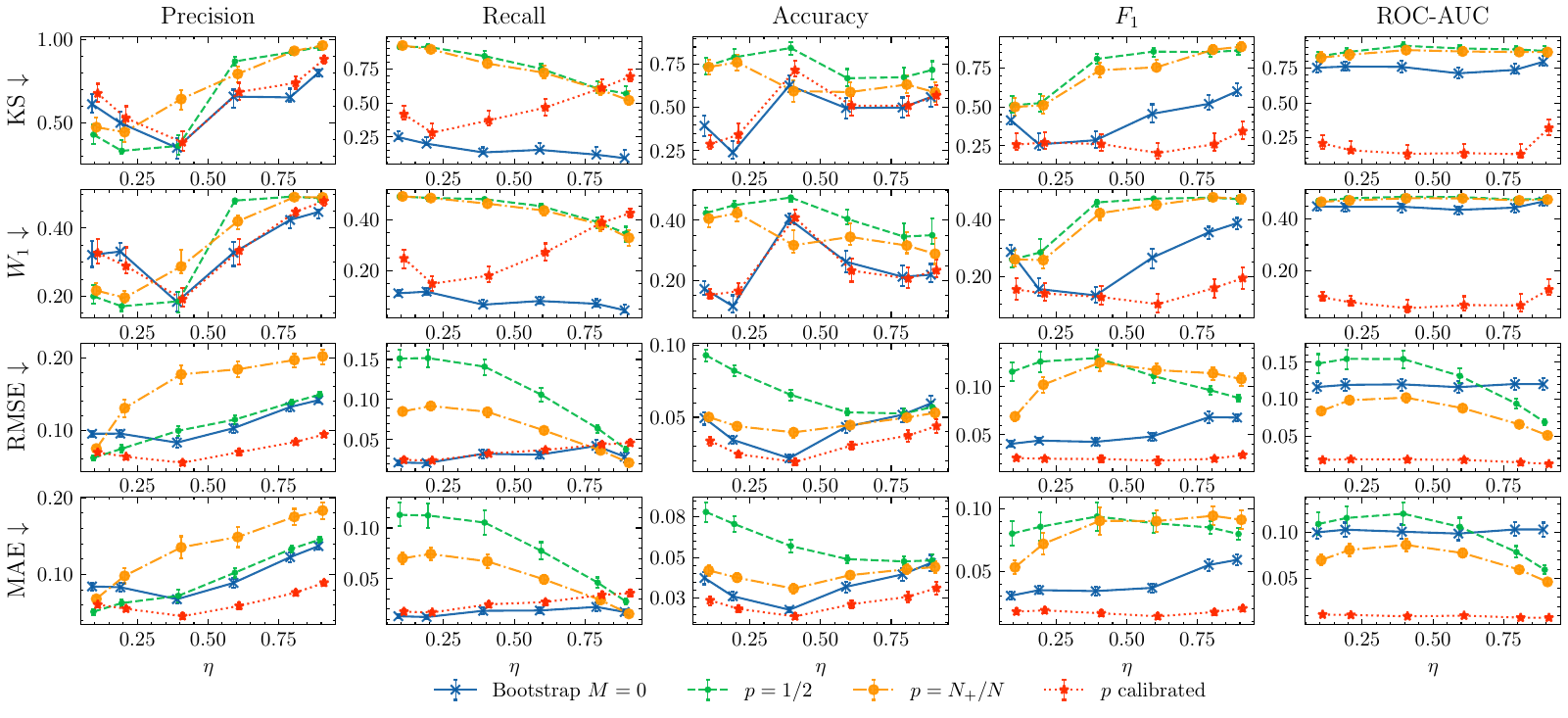}
    \caption{Effect of varying MNAR class imbalance $\eta \in (0, 1)$ on evaluation metric predictive distribution quality using PEMI-Gauss for $p_{m} = 0.3$. Each plot represents a metric $\widehat{Q}_n^{(S)}$. Error bars are bootstrapped confidence intervals at the $\alpha=0.9$ level. Quality is on $y$-axis with $\eta$ varying on the $x$-axis, with $x$-axis values slightly jittered to help separate each series visually. Each line corresponds to a predictive distribution, with $p$ algorithms using PEMI-Gauss. }
    \label{fig:mnar_seq_plots_all_p_03}
\end{figure}
\afterpage{\FloatBarrier}

%% file: tables/ks_df_conf_p_0.1_alpha_0.9.tex
\begin{tabular}{lllllll}
\toprule
 &  & Precision & Recall & Accuracy & $F_1$ & ROC-AUC \\
\midrule
Bootstrap $M = 0$ &  & (0.2688, 0.311) & (0.2879, 0.3371) & (0.2691, 0.312) & (0.2843, 0.3351) & (0.7648, 0.8359) \\
\addlinespace
\multirow[c]{3}{*}{PEMI $B = 5$} & $p = \frac{1}{2}$ & (0.3908, 0.5) & (0.5333, 0.6333) & (0.6667, 0.7333) & (0.575, 0.6667) & (0.6083, 0.6917) \\
 & $p$ calibrated & \bfseries (0.1083, 0.175) & \bfseries (0.1083, 0.1833) & \bfseries (0.0825, 0.15) & \bfseries (0.1, 0.175) & \bfseries (0.1, 0.175) \\
 & $p = N_+ / N$ & (0.5917, 0.675) & (0.5417, 0.6417) & (0.625, 0.7083) & (0.575, 0.6583) & (0.6, 0.6833) \\
\addlinespace
\multirow[c]{3}{*}{PEMI $B = 10$} & $p = \frac{1}{2}$ & (0.425, 0.5333) & (0.625, 0.725) & (0.7167, 0.8) & (0.6667, 0.7583) & (0.7, 0.7833) \\
 & $p$ calibrated & \bfseries (0.0833, 0.1583) & \bfseries (0.075, 0.1417) & \bfseries (0.1083, 0.2167) & \bfseries (0.0667, 0.125) & \bfseries (0.0833, 0.1833) \\
 & $p = N_+ / N$ & (0.625, 0.7083) & (0.6417, 0.7333) & (0.65, 0.75) & (0.5917, 0.6833) & (0.6583, 0.7508) \\
\addlinespace
\multirow[c]{3}{*}{PEMI $B = 100$} & $p = \frac{1}{2}$ & (0.4933, 0.6067) & (0.6632, 0.7567) & (0.765, 0.8467) & (0.705, 0.7917) & (0.715, 0.795) \\
 & $p$ calibrated & \bfseries (0.115, 0.21) & \bfseries (0.065, 0.1617) & (0.1583, 0.2683) & \bfseries (0.0667, 0.1633) & \bfseries (0.1217, 0.2317) \\
 & $p = N_+ / N$ & (0.66, 0.75) & (0.6717, 0.7633) & (0.7, 0.7917) & (0.6617, 0.7483) & (0.7017, 0.7967) \\
\addlinespace
\multirow[c]{3}{*}{PEMI-Gauss} & $p = \frac{1}{2}$ & (0.3987, 0.5108) & (0.6587, 0.7536) & (0.7289, 0.8289) & (0.6893, 0.781) & (0.7299, 0.8065) \\
 & $p$ calibrated & \bfseries (0.0666, 0.1171) & \bfseries (0.0606, 0.1423) & \bfseries (0.0922, 0.2058) & \bfseries (0.0714, 0.1521) & \bfseries (0.1499, 0.2583) \\
 & $p = N_+ / N$ & (0.5865, 0.6821) & (0.6403, 0.7288) & (0.6522, 0.7477) & (0.6525, 0.7424) & (0.6937, 0.7853) \\
\bottomrule
\end{tabular}

%% file: tables/W_1_df_conf_p_0.1_alpha_0.9.tex
\begin{tabular}{lllllll}
\toprule
 &  & Precision & Recall & Accuracy & $F_1$ & ROC-AUC \\
\midrule
Bootstrap $M = 0$ &  & (0.1512, 0.169) & (0.1519, 0.1677) & (0.0977, 0.1224) & (0.153, 0.1682) & (0.4526, 0.4712) \\
\addlinespace
\multirow[c]{3}{*}{PEMI $B = 5$} & $p = \frac{1}{2}$ & (0.1842, 0.2467) & (0.2892, 0.335) & (0.3433, 0.375) & (0.3117, 0.3492) & (0.3317, 0.3642) \\
 & $p$ calibrated & \bfseries (0.0408, 0.0792) & \bfseries (0.0425, 0.0784) & \bfseries (0.0317, 0.0808) & \bfseries (0.0374, 0.1067) & \bfseries (0.0475, 0.0867) \\
 & $p = N_+ / N$ & (0.3058, 0.3475) & (0.2917, 0.3383) & (0.3283, 0.3633) & (0.3142, 0.3508) & (0.3025, 0.3458) \\
\addlinespace
\multirow[c]{3}{*}{PEMI $B = 10$} & $p = \frac{1}{2}$ & (0.2087, 0.273) & (0.3388, 0.3854) & (0.4, 0.4246) & (0.3737, 0.4084) & (0.3825, 0.4133) \\
 & $p$ calibrated & \bfseries (0.0346, 0.0863) & \bfseries (0.03, 0.0679) & \bfseries (0.0554, 0.1171) & \bfseries (0.0229, 0.0521) & \bfseries (0.0396, 0.0858) \\
 & $p = N_+ / N$ & (0.3554, 0.395) & (0.3512, 0.3934) & (0.3733, 0.4067) & (0.3592, 0.3925) & (0.3592, 0.4004) \\
\addlinespace
\multirow[c]{3}{*}{PEMI $B = 100$} & $p = \frac{1}{2}$ & (0.2465, 0.3173) & (0.386, 0.4301) & (0.4475, 0.4694) & (0.4154, 0.449) & (0.4224, 0.4533) \\
 & $p$ calibrated & (0.0615, 0.1238) & \bfseries (0.0234, 0.0734) & \bfseries (0.0776, 0.1471) & \bfseries (0.0255, 0.0739) & \bfseries (0.0523, 0.1156) \\
 & $p = N_+ / N$ & (0.3915, 0.4325) & (0.3875, 0.4297) & (0.4118, 0.4456) & (0.3992, 0.4354) & (0.4076, 0.4445) \\
\addlinespace
\multirow[c]{3}{*}{PEMI-Gauss} & $p = \frac{1}{2}$ & (0.1979, 0.2696) & (0.3795, 0.4268) & (0.4168, 0.4515) & (0.4088, 0.4455) & (0.4339, 0.4619) \\
 & $p$ calibrated & \bfseries (0.0245, 0.0529) & \bfseries (0.0218, 0.0569) & \bfseries (0.035, 0.0977) & \bfseries (0.0269, 0.064) & \bfseries (0.0591, 0.1203) \\
 & $p = N_+ / N$ & (0.3502, 0.4014) & (0.37, 0.4199) & (0.3881, 0.4304) & (0.3841, 0.4278) & (0.415, 0.4495) \\
\bottomrule
\end{tabular}

%% file: tables/df_mae_conf_p_0.1_alpha_0.9.tex
\begin{tabular}{lllllll}
\toprule
 &  & Precision & Recall & Accuracy & $F_1$ & ROC-AUC \\
\midrule
Bootstrap $M = 0$ &  & \bfseries (0.0074, 0.0107) & (0.0067, 0.0098) & (0.0055, 0.0071) & \bfseries (0.0057, 0.0085) & (0.0931, 0.1073) \\
\addlinespace
\multirow[c]{3}{*}{PEMI $B = 5$} & $p = \frac{1}{2}$ & (0.0165, 0.0206) & (0.0409, 0.0529) & (0.0248, 0.0284) & (0.0334, 0.0417) & (0.0495, 0.0638) \\
 & $p$ calibrated & \bfseries (0.0071, 0.0103) & \bfseries (0.0054, 0.0077) & \bfseries (0.0038, 0.0056) & \bfseries (0.0046, 0.0067) & \bfseries (0.0052, 0.0072) \\
 & $p = N_+ / N$ & (0.0195, 0.0245) & (0.0194, 0.0236) & (0.0122, 0.0145) & (0.0194, 0.0237) & (0.0269, 0.0317) \\
\addlinespace
\multirow[c]{3}{*}{PEMI $B = 10$} & $p = \frac{1}{2}$ & (0.0163, 0.0205) & (0.0396, 0.0516) & (0.024, 0.0277) & (0.033, 0.0415) & (0.0495, 0.0638) \\
 & $p$ calibrated & \bfseries (0.0069, 0.0096) & \bfseries (0.0045, 0.0063) & \bfseries (0.0035, 0.0052) & \bfseries (0.004, 0.0061) & \bfseries (0.0045, 0.0063) \\
 & $p = N_+ / N$ & (0.0196, 0.0245) & (0.0202, 0.0244) & (0.0126, 0.0148) & (0.0193, 0.0233) & (0.026, 0.0308) \\
\addlinespace
\multirow[c]{3}{*}{PEMI $B = 100$} & $p = \frac{1}{2}$ & (0.0163, 0.0201) & (0.0397, 0.0516) & (0.0241, 0.0276) & (0.0332, 0.0416) & (0.0495, 0.0636) \\
 & $p$ calibrated & \bfseries (0.0063, 0.009) & \bfseries (0.0041, 0.006) & \bfseries (0.0035, 0.0049) & \bfseries (0.0043, 0.0061) & \bfseries (0.0044, 0.0059) \\
 & $p = N_+ / N$ & (0.0193, 0.0242) & (0.0202, 0.0244) & (0.0121, 0.0144) & (0.0194, 0.0235) & (0.0261, 0.0307) \\
\addlinespace
\multirow[c]{3}{*}{PEMI-Gauss} & $p = \frac{1}{2}$ & (0.0161, 0.02) & (0.04, 0.052) & (0.0241, 0.0276) & (0.0333, 0.0416) & (0.0496, 0.0638) \\
 & $p$ calibrated & \bfseries (0.0063, 0.009) & \bfseries (0.0042, 0.0061) & \bfseries (0.0035, 0.0049) & \bfseries (0.0043, 0.0061) & \bfseries (0.0043, 0.0059) \\
 & $p = N_+ / N$ & (0.0192, 0.024) & (0.02, 0.0242) & (0.0121, 0.0143) & (0.0194, 0.0235) & (0.0262, 0.0309) \\
\bottomrule
\end{tabular}

%% file: tables/df_mse_conf_p_0.1_alpha_0.9.tex
\begin{tabular}{lllllll}
\toprule
 &  & Precision & Recall & Accuracy & $F_1$ & ROC-AUC \\
\midrule
Bootstrap $M = 0$ &  & \bfseries (0.0134, 0.0198) & (0.0125, 0.0183) & \bfseries (0.0082, 0.0104) & \bfseries (0.0102, 0.0167) & (0.1086, 0.1236) \\
\addlinespace
\multirow[c]{3}{*}{PEMI $B = 5$} & $p = \frac{1}{2}$ & (0.0227, 0.0283) & (0.0631, 0.0749) & (0.0289, 0.0323) & (0.0474, 0.0557) & (0.0741, 0.0909) \\
 & $p$ calibrated & \bfseries (0.0122, 0.019) & \bfseries (0.0095, 0.0134) & \bfseries (0.0075, 0.0105) & \bfseries (0.0084, 0.0122) & \bfseries (0.009, 0.0117) \\
 & $p = N_+ / N$ & (0.0276, 0.0345) & (0.026, 0.0298) & (0.0154, 0.0182) & (0.0264, 0.0312) & (0.0334, 0.038) \\
\addlinespace
\multirow[c]{3}{*}{PEMI $B = 10$} & $p = \frac{1}{2}$ & (0.0224, 0.0283) & (0.0623, 0.0739) & (0.0283, 0.0316) & (0.047, 0.0551) & (0.0736, 0.0906) \\
 & $p$ calibrated & \bfseries (0.0116, 0.0166) & \bfseries (0.008, 0.0113) & \bfseries (0.0069, 0.0097) & \bfseries (0.0077, 0.0128) & \bfseries (0.0082, 0.0107) \\
 & $p = N_+ / N$ & (0.0276, 0.034) & (0.0268, 0.0307) & (0.0156, 0.0182) & (0.0258, 0.03) & (0.0325, 0.0368) \\
\addlinespace
\multirow[c]{3}{*}{PEMI $B = 100$} & $p = \frac{1}{2}$ & (0.0217, 0.0267) & (0.0622, 0.074) & (0.0281, 0.0314) & (0.0471, 0.0552) & (0.0735, 0.0903) \\
 & $p$ calibrated & \bfseries (0.0109, 0.0157) & \bfseries (0.0076, 0.0113) & \bfseries (0.0065, 0.0087) & \bfseries (0.0074, 0.011) & \bfseries (0.0076, 0.0098) \\
 & $p = N_+ / N$ & (0.0274, 0.0342) & (0.0269, 0.0307) & (0.0151, 0.0175) & (0.0259, 0.0304) & (0.0322, 0.0364) \\
\addlinespace
\multirow[c]{3}{*}{PEMI-Gauss} & $p = \frac{1}{2}$ & (0.0216, 0.0266) & (0.0624, 0.0742) & (0.0282, 0.0314) & (0.047, 0.0552) & (0.0736, 0.0903) \\
 & $p$ calibrated & \bfseries (0.0109, 0.0156) & \bfseries (0.0076, 0.0113) & \bfseries (0.0065, 0.0086) & \bfseries (0.0073, 0.0111) & \bfseries (0.0075, 0.0097) \\
 & $p = N_+ / N$ & (0.0271, 0.0336) & (0.0267, 0.0305) & (0.0151, 0.0176) & (0.0259, 0.0304) & (0.0323, 0.0365) \\
\bottomrule
\end{tabular}

%% file: tables/ks_df_conf_p_0.2_alpha_0.9.tex
\begin{tabular}{lllllll}
\toprule
 &  & Precision & Recall & Accuracy & $F_1$ & ROC-AUC \\
\midrule
Bootstrap $M = 0$ &  & (0.2705, 0.3091) & (0.2337, 0.285) & (0.2174, 0.2737) & (0.2336, 0.2883) & (0.7302, 0.795) \\
\addlinespace
\multirow[c]{3}{*}{PEMI $B = 5$} & $p = \frac{1}{2}$ & (0.3667, 0.4833) & (0.6, 0.6917) & (0.7333, 0.7833) & (0.675, 0.7417) & (0.6833, 0.75) \\
 & $p$ calibrated & \bfseries (0.1083, 0.1833) & \bfseries (0.125, 0.2092) & \bfseries (0.0833, 0.1583) & \bfseries (0.1333, 0.2167) & \bfseries (0.1158, 0.175) \\
 & $p = N_+ / N$ & (0.5667, 0.65) & (0.5833, 0.675) & (0.6083, 0.6917) & (0.6, 0.6833) & (0.6833, 0.75) \\
\addlinespace
\multirow[c]{3}{*}{PEMI $B = 10$} & $p = \frac{1}{2}$ & (0.4167, 0.525) & (0.65, 0.7425) & (0.7917, 0.8583) & (0.75, 0.8167) & (0.7667, 0.8258) \\
 & $p$ calibrated & \bfseries (0.05, 0.1167) & \bfseries (0.075, 0.1333) & \bfseries (0.0833, 0.175) & \bfseries (0.0917, 0.175) & \bfseries (0.1, 0.1917) \\
 & $p = N_+ / N$ & (0.6667, 0.7333) & (0.65, 0.7333) & (0.65, 0.7333) & (0.6833, 0.7508) & (0.6917, 0.7833) \\
\addlinespace
\multirow[c]{3}{*}{PEMI $B = 100$} & $p = \frac{1}{2}$ & (0.445, 0.5567) & (0.6998, 0.7883) & (0.84, 0.8833) & (0.7783, 0.85) & (0.835, 0.8917) \\
 & $p$ calibrated & \bfseries (0.055, 0.115) & \bfseries (0.0617, 0.1483) & \bfseries (0.1167, 0.215) & \bfseries (0.075, 0.1335) & \bfseries (0.125, 0.2217) \\
 & $p = N_+ / N$ & (0.7133, 0.7867) & (0.6967, 0.775) & (0.7317, 0.7933) & (0.705, 0.7767) & (0.77, 0.8317) \\
\addlinespace
\multirow[c]{3}{*}{PEMI-Gauss} & $p = \frac{1}{2}$ & (0.4137, 0.5303) & (0.7024, 0.7901) & (0.8298, 0.8798) & (0.7833, 0.8509) & (0.8471, 0.893) \\
 & $p$ calibrated & \bfseries (0.0877, 0.181) & \bfseries (0.0592, 0.1378) & \bfseries (0.0907, 0.1893) & \bfseries (0.0771, 0.1408) & \bfseries (0.1458, 0.2458) \\
 & $p = N_+ / N$ & (0.6369, 0.7239) & (0.6886, 0.7589) & (0.6742, 0.7412) & (0.6922, 0.7675) & (0.7795, 0.8435) \\
\bottomrule
\end{tabular}

%% file: tables/W_1_df_conf_p_0.2_alpha_0.9.tex
\begin{tabular}{lllllll}
\toprule
 &  & Precision & Recall & Accuracy & $F_1$ & ROC-AUC \\
\midrule
Bootstrap $M = 0$ &  & (0.1385, 0.1573) & (0.1201, 0.1406) & (0.0895, 0.1265) & (0.1293, 0.1502) & (0.4432, 0.4644) \\
\addlinespace
\multirow[c]{3}{*}{PEMI $B = 5$} & $p = \frac{1}{2}$ & (0.1657, 0.2308) & (0.325, 0.3618) & (0.375, 0.3917) & (0.355, 0.3792) & (0.3625, 0.3833) \\
 & $p$ calibrated & (0.0508, 0.1192) & \bfseries (0.0542, 0.1208) & \bfseries (0.0325, 0.0684) & (0.0658, 0.1292) & \bfseries (0.0417, 0.09) \\
 & $p = N_+ / N$ & (0.3283, 0.3583) & (0.3258, 0.3592) & (0.3367, 0.3658) & (0.33, 0.3625) & (0.3583, 0.3817) \\
\addlinespace
\multirow[c]{3}{*}{PEMI $B = 10$} & $p = \frac{1}{2}$ & (0.1867, 0.2554) & (0.3808, 0.41) & (0.4125, 0.4375) & (0.4117, 0.4321) & (0.4125, 0.4371) \\
 & $p$ calibrated & \bfseries (0.0212, 0.058) & \bfseries (0.0254, 0.0571) & \bfseries (0.0354, 0.0888) & \bfseries (0.0341, 0.0838) & \bfseries (0.0408, 0.0971) \\
 & $p = N_+ / N$ & (0.3817, 0.41) & (0.3763, 0.4104) & (0.3842, 0.4129) & (0.3983, 0.42) & (0.3958, 0.4208) \\
\addlinespace
\multirow[c]{3}{*}{PEMI $B = 100$} & $p = \frac{1}{2}$ & (0.232, 0.3021) & (0.4202, 0.4513) & (0.4703, 0.4815) & (0.4527, 0.471) & (0.4638, 0.479) \\
 & $p$ calibrated & \bfseries (0.0197, 0.0469) & \bfseries (0.0234, 0.0713) & \bfseries (0.0615, 0.1291) & \bfseries (0.0263, 0.0553) & \bfseries (0.0666, 0.1315) \\
 & $p = N_+ / N$ & (0.433, 0.4556) & (0.419, 0.4491) & (0.4393, 0.4591) & (0.4328, 0.4558) & (0.4442, 0.4679) \\
\addlinespace
\multirow[c]{3}{*}{PEMI-Gauss} & $p = \frac{1}{2}$ & (0.1972, 0.2702) & (0.4252, 0.4555) & (0.4627, 0.4807) & (0.4529, 0.4723) & (0.4696, 0.4843) \\
 & $p$ calibrated & \bfseries (0.0344, 0.0925) & \bfseries (0.0202, 0.0634) & \bfseries (0.0383, 0.1006) & \bfseries (0.029, 0.0577) & \bfseries (0.0816, 0.1453) \\
 & $p = N_+ / N$ & (0.4089, 0.4389) & (0.4151, 0.4485) & (0.4233, 0.4503) & (0.4313, 0.4564) & (0.4529, 0.4732) \\
\bottomrule
\end{tabular}

%% file: tables/df_mae_conf_p_0.2_alpha_0.9.tex
\begin{tabular}{lllllll}
\toprule
 &  & Precision & Recall & Accuracy & $F_1$ & ROC-AUC \\
\midrule
Bootstrap $M = 0$ &  & \bfseries (0.0086, 0.0119) & (0.0095, 0.0132) & \bfseries (0.0068, 0.009) & \bfseries (0.0071, 0.0098) & (0.092, 0.106) \\
\addlinespace
\multirow[c]{3}{*}{PEMI $B = 5$} & $p = \frac{1}{2}$ & (0.0287, 0.0347) & (0.0678, 0.0867) & (0.047, 0.0535) & (0.0584, 0.0726) & (0.0813, 0.1003) \\
 & $p$ calibrated & \bfseries (0.0088, 0.0121) & \bfseries (0.0073, 0.0101) & \bfseries (0.0057, 0.0083) & \bfseries (0.0055, 0.0076) & \bfseries (0.0068, 0.0099) \\
 & $p = N_+ / N$ & (0.0359, 0.0445) & (0.0371, 0.0449) & (0.0224, 0.0267) & (0.0375, 0.0459) & (0.0509, 0.0586) \\
\addlinespace
\multirow[c]{3}{*}{PEMI $B = 10$} & $p = \frac{1}{2}$ & (0.0275, 0.0337) & (0.0681, 0.0865) & (0.0478, 0.0543) & (0.0595, 0.0736) & (0.0827, 0.1011) \\
 & $p$ calibrated & \bfseries (0.0095, 0.0129) & \bfseries (0.0067, 0.0093) & \bfseries (0.0056, 0.0079) & \bfseries (0.0066, 0.0094) & \bfseries (0.0066, 0.0094) \\
 & $p = N_+ / N$ & (0.0374, 0.0459) & (0.0376, 0.0455) & (0.0227, 0.0267) & (0.0378, 0.0458) & (0.0498, 0.0572) \\
\addlinespace
\multirow[c]{3}{*}{PEMI $B = 100$} & $p = \frac{1}{2}$ & (0.0273, 0.0334) & (0.0679, 0.0865) & (0.0473, 0.0535) & (0.0594, 0.0734) & (0.0823, 0.1009) \\
 & $p$ calibrated & \bfseries (0.0091, 0.0125) & \bfseries (0.0064, 0.0089) & \bfseries (0.0051, 0.0073) & \bfseries (0.006, 0.0083) & \bfseries (0.006, 0.0086) \\
 & $p = N_+ / N$ & (0.0364, 0.0451) & (0.0373, 0.045) & (0.0227, 0.0265) & (0.0372, 0.045) & (0.0497, 0.0571) \\
\addlinespace
\multirow[c]{3}{*}{PEMI-Gauss} & $p = \frac{1}{2}$ & (0.027, 0.033) & (0.0681, 0.0867) & (0.0475, 0.0537) & (0.059, 0.0732) & (0.082, 0.1005) \\
 & $p$ calibrated & \bfseries (0.009, 0.0122) & \bfseries (0.0063, 0.0087) & \bfseries (0.0052, 0.0075) & \bfseries (0.006, 0.0084) & \bfseries (0.0061, 0.0086) \\
 & $p = N_+ / N$ & (0.0363, 0.0449) & (0.0373, 0.045) & (0.0226, 0.0264) & (0.0373, 0.0451) & (0.0497, 0.0572) \\
\bottomrule
\end{tabular}

%% file: tables/df_mse_conf_p_0.2_alpha_0.9.tex
\begin{tabular}{lllllll}
\toprule
 &  & Precision & Recall & Accuracy & $F_1$ & ROC-AUC \\
\midrule
Bootstrap $M = 0$ &  & \bfseries (0.0143, 0.02) & (0.0162, 0.023) & \bfseries (0.0104, 0.0137) & \bfseries (0.0121, 0.017) & (0.1083, 0.1233) \\
\addlinespace
\multirow[c]{3}{*}{PEMI $B = 5$} & $p = \frac{1}{2}$ & (0.0374, 0.0436) & (0.1027, 0.1202) & (0.0545, 0.0605) & (0.0817, 0.0953) & (0.1113, 0.1326) \\
 & $p$ calibrated & \bfseries (0.0145, 0.0209) & \bfseries (0.0127, 0.0169) & \bfseries (0.0104, 0.0148) & \bfseries (0.0094, 0.0133) & \bfseries (0.0128, 0.0176) \\
 & $p = N_+ / N$ & (0.0492, 0.059) & (0.0494, 0.0567) & (0.028, 0.0334) & (0.0506, 0.0588) & (0.0605, 0.0671) \\
\addlinespace
\multirow[c]{3}{*}{PEMI $B = 10$} & $p = \frac{1}{2}$ & (0.0366, 0.043) & (0.1027, 0.1199) & (0.0549, 0.061) & (0.0825, 0.0959) & (0.1119, 0.1331) \\
 & $p$ calibrated & \bfseries (0.0152, 0.022) & \bfseries (0.0116, 0.0168) & \bfseries (0.0099, 0.0136) & \bfseries (0.0118, 0.0161) & \bfseries (0.0122, 0.017) \\
 & $p = N_+ / N$ & (0.0505, 0.0603) & (0.0498, 0.0571) & (0.0281, 0.0325) & (0.0501, 0.0582) & (0.0594, 0.0659) \\
\addlinespace
\multirow[c]{3}{*}{PEMI $B = 100$} & $p = \frac{1}{2}$ & (0.0363, 0.0423) & (0.1025, 0.12) & (0.0541, 0.0601) & (0.0821, 0.0956) & (0.1116, 0.1329) \\
 & $p$ calibrated & \bfseries (0.0149, 0.0208) & \bfseries (0.0112, 0.0149) & \bfseries (0.0093, 0.0127) & \bfseries (0.0102, 0.0142) & \bfseries (0.0111, 0.0153) \\
 & $p = N_+ / N$ & (0.0496, 0.0595) & (0.0495, 0.0566) & (0.0275, 0.0318) & (0.0492, 0.0571) & (0.0592, 0.0658) \\
\addlinespace
\multirow[c]{3}{*}{PEMI-Gauss} & $p = \frac{1}{2}$ & (0.0362, 0.0421) & (0.1026, 0.12) & (0.0542, 0.0602) & (0.0819, 0.0955) & (0.1112, 0.1327) \\
 & $p$ calibrated & \bfseries (0.0145, 0.0202) & \bfseries (0.0109, 0.0147) & \bfseries (0.0094, 0.0129) & \bfseries (0.0104, 0.0144) & \bfseries (0.0111, 0.0153) \\
 & $p = N_+ / N$ & (0.0496, 0.0593) & (0.0495, 0.0565) & (0.0275, 0.0318) & (0.0492, 0.0572) & (0.0592, 0.0657) \\
\bottomrule
\end{tabular}

%% file: tables/ks_df_conf_p_0.3_alpha_0.9.tex
\begin{tabular}{lllllll}
\toprule
 &  & Precision & Recall & Accuracy & $F_1$ & ROC-AUC \\
\midrule
Bootstrap $M = 0$ &  & (0.1759, 0.2381) & \bfseries (0.1504, 0.2124) & (0.14, 0.1984) & (0.1478, 0.2093) & (0.7004, 0.7812) \\
\addlinespace
\multirow[c]{3}{*}{PEMI $B = 5$} & $p = \frac{1}{2}$ & (0.35, 0.4583) & (0.6667, 0.7333) & (0.5917, 0.7833) & (0.7167, 0.7667) & (0.725, 0.775) \\
 & $p$ calibrated & \bfseries (0.1075, 0.1833) & \bfseries (0.15, 0.2417) & \bfseries (0.1333, 0.2167) & \bfseries (0.1167, 0.2) & \bfseries (0.1167, 0.2083) \\
 & $p = N_+ / N$ & (0.625, 0.7) & (0.6583, 0.725) & (0.5833, 0.6667) & (0.6667, 0.7333) & (0.6917, 0.7583) \\
\addlinespace
\multirow[c]{3}{*}{PEMI $B = 10$} & $p = \frac{1}{2}$ & (0.475, 0.5833) & (0.7167, 0.7917) & (0.7917, 0.8833) & (0.7167, 0.7917) & (0.8075, 0.8583) \\
 & $p$ calibrated & \bfseries (0.0833, 0.1758) & \bfseries (0.0917, 0.1583) & \bfseries (0.075, 0.1333) & \bfseries (0.0667, 0.1333) & \bfseries (0.1242, 0.2167) \\
 & $p = N_+ / N$ & (0.6917, 0.7667) & (0.675, 0.75) & (0.625, 0.7083) & (0.6833, 0.7667) & (0.7917, 0.8583) \\
\addlinespace
\multirow[c]{3}{*}{PEMI $B = 100$} & $p = \frac{1}{2}$ & (0.4883, 0.5983) & (0.7417, 0.82) & (0.88, 0.93) & (0.8167, 0.8767) & (0.88, 0.9283) \\
 & $p$ calibrated & \bfseries (0.0767, 0.155) & \bfseries (0.075, 0.1617) & \bfseries (0.085, 0.1817) & \bfseries (0.065, 0.1383) & \bfseries (0.1533, 0.2483) \\
 & $p = N_+ / N$ & (0.7317, 0.8033) & (0.7465, 0.8185) & (0.6567, 0.7417) & (0.7633, 0.835) & (0.8265, 0.8917) \\
\addlinespace
\multirow[c]{3}{*}{PEMI-Gauss} & $p = \frac{1}{2}$ & (0.4497, 0.5632) & (0.7517, 0.8313) & (0.8697, 0.9244) & (0.8385, 0.8935) & (0.8768, 0.9308) \\
 & $p$ calibrated & \bfseries (0.0923, 0.1981) & \bfseries (0.0833, 0.1676) & \bfseries (0.0731, 0.1491) & \bfseries (0.0714, 0.152) & \bfseries (0.1653, 0.2643) \\
 & $p = N_+ / N$ & (0.693, 0.773) & (0.7609, 0.8286) & (0.6405, 0.7249) & (0.7575, 0.8313) & (0.8335, 0.9001) \\
\bottomrule
\end{tabular}

%% file: tables/W_1_df_conf_p_0.3_alpha_0.9.tex
\begin{tabular}{lllllll}
\toprule
 &  & Precision & Recall & Accuracy & $F_1$ & ROC-AUC \\
\midrule
Bootstrap $M = 0$ &  & (0.0954, 0.1231) & (0.0772, 0.1047) & \bfseries (0.0607, 0.0846) & (0.0768, 0.1073) & (0.4246, 0.4546) \\
\addlinespace
\multirow[c]{3}{*}{PEMI $B = 5$} & $p = \frac{1}{2}$ & (0.1742, 0.2367) & (0.345, 0.375) & (0.2958, 0.3933) & (0.3675, 0.3875) & (0.3625, 0.3875) \\
 & $p$ calibrated & \bfseries (0.0433, 0.1117) & \bfseries (0.0608, 0.105) & \bfseries (0.0425, 0.0917) & \bfseries (0.04, 0.1092) & \bfseries (0.0625, 0.1158) \\
 & $p = N_+ / N$ & (0.3383, 0.3675) & (0.3492, 0.3767) & (0.3225, 0.3558) & (0.3383, 0.3725) & (0.3642, 0.3842) \\
\addlinespace
\multirow[c]{3}{*}{PEMI $B = 10$} & $p = \frac{1}{2}$ & (0.2279, 0.2938) & (0.4029, 0.4258) & (0.3958, 0.4446) & (0.3975, 0.425) & (0.4175, 0.4379) \\
 & $p$ calibrated & \bfseries (0.0383, 0.1013) & \bfseries (0.0375, 0.0725) & \bfseries (0.0288, 0.0667) & \bfseries (0.025, 0.0642) & \bfseries (0.0667, 0.1309) \\
 & $p = N_+ / N$ & (0.3917, 0.4192) & (0.3888, 0.4146) & (0.3579, 0.3946) & (0.3992, 0.4212) & (0.4125, 0.4362) \\
\addlinespace
\multirow[c]{3}{*}{PEMI $B = 100$} & $p = \frac{1}{2}$ & (0.2503, 0.3216) & (0.4345, 0.4618) & (0.4757, 0.4869) & (0.4567, 0.4751) & (0.4712, 0.486) \\
 & $p$ calibrated & \bfseries (0.0276, 0.0603) & \bfseries (0.0286, 0.0759) & \bfseries (0.0357, 0.0943) & \bfseries (0.0229, 0.0512) & \bfseries (0.0845, 0.1536) \\
 & $p = N_+ / N$ & (0.4333, 0.4584) & (0.4392, 0.4617) & (0.4045, 0.4372) & (0.4449, 0.4663) & (0.457, 0.478) \\
\addlinespace
\multirow[c]{3}{*}{PEMI-Gauss} & $p = \frac{1}{2}$ & (0.2294, 0.3005) & (0.4397, 0.467) & (0.4749, 0.4882) & (0.4596, 0.479) & (0.4752, 0.4897) \\
 & $p$ calibrated & \bfseries (0.0354, 0.0942) & \bfseries (0.031, 0.0771) & \bfseries (0.0273, 0.0739) & \bfseries (0.0242, 0.0519) & \bfseries (0.0959, 0.1634) \\
 & $p = N_+ / N$ & (0.419, 0.4493) & (0.4413, 0.4666) & (0.3989, 0.4326) & (0.4442, 0.4678) & (0.4639, 0.4832) \\
\bottomrule
\end{tabular}

%% file: tables/df_mae_conf_p_0.3_alpha_0.9.tex
\begin{tabular}{lllllll}
\toprule
 &  & Precision & Recall & Accuracy & $F_1$ & ROC-AUC \\
\midrule
Bootstrap $M = 0$ &  & \bfseries (0.0155, 0.0217) & (0.0124, 0.0178) & (0.0102, 0.0137) & (0.0116, 0.0158) & (0.091, 0.1055) \\
\addlinespace
\multirow[c]{3}{*}{PEMI $B = 5$} & $p = \frac{1}{2}$ & (0.0409, 0.0492) & (0.0922, 0.1156) & (0.0694, 0.0793) & (0.0836, 0.1021) & (0.1068, 0.129) \\
 & $p$ calibrated & \bfseries (0.013, 0.0181) & \bfseries (0.0085, 0.0119) & \bfseries (0.0066, 0.0097) & \bfseries (0.0088, 0.0128) & \bfseries (0.0088, 0.0132) \\
 & $p = N_+ / N$ & (0.0572, 0.0699) & (0.0573, 0.0687) & (0.0334, 0.0388) & (0.0561, 0.068) & (0.0716, 0.0822) \\
\addlinespace
\multirow[c]{3}{*}{PEMI $B = 10$} & $p = \frac{1}{2}$ & (0.0399, 0.0484) & (0.091, 0.1144) & (0.07, 0.0798) & (0.0823, 0.1007) & (0.1066, 0.1292) \\
 & $p$ calibrated & \bfseries (0.0129, 0.0174) & \bfseries (0.0091, 0.0131) & \bfseries (0.0067, 0.0098) & \bfseries (0.0078, 0.0114) & \bfseries (0.0088, 0.0132) \\
 & $p = N_+ / N$ & (0.0571, 0.0699) & (0.0554, 0.067) & (0.0333, 0.0389) & (0.0555, 0.0673) & (0.0719, 0.0824) \\
\addlinespace
\multirow[c]{3}{*}{PEMI $B = 100$} & $p = \frac{1}{2}$ & (0.0419, 0.0504) & (0.0908, 0.1143) & (0.0701, 0.0796) & (0.083, 0.1015) & (0.1071, 0.1292) \\
 & $p$ calibrated & \bfseries (0.0127, 0.0169) & \bfseries (0.0086, 0.0123) & \bfseries (0.0063, 0.0092) & \bfseries (0.0076, 0.0109) & \bfseries (0.0081, 0.0122) \\
 & $p = N_+ / N$ & (0.0556, 0.0682) & (0.0547, 0.0662) & (0.0328, 0.0383) & (0.0563, 0.0682) & (0.0723, 0.0826) \\
\addlinespace
\multirow[c]{3}{*}{PEMI-Gauss} & $p = \frac{1}{2}$ & (0.0414, 0.05) & (0.0912, 0.1146) & (0.07, 0.0795) & (0.0827, 0.1011) & (0.1068, 0.1292) \\
 & $p$ calibrated & \bfseries (0.0126, 0.0168) & \bfseries (0.0088, 0.0124) & \bfseries (0.0064, 0.0093) & \bfseries (0.0075, 0.0109) & \bfseries (0.0083, 0.0124) \\
 & $p = N_+ / N$ & (0.0561, 0.0689) & (0.0552, 0.0666) & (0.0329, 0.0384) & (0.056, 0.0679) & (0.0726, 0.0829) \\
\bottomrule
\end{tabular}

%% file: tables/df_mse_conf_p_0.3_alpha_0.9.tex
\begin{tabular}{lllllll}
\toprule
 &  & Precision & Recall & Accuracy & $F_1$ & ROC-AUC \\
\midrule
Bootstrap $M = 0$ &  & \bfseries (0.0266, 0.0376) & (0.0214, 0.0319) & \bfseries (0.0159, 0.0209) & \bfseries (0.0189, 0.0266) & (0.1075, 0.1226) \\
\addlinespace
\multirow[c]{3}{*}{PEMI $B = 5$} & $p = \frac{1}{2}$ & (0.0536, 0.0624) & (0.1344, 0.1559) & (0.0805, 0.0897) & (0.113, 0.1304) & (0.1391, 0.1627) \\
 & $p$ calibrated & \bfseries (0.0227, 0.0325) & \bfseries (0.0143, 0.02) & \bfseries (0.0126, 0.0178) & \bfseries (0.0165, 0.0239) & \bfseries (0.0166, 0.025) \\
 & $p = N_+ / N$ & (0.0766, 0.0907) & (0.0749, 0.0853) & (0.0406, 0.0457) & (0.0742, 0.0858) & (0.0844, 0.0941) \\
\addlinespace
\multirow[c]{3}{*}{PEMI $B = 10$} & $p = \frac{1}{2}$ & (0.0529, 0.0613) & (0.1331, 0.1548) & (0.0809, 0.0902) & (0.1118, 0.1294) & (0.1396, 0.163) \\
 & $p$ calibrated & \bfseries (0.0207, 0.0296) & \bfseries (0.016, 0.0232) & \bfseries (0.0127, 0.0174) & \bfseries (0.0149, 0.0209) & \bfseries (0.0175, 0.0255) \\
 & $p = N_+ / N$ & (0.0766, 0.0912) & (0.0735, 0.0841) & (0.0408, 0.0464) & (0.0736, 0.0849) & (0.0844, 0.0937) \\
\addlinespace
\multirow[c]{3}{*}{PEMI $B = 100$} & $p = \frac{1}{2}$ & (0.0546, 0.0636) & (0.1333, 0.1549) & (0.0806, 0.0896) & (0.1122, 0.1299) & (0.1395, 0.1629) \\
 & $p$ calibrated & \bfseries (0.02, 0.0279) & \bfseries (0.015, 0.0218) & \bfseries (0.0117, 0.0167) & \bfseries (0.0138, 0.019) & \bfseries (0.0161, 0.0236) \\
 & $p = N_+ / N$ & (0.0747, 0.0895) & (0.0725, 0.0832) & (0.0402, 0.0454) & (0.0744, 0.0862) & (0.0846, 0.0937) \\
\addlinespace
\multirow[c]{3}{*}{PEMI-Gauss} & $p = \frac{1}{2}$ & (0.0542, 0.0629) & (0.1335, 0.1551) & (0.0804, 0.0894) & (0.112, 0.1296) & (0.1395, 0.1629) \\
 & $p$ calibrated & \bfseries (0.0198, 0.0276) & \bfseries (0.0152, 0.0215) & \bfseries (0.0119, 0.0167) & \bfseries (0.0139, 0.0192) & \bfseries (0.0163, 0.024) \\
 & $p = N_+ / N$ & (0.0753, 0.0898) & (0.0729, 0.0835) & (0.0403, 0.0453) & (0.0742, 0.0859) & (0.0848, 0.0942) \\
\bottomrule
\end{tabular}

%% file: main.bbl

\begin{thebibliography}{59}


\ifx \showCODEN    \undefined \def \showCODEN     #1{\unskip}     \fi
\ifx \showDOI      \undefined \def \showDOI       #1{#1}\fi
\ifx \showISBNx    \undefined \def \showISBNx     #1{\unskip}     \fi
\ifx \showISBNxiii \undefined \def \showISBNxiii  #1{\unskip}     \fi
\ifx \showISSN     \undefined \def \showISSN      #1{\unskip}     \fi
\ifx \showLCCN     \undefined \def \showLCCN      #1{\unskip}     \fi
\ifx \shownote     \undefined \def \shownote      #1{#1}          \fi
\ifx \showarticletitle \undefined \def \showarticletitle #1{#1}   \fi
\ifx \showURL      \undefined \def \showURL       {\relax}        \fi
\providecommand\bibfield[2]{#2}
\providecommand\bibinfo[2]{#2}
\providecommand\natexlab[1]{#1}
\providecommand\showeprint[2][]{arXiv:#2}

\bibitem[Allison(2001)]%
        {allison2001missing}
\bibfield{author}{\bibinfo{person}{P.D. Allison}.} \bibinfo{year}{2001}\natexlab{}.
\newblock \bibinfo{booktitle}{\emph{Missing Data}}.
\newblock \bibinfo{publisher}{SAGE Publications}.
\newblock
\showISBNx{9780761916727}
\showLCCN{00100129}


\bibitem[Amoukou et~al\mbox{.}(2024)]%
        {Amoukou2024}
\bibfield{author}{\bibinfo{person}{Salim~I. Amoukou}, \bibinfo{person}{Tom Bewley}, \bibinfo{person}{Saumitra Mishra}, \bibinfo{person}{Freddy Lecue}, \bibinfo{person}{Daniele Magazzeni}, {and} \bibinfo{person}{Manuela Veloso}.} \bibinfo{year}{2024}\natexlab{}.
\newblock \showarticletitle{Sequential Harmful Shift Detection Without Labels}. In \bibinfo{booktitle}{\emph{Advances in Neural Information Processing Systems}}.
\newblock


\bibitem[Ayilara et~al\mbox{.}(2019)]%
        {ayilara2019impact}
\bibfield{author}{\bibinfo{person}{Olawale~F Ayilara}, \bibinfo{person}{Lisa Zhang}, \bibinfo{person}{Tolulope~T Sajobi}, \bibinfo{person}{Richard Sawatzky}, \bibinfo{person}{Eric Bohm}, {and} \bibinfo{person}{Lisa~M Lix}.} \bibinfo{year}{2019}\natexlab{}.
\newblock \showarticletitle{Impact of missing data on bias and precision when estimating change in patient-reported outcomes from a clinical registry}.
\newblock \bibinfo{journal}{\emph{Health and quality of life outcomes}}  \bibinfo{volume}{17} (\bibinfo{year}{2019}), \bibinfo{pages}{1--9}.
\newblock


\bibitem[Ayme et~al\mbox{.}(2024)]%
        {ayme2024randomfeaturesmodelsway}
\bibfield{author}{\bibinfo{person}{Alexis Ayme}, \bibinfo{person}{Claire Boyer}, \bibinfo{person}{Aymeric Dieuleveut}, {and} \bibinfo{person}{Erwan Scornet}.} \bibinfo{year}{2024}\natexlab{}.
\newblock \bibinfo{title}{Random features models: a way to study the success of naive imputation}.
\newblock
\newblock
\showeprint[arxiv]{2402.03839}~[math.ST]


\bibitem[Becker and Kohavi(1996)]%
        {adult_2}
\bibfield{author}{\bibinfo{person}{Barry Becker} {and} \bibinfo{person}{Ronny Kohavi}.} \bibinfo{year}{1996}\natexlab{}.
\newblock \bibinfo{title}{{Adult}}.
\newblock \bibinfo{howpublished}{UCI Machine Learning Repository}.
\newblock
\urldef\tempurl%
\url{https://doi.org/10.24432/C5XW20}
\showURL{%
\tempurl}


\bibitem[Bertsimas et~al\mbox{.}(2018)]%
        {BertsimasMultipleImp2018}
\bibfield{author}{\bibinfo{person}{Dimitris Bertsimas}, \bibinfo{person}{Colin Pawlowski}, {and} \bibinfo{person}{Ying~Daisy Zhuo}.} \bibinfo{year}{2018}\natexlab{}.
\newblock \showarticletitle{From Predictive Methods to Missing Data Imputation: An Optimization Approach}.
\newblock \bibinfo{journal}{\emph{Journal of Machine Learning Research}} \bibinfo{volume}{18}, \bibinfo{number}{196} (\bibinfo{year}{2018}), \bibinfo{pages}{1--39}.
\newblock


\bibitem[Bradley(1997)]%
        {BRADLEY19971145}
\bibfield{author}{\bibinfo{person}{Andrew~P. Bradley}.} \bibinfo{year}{1997}\natexlab{}.
\newblock \showarticletitle{{The use of the area under the ROC curve in the evaluation of machine learning algorithms}}.
\newblock \bibinfo{journal}{\emph{Pattern Recognition}} \bibinfo{volume}{30}, \bibinfo{number}{7} (\bibinfo{year}{1997}), \bibinfo{pages}{1145--1159}.
\newblock


\bibitem[Bradley(2007)]%
        {bradley2007introduction}
\bibfield{author}{\bibinfo{person}{R.C. Bradley}.} \bibinfo{year}{2007}\natexlab{}.
\newblock \bibinfo{booktitle}{\emph{Introduction to Strong Mixing Conditions}}.
\newblock Number v. 1 in \bibinfo{series}{Introduction to Strong Mixing Conditions}. \bibinfo{publisher}{Kendrick Press}.
\newblock


\bibitem[Chapelle et~al\mbox{.}(2006)]%
        {chapelle2006}
\bibfield{author}{\bibinfo{person}{Olivier Chapelle}, \bibinfo{person}{Bernhard Schölkopf}, {and} \bibinfo{person}{Alexander Zien}.} \bibinfo{year}{2006}\natexlab{}.
\newblock \bibinfo{booktitle}{\emph{Semi-Supervised Learning}}.
\newblock \bibinfo{publisher}{MIT Press}.
\newblock


\bibitem[Chen and Guestrin(2016)]%
        {xgboost}
\bibfield{author}{\bibinfo{person}{Tianqi Chen} {and} \bibinfo{person}{Carlos Guestrin}.} \bibinfo{year}{2016}\natexlab{}.
\newblock \showarticletitle{{XGBoost: A Scalable Tree Boosting System}}. In \bibinfo{booktitle}{\emph{{Proceedings of the 22nd ACM SIGKDD International Conference on Knowledge Discovery and Data Mining}}} \emph{(\bibinfo{series}{KDD '16})}.
\newblock


\bibitem[Dedecker et~al\mbox{.}(2007)]%
        {dedecker2007weak}
\bibfield{author}{\bibinfo{person}{J. Dedecker}, \bibinfo{person}{P. Doukhan}, \bibinfo{person}{G. Lang}, \bibinfo{person}{J.R. Leon}, \bibinfo{person}{S. Louhichi}, {and} \bibinfo{person}{C. Prieur}.} \bibinfo{year}{2007}\natexlab{}.
\newblock \bibinfo{booktitle}{\emph{Weak Dependence: With Examples and Applications}}.
\newblock \bibinfo{publisher}{Springer New York}.
\newblock
\showISBNx{9780387699523}
\showLCCN{2007929938}


\bibitem[DeGroot and Fienberg(1983)]%
        {DeGroot1983}
\bibfield{author}{\bibinfo{person}{Morris~H. DeGroot} {and} \bibinfo{person}{Stephen~E. Fienberg}.} \bibinfo{year}{1983}\natexlab{}.
\newblock \showarticletitle{The Comparison and Evaluation of Forecasters}.
\newblock \bibinfo{journal}{\emph{Journal of the Royal Statistical Society. Series D (The Statistician)}} \bibinfo{volume}{32}, \bibinfo{number}{1/2} (\bibinfo{year}{1983}), \bibinfo{pages}{12--22}.
\newblock
\showISSN{00390526, 14679884}


\bibitem[Dempster et~al\mbox{.}(1977)]%
        {dempster1977maximum}
\bibfield{author}{\bibinfo{person}{Arthur~P Dempster}, \bibinfo{person}{Nan~M Laird}, {and} \bibinfo{person}{Donald~B Rubin}.} \bibinfo{year}{1977}\natexlab{}.
\newblock \showarticletitle{Maximum likelihood from incomplete data via the EM algorithm}.
\newblock \bibinfo{journal}{\emph{{Journal of the Royal Statistical Society: Series B (Methodological)}}} \bibinfo{volume}{39}, \bibinfo{number}{1} (\bibinfo{year}{1977}), \bibinfo{pages}{1--22}.
\newblock


\bibitem[Diebold et~al\mbox{.}(1998)]%
        {Diebold1998}
\bibfield{author}{\bibinfo{person}{Francis~X. Diebold}, \bibinfo{person}{Todd~A. Gunther}, {and} \bibinfo{person}{Anthony~S. Tay}.} \bibinfo{year}{1998}\natexlab{}.
\newblock \showarticletitle{Evaluating Density Forecasts with Applications to Financial Risk Management}.
\newblock \bibinfo{journal}{\emph{International Economic Review}} \bibinfo{volume}{39}, \bibinfo{number}{4} (\bibinfo{year}{1998}), \bibinfo{pages}{863--883}.
\newblock


\bibitem[Donders et~al\mbox{.}(2006)]%
        {donders2006gentle}
\bibfield{author}{\bibinfo{person}{A~Rogier~T Donders}, \bibinfo{person}{Geert~JMG Van Der~Heijden}, \bibinfo{person}{Theo Stijnen}, {and} \bibinfo{person}{Karel~GM Moons}.} \bibinfo{year}{2006}\natexlab{}.
\newblock \showarticletitle{A gentle introduction to imputation of missing values}.
\newblock \bibinfo{journal}{\emph{Journal of clinical epidemiology}} \bibinfo{volume}{59}, \bibinfo{number}{10} (\bibinfo{year}{2006}), \bibinfo{pages}{1087--1091}.
\newblock


\bibitem[Doukhan(2012)]%
        {doukhan2012mixing}
\bibfield{author}{\bibinfo{person}{P. Doukhan}.} \bibinfo{year}{2012}\natexlab{}.
\newblock \bibinfo{booktitle}{\emph{Mixing: Properties and Examples}}.
\newblock \bibinfo{publisher}{Springer New York}.
\newblock
\showISBNx{9781461226420}


\bibitem[d’Haultfoeuille(2010)]%
        {DHAULTFOEUILLE20101}
\bibfield{author}{\bibinfo{person}{Xavier d’Haultfoeuille}.} \bibinfo{year}{2010}\natexlab{}.
\newblock \showarticletitle{A new instrumental method for dealing with endogenous selection}.
\newblock \bibinfo{journal}{\emph{Journal of Econometrics}} \bibinfo{volume}{154}, \bibinfo{number}{1} (\bibinfo{year}{2010}), \bibinfo{pages}{1--15}.
\newblock
\showISSN{0304-4076}


\bibitem[Esseen(1956)]%
        {Esseen1956}
\bibfield{author}{\bibinfo{person}{C.~G. Esseen}.} \bibinfo{year}{1956}\natexlab{}.
\newblock \showarticletitle{A moment inequality with an application to the central limit theorem}.
\newblock \bibinfo{journal}{\emph{Scandinavian Actuarial Journal}} \bibinfo{volume}{1956}, \bibinfo{number}{2} (\bibinfo{year}{1956}), \bibinfo{pages}{160--170}.
\newblock


\bibitem[Feurer et~al\mbox{.}(2021)]%
        {openml}
\bibfield{author}{\bibinfo{person}{Matthias Feurer}, \bibinfo{person}{Jan~N. van Rijn}, \bibinfo{person}{Arlind Kadra}, \bibinfo{person}{Pieter Gijsbers}, \bibinfo{person}{Neeratyoy Mallik}, \bibinfo{person}{Sahithya Ravi}, \bibinfo{person}{Andreas Müller}, \bibinfo{person}{Joaquin Vanschoren}, {and} \bibinfo{person}{Frank Hutter}.} \bibinfo{year}{2021}\natexlab{}.
\newblock \showarticletitle{OpenML-Python: an extensible Python API for OpenML}.
\newblock \bibinfo{journal}{\emph{Journal of Machine Learning Research}} \bibinfo{volume}{22}, \bibinfo{number}{100} (\bibinfo{year}{2021}), \bibinfo{pages}{1--5}.
\newblock


\bibitem[Gong and Ma(2025)]%
        {Gong2025}
\bibfield{author}{\bibinfo{person}{Shuxia Gong} {and} \bibinfo{person}{Chen Ma}.} \bibinfo{year}{2025}\natexlab{}.
\newblock \showarticletitle{Gradient-Based Multiple Robust Learning Calibration on Data Missing-Not-at-Random via Bi-Level Optimization}.
\newblock \bibinfo{journal}{\emph{Entropy}} \bibinfo{volume}{27}, \bibinfo{number}{2} (\bibinfo{year}{2025}).
\newblock


\bibitem[Guan and Yuan(2024)]%
        {GUAN2024107204}
\bibfield{author}{\bibinfo{person}{Licong Guan} {and} \bibinfo{person}{Xue Yuan}.} \bibinfo{year}{2024}\natexlab{}.
\newblock \showarticletitle{Label-free model evaluation and weighted uncertainty sample selection for domain adaptive instance segmentation}.
\newblock \bibinfo{journal}{\emph{Engineering Applications of Artificial Intelligence}}  \bibinfo{volume}{127} (\bibinfo{year}{2024}), \bibinfo{pages}{107204}.
\newblock
\showISSN{0952-1976}


\bibitem[Guo et~al\mbox{.}(2017)]%
        {Guo2017}
\bibfield{author}{\bibinfo{person}{Chuan Guo}, \bibinfo{person}{Geoff Pleiss}, \bibinfo{person}{Yu Sun}, {and} \bibinfo{person}{Kilian~Q. Weinberger}.} \bibinfo{year}{2017}\natexlab{}.
\newblock \showarticletitle{On calibration of modern neural networks}. In \bibinfo{booktitle}{\emph{Proceedings of the 34th International Conference on Machine Learning}} \emph{(\bibinfo{series}{ICML'17})}. \bibinfo{pages}{1321–1330}.
\newblock


\bibitem[Hanley and McNeil(1983)]%
        {Hanley1983}
\bibfield{author}{\bibinfo{person}{J~A Hanley} {and} \bibinfo{person}{B~J McNeil}.} \bibinfo{year}{1983}\natexlab{}.
\newblock \showarticletitle{{A method of comparing the areas under receiver operating characteristic curves derived from the same cases.}}
\newblock \bibinfo{journal}{\emph{Radiology}} \bibinfo{volume}{148}, \bibinfo{number}{3} (\bibinfo{year}{1983}), \bibinfo{pages}{839--843}.
\newblock


\bibitem[Hayya et~al\mbox{.}(1975)]%
        {Hayya1975ANO}
\bibfield{author}{\bibinfo{person}{Jack~C. Hayya}, \bibinfo{person}{Don Armstrong}, {and} \bibinfo{person}{Nicolas Gressis}.} \bibinfo{year}{1975}\natexlab{}.
\newblock \showarticletitle{A Note on the Ratio of Two Normally Distributed Variables}.
\newblock \bibinfo{journal}{\emph{Management Science}}  \bibinfo{volume}{21} (\bibinfo{year}{1975}), \bibinfo{pages}{1338--1341}.
\newblock


\bibitem[Hesterberg(2015)]%
        {Hesterberg2015s}
\bibfield{author}{\bibinfo{person}{Tim~C. Hesterberg}.} \bibinfo{year}{2015}\natexlab{}.
\newblock \showarticletitle{What Teachers Should Know About the Bootstrap: Resampling in the Undergraduate Statistics Curriculum}.
\newblock \bibinfo{journal}{\emph{The American Statistician}} \bibinfo{volume}{69}, \bibinfo{number}{4} (\bibinfo{year}{2015}), \bibinfo{pages}{371--386}.
\newblock


\bibitem[Hino(2020)]%
        {hino2020active}
\bibfield{author}{\bibinfo{person}{Hideitsu Hino}.} \bibinfo{year}{2020}\natexlab{}.
\newblock \bibinfo{title}{Active Learning: Problem Settings and Recent Developments}.
\newblock
\newblock
\showeprint[arxiv]{2012.04225}~[cs.LG]


\bibitem[Hofmann(1994)]%
        {statlog_(german_credit_data)_144}
\bibfield{author}{\bibinfo{person}{Hans Hofmann}.} \bibinfo{year}{1994}\natexlab{}.
\newblock \bibinfo{title}{{Statlog (German Credit Data)}}.
\newblock \bibinfo{howpublished}{UCI Machine Learning Repository}.
\newblock
\urldef\tempurl%
\url{https://doi.org/10.24432/C5NC77}
\showURL{%
\tempurl}


\bibitem[Hu et~al\mbox{.}(2022)]%
        {DBLP:conf/iclr/HuNM0Z22}
\bibfield{author}{\bibinfo{person}{Xinting Hu}, \bibinfo{person}{Yulei Niu}, \bibinfo{person}{Chunyan Miao}, \bibinfo{person}{Xian{-}Sheng Hua}, {and} \bibinfo{person}{Hanwang Zhang}.} \bibinfo{year}{2022}\natexlab{}.
\newblock \showarticletitle{On Non-Random Missing Labels in Semi-Supervised Learning}. In \bibinfo{booktitle}{\emph{The Tenth International Conference on Learning Representations, {ICLR} 2022}}.
\newblock


\bibitem[Ibrahim et~al\mbox{.}(2001)]%
        {Ibrahim2001}
\bibfield{author}{\bibinfo{person}{Joseph~G. Ibrahim}, \bibinfo{person}{Ming-Hui Chen}, {and} \bibinfo{person}{Stuart~R. Lipsitz}.} \bibinfo{year}{2001}\natexlab{}.
\newblock \showarticletitle{Missing Responses in Generalised Linear Mixed Models When the Missing Data Mechanism is Nonignorable}.
\newblock \bibinfo{journal}{\emph{Biometrika}} \bibinfo{volume}{88}, \bibinfo{number}{2} (\bibinfo{year}{2001}), \bibinfo{pages}{551--564}.
\newblock
\showISSN{00063444, 14643510}


\bibitem[Ibrahim and Lipsitz(1996)]%
        {Ibrahim1996}
\bibfield{author}{\bibinfo{person}{Joseph~G. Ibrahim} {and} \bibinfo{person}{Stuart~R. Lipsitz}.} \bibinfo{year}{1996}\natexlab{}.
\newblock \showarticletitle{Parameter Estimation from Incomplete Data in Binomial Regression When the Missing Data Mechanism is Nonignorable}.
\newblock \bibinfo{journal}{\emph{Biometrics}} \bibinfo{volume}{52}, \bibinfo{number}{3} (\bibinfo{year}{1996}), \bibinfo{pages}{1071--1078}.
\newblock
\showISSN{0006341X, 15410420}


\bibitem[IMDB(2024)]%
        {IMDB.drama.permission}
\bibfield{author}{\bibinfo{person}{IMDB}.} \bibinfo{year}{2024}\natexlab{}.
\newblock \bibinfo{title}{Information courtesy of IMDb (\url{https://www.imdb.com}). Used with permission.}
\newblock
\newblock


\bibitem[Jaynes(1957)]%
        {Jaynes1957}
\bibfield{author}{\bibinfo{person}{E.~T. Jaynes}.} \bibinfo{year}{1957}\natexlab{}.
\newblock \showarticletitle{Information Theory and Statistical Mechanics}.
\newblock \bibinfo{journal}{\emph{Phys. Rev.}}  \bibinfo{volume}{106} (\bibinfo{date}{May} \bibinfo{year}{1957}), \bibinfo{pages}{620--630}.
\newblock
Issue 4.


\bibitem[Josse et~al\mbox{.}(2024)]%
        {josse2024consistencysupervisedlearningmissing}
\bibfield{author}{\bibinfo{person}{Julie Josse}, \bibinfo{person}{Jacob~M. Chen}, \bibinfo{person}{Nicolas Prost}, \bibinfo{person}{Erwan Scornet}, {and} \bibinfo{person}{Gaël Varoquaux}.} \bibinfo{year}{2024}\natexlab{}.
\newblock \bibinfo{title}{On the consistency of supervised learning with missing values}.
\newblock
\newblock
\showeprint[arxiv]{1902.06931}~[stat.ML]


\bibitem[Josse and Reiter(2018)]%
        {JosseReiter2018}
\bibfield{author}{\bibinfo{person}{Julie Josse} {and} \bibinfo{person}{Jerome~P. Reiter}.} \bibinfo{year}{2018}\natexlab{}.
\newblock \showarticletitle{Introduction to the Special Section on Missing Data}.
\newblock \bibinfo{journal}{\emph{Statist. Sci.}} \bibinfo{volume}{33}, \bibinfo{number}{2} (\bibinfo{year}{2018}), \bibinfo{pages}{139--141}.
\newblock
\showISSN{08834237, 21688745}


\bibitem[Kahn(1994)]%
        {diabetes_34}
\bibfield{author}{\bibinfo{person}{Michael Kahn}.} \bibinfo{year}{1994}\natexlab{}.
\newblock \bibinfo{title}{{Diabetes}}.
\newblock \bibinfo{howpublished}{UCI Machine Learning Repository}.
\newblock
\urldef\tempurl%
\url{https://doi.org/10.24432/C5T59G}
\showURL{%
\tempurl}


\bibitem[Kenthapadi et~al\mbox{.}(2022)]%
        {Krishnaram2022}
\bibfield{author}{\bibinfo{person}{Krishnaram Kenthapadi}, \bibinfo{person}{Himabindu Lakkaraju}, \bibinfo{person}{Pradeep Natarajan}, {and} \bibinfo{person}{Mehrnoosh Sameki}.} \bibinfo{year}{2022}\natexlab{}.
\newblock \showarticletitle{Model Monitoring in Practice: Lessons Learned and Open Challenges}. In \bibinfo{booktitle}{\emph{Proceedings of the 28th ACM SIGKDD Conference on Knowledge Discovery and Data Mining}} \emph{(\bibinfo{series}{KDD '22})}. \bibinfo{pages}{4800–4801}.
\newblock


\bibitem[Kuleshov et~al\mbox{.}(2018)]%
        {kuleshov2018accurate}
\bibfield{author}{\bibinfo{person}{Volodymyr Kuleshov}, \bibinfo{person}{Nathan Fenner}, {and} \bibinfo{person}{Stefano Ermon}.} \bibinfo{year}{2018}\natexlab{}.
\newblock \showarticletitle{Accurate uncertainties for deep learning using calibrated regression}. In \bibinfo{booktitle}{\emph{International conference on machine learning}}. PMLR, \bibinfo{pages}{2796--2804}.
\newblock


\bibitem[Kumar et~al\mbox{.}(2019)]%
        {kumar2019calibration}
\bibfield{author}{\bibinfo{person}{Ananya Kumar}, \bibinfo{person}{Percy Liang}, {and} \bibinfo{person}{Tengyu Ma}.} \bibinfo{year}{2019}\natexlab{}.
\newblock \showarticletitle{Verified Uncertainty Calibration}. In \bibinfo{booktitle}{\emph{Advances in Neural Information Processing Systems}} \emph{(\bibinfo{series}{NeurIPS '19})}.
\newblock


\bibitem[Kweon and Yu(2024)]%
        {Kweon2024}
\bibfield{author}{\bibinfo{person}{Wonbin Kweon} {and} \bibinfo{person}{Hwanjo Yu}.} \bibinfo{year}{2024}\natexlab{}.
\newblock \showarticletitle{Doubly Calibrated Estimator for Recommendation on Data Missing Not at Random}. In \bibinfo{booktitle}{\emph{Proceedings of the ACM Web Conference 2024}} \emph{(\bibinfo{series}{WWW '24})}. \bibinfo{pages}{3810–3820}.
\newblock


\bibitem[Little et~al\mbox{.}(2019)]%
        {little2019statistical}
\bibfield{author}{\bibinfo{person}{R. Little}, \bibinfo{person}{D. Rubin}, {and} \bibinfo{person}{an~O'Reilly Media~Company Safari}.} \bibinfo{year}{2019}\natexlab{}.
\newblock \bibinfo{booktitle}{\emph{Statistical Analysis with Missing Data., 3rd Edition}}.
\newblock \bibinfo{publisher}{Wiley}.
\newblock


\bibitem[Lobo et~al\mbox{.}(2024)]%
        {lobo2024harnessingpatternbypatternlinearclassifiers}
\bibfield{author}{\bibinfo{person}{Angel D~Reyero Lobo}, \bibinfo{person}{Alexis Ayme}, \bibinfo{person}{Claire Boyer}, {and} \bibinfo{person}{Erwan Scornet}.} \bibinfo{year}{2024}\natexlab{}.
\newblock \bibinfo{title}{Harnessing pattern-by-pattern linear classifiers for prediction with missing data}.
\newblock
\newblock
\showeprint[arxiv]{2405.09196}~[math.ST]


\bibitem[Marsaglia(1965)]%
        {Marsaglia1965}
\bibfield{author}{\bibinfo{person}{George Marsaglia}.} \bibinfo{year}{1965}\natexlab{}.
\newblock \showarticletitle{Ratios of Normal Variables and Ratios of Sums of Uniform Variables}.
\newblock \bibinfo{journal}{\emph{J. Amer. Statist. Assoc.}} \bibinfo{volume}{60}, \bibinfo{number}{309} (\bibinfo{year}{1965}), \bibinfo{pages}{193--204}.
\newblock
\showISSN{01621459, 1537274X}


\bibitem[Marsaglia(2006)]%
        {Marsaglia2006}
\bibfield{author}{\bibinfo{person}{George Marsaglia}.} \bibinfo{year}{2006}\natexlab{}.
\newblock \showarticletitle{Ratios of Normal Variables}.
\newblock \bibinfo{journal}{\emph{Journal of Statistical Software}} \bibinfo{volume}{16}, \bibinfo{number}{4} (\bibinfo{year}{2006}), \bibinfo{pages}{1–10}.
\newblock


\bibitem[Mishler et~al\mbox{.}(2023)]%
        {mishler2023active}
\bibfield{author}{\bibinfo{person}{Alan Mishler}, \bibinfo{person}{Mohsen Ghassemi}, \bibinfo{person}{Alec Koppel}, {and} \bibinfo{person}{Sumitra Ganesh}.} \bibinfo{year}{2023}\natexlab{}.
\newblock \showarticletitle{Active Learning with Missing Not At Random Outcomes}. In \bibinfo{booktitle}{\emph{NeurIPS 2023 Workshop on Adaptive Experimental Design and Active Learning in the Real World}}.
\newblock


\bibitem[Mohan et~al\mbox{.}(2013)]%
        {Mohan2014}
\bibfield{author}{\bibinfo{person}{Karthika Mohan}, \bibinfo{person}{Judea Pearl}, {and} \bibinfo{person}{Jin Tian}.} \bibinfo{year}{2013}\natexlab{}.
\newblock \showarticletitle{Graphical Models for Inference with Missing Data}. In \bibinfo{booktitle}{\emph{Advances in Neural Information Processing Systems}} \emph{(\bibinfo{series}{NeurIPS '13})}.
\newblock


\bibitem[Moro and Rita(2014)]%
        {bank_marketing_222}
\bibfield{author}{\bibinfo{person}{S. Moro} {and} \bibinfo{person}{P. Rita, P.and~Cortez}.} \bibinfo{year}{2014}\natexlab{}.
\newblock \bibinfo{title}{{Bank Marketing}}.
\newblock \bibinfo{howpublished}{UCI Machine Learning Repository}.
\newblock
\urldef\tempurl%
\url{https://doi.org/10.24432/C5K306}
\showURL{%
\tempurl}


\bibitem[Morvan et~al\mbox{.}(2024)]%
        {MorvanNeurIPS21}
\bibfield{author}{\bibinfo{person}{Marine~Le Morvan}, \bibinfo{person}{Julie Josse}, \bibinfo{person}{Erwan Scornet}, {and} \bibinfo{person}{Ga\"{e}l Varoquaux}.} \bibinfo{year}{2024}\natexlab{}.
\newblock \showarticletitle{What's a good imputation to predict with missing values?}. In \bibinfo{booktitle}{\emph{Proceedings of the 35th International Conference on Neural Information Processing Systems}} \emph{(\bibinfo{series}{NeurIPS '21})}.
\newblock


\bibitem[Read et~al\mbox{.}(2012)]%
        {IMDB.drama}
\bibfield{author}{\bibinfo{person}{Jesse Read}, \bibinfo{person}{Albert Bifet}, \bibinfo{person}{Bernhard Pfahringer}, {and} \bibinfo{person}{Geoff Holmes}.} \bibinfo{year}{2012}\natexlab{}.
\newblock \showarticletitle{Batch-Incremental versus Instance-Incremental Learning in Dynamic and Evolving Data}. In \bibinfo{booktitle}{\emph{Advances in Intelligent Data Analysis XI}}. \bibinfo{publisher}{Springer Berlin Heidelberg}, \bibinfo{address}{Berlin, Heidelberg}, \bibinfo{pages}{313--323}.
\newblock
\showISBNx{978-3-642-34156-4}


\bibitem[Rubin(1976)]%
        {rubin1976inference}
\bibfield{author}{\bibinfo{person}{Donald~B Rubin}.} \bibinfo{year}{1976}\natexlab{}.
\newblock \showarticletitle{Inference and missing data}.
\newblock \bibinfo{journal}{\emph{Biometrika}} \bibinfo{volume}{63}, \bibinfo{number}{3} (\bibinfo{year}{1976}), \bibinfo{pages}{581--592}.
\newblock


\bibitem[Scudder(1965)]%
        {Scudder1965}
\bibfield{author}{\bibinfo{person}{H. Scudder}.} \bibinfo{year}{1965}\natexlab{}.
\newblock \showarticletitle{Probability of error of some adaptive pattern-recognition machines}.
\newblock \bibinfo{journal}{\emph{IEEE Transactions on Information Theory}} \bibinfo{volume}{11}, \bibinfo{number}{3} (\bibinfo{year}{1965}), \bibinfo{pages}{363--371}.
\newblock


\bibitem[Settles(2009)]%
        {settles.tr09}
\bibfield{author}{\bibinfo{person}{Burr Settles}.} \bibinfo{year}{2009}\natexlab{}.
\newblock \bibinfo{booktitle}{\emph{Active Learning Literature Survey}}.
\newblock \bibinfo{type}{Computer Sciences Technical Report} 1648. \bibinfo{institution}{University of Wisconsin--Madison}.
\newblock


\bibitem[Shevtsova(2010)]%
        {shevtsova2010improvement}
\bibfield{author}{\bibinfo{person}{Irina~G Shevtsova}.} \bibinfo{year}{2010}\natexlab{}.
\newblock \showarticletitle{An improvement of convergence rate estimates in the Lyapunov theorem.}. In \bibinfo{booktitle}{\emph{Doklady Mathematics}}, Vol.~\bibinfo{volume}{82}.
\newblock


\bibitem[Song and Shepperd(2007)]%
        {song2007missing}
\bibfield{author}{\bibinfo{person}{Qinbao Song} {and} \bibinfo{person}{Martin Shepperd}.} \bibinfo{year}{2007}\natexlab{}.
\newblock \showarticletitle{Missing data imputation techniques}.
\newblock \bibinfo{journal}{\emph{International journal of business intelligence and data mining}} \bibinfo{volume}{2}, \bibinfo{number}{3} (\bibinfo{year}{2007}), \bibinfo{pages}{261--291}.
\newblock


\bibitem[Sportisse et~al\mbox{.}(2023)]%
        {Sportisse2023}
\bibfield{author}{\bibinfo{person}{Aude Sportisse}, \bibinfo{person}{Hugo Schmutz}, \bibinfo{person}{Olivier Humbert}, \bibinfo{person}{Charles Bouveyron}, {and} \bibinfo{person}{Pierre-Alexandre Mattei}.} \bibinfo{year}{2023}\natexlab{}.
\newblock \showarticletitle{Are labels informative in semi-supervised learning? estimating and leveraging the missing-data mechanism}. In \bibinfo{booktitle}{\emph{Proceedings of the 40th International Conference on Machine Learning}} \emph{(\bibinfo{series}{ICML'23})}.
\newblock


\bibitem[Sun et~al\mbox{.}(2021)]%
        {Sun2021}
\bibfield{author}{\bibinfo{person}{Xiaoxiao Sun}, \bibinfo{person}{Yunzhong Hou}, \bibinfo{person}{Weijian Deng}, \bibinfo{person}{Hongdong Li}, {and} \bibinfo{person}{Liang Zheng}.} \bibinfo{year}{2021}\natexlab{}.
\newblock \showarticletitle{Ranking Models in Unlabeled New Environments}. In \bibinfo{booktitle}{\emph{2021 IEEE/CVF International Conference on Computer Vision}} \emph{(\bibinfo{series}{ICCV '21})}.
\newblock


\bibitem[Tikhomirov(1981)]%
        {Tikhomirov1981}
\bibfield{author}{\bibinfo{person}{A.~N. Tikhomirov}.} \bibinfo{year}{1981}\natexlab{}.
\newblock \showarticletitle{On the Convergence Rate in the Central Limit Theorem for Weakly Dependent Random Variables}.
\newblock \bibinfo{journal}{\emph{Theory of Probability \& Its Applications}} \bibinfo{volume}{25}, \bibinfo{number}{4} (\bibinfo{year}{1981}), \bibinfo{pages}{790--809}.
\newblock


\bibitem[Tridgell(2016)]%
        {dota2_games_results_367}
\bibfield{author}{\bibinfo{person}{Stephen Tridgell}.} \bibinfo{year}{2016}\natexlab{}.
\newblock \bibinfo{title}{{Dota2 Games Results}}.
\newblock \bibinfo{howpublished}{UCI Machine Learning Repository}.
\newblock
\urldef\tempurl%
\url{https://doi.org/10.24432/C5W593}
\showURL{%
\tempurl}


\bibitem[Williams(2015)]%
        {williams2015missing}
\bibfield{author}{\bibinfo{person}{Richard Williams}.} \bibinfo{year}{2015}\natexlab{}.
\newblock \showarticletitle{Missing data part 1: Overview, traditional methods}.
\newblock \bibinfo{journal}{\emph{Notre Dame: University of Notre Dame}} (\bibinfo{year}{2015}).
\newblock


\bibitem[Zadrozny and Elkan(2001)]%
        {Zadrozny2001}
\bibfield{author}{\bibinfo{person}{Bianca Zadrozny} {and} \bibinfo{person}{Charles Elkan}.} \bibinfo{year}{2001}\natexlab{}.
\newblock \showarticletitle{Obtaining calibrated probability estimates from decision trees and naive Bayesian classifiers}. In \bibinfo{booktitle}{\emph{Proceedings of the Eighteenth International Conference on Machine Learning}} \emph{(\bibinfo{series}{ICML '01})}. \bibinfo{pages}{609–616}.
\newblock


\end{thebibliography}
